\documentclass[12pt]{article}
\usepackage[margin=1in]{geometry}
\usepackage[round]{natbib}

\usepackage{times,morenotations}

\newcommand{\sndS}{$\mathsf{Simple\mbox{-}Evol\mbox{-}Gen}$}

\title{Evolving a Vector Space with any
  Generating Set}

\author{
  Richard Nock\\
{\normalsize Data61, The Australian National University \& The
  University of Sydney}\\
  \texttt{richard.nock@data61.csiro.au}
\and
 Frank Nielsen\\
{\normalsize Ecole Polytechnique \& Sony Computer Science
  Laboratories, Inc.}\\
 \texttt{frank.nielsen@acm.org}
}
\date{}

\begin{document}

\maketitle

\begin{abstract}

In Valiant's model of evolution, a
class of representations is evolvable iff a polynomial-time process of
random mutations guided by selection converges with high probability
to a representation as $\epsilon$-close
as desired from the optimal one, for any required
$\epsilon>0$. Several previous positive results exist that can be
related to evolving a vector space, but each former result imposes
disproportionate representations or
restrictions on (re)initialisations, distributions,
performance functions and/or the
mutator. In this paper, we show that all it takes to evolve a 
normed vector
space is merely a set that generates the
space. Furthermore, it takes only
$\tilde{O}(1/\epsilon^2)$ steps and it is essentially stable, agnostic and handles target drifts that rival some proven
in fairly restricted settings. Our algorithm can be viewed as a close
relative to a popular fifty-years old gradient-free optimization method for which
little is still known from the convergence standpoint: Nelder-Mead simplex method.
\end{abstract}

\noindent {\bf keywords}:
Evolvability, vector space, Bregman divergence.

\section{Introduction}

How can evolution learn ? About a year ago, strong connections between
evolution and machine learning at large were highlighted and discussed
\citep{bEA,jHC,lpEA,wsHCa,wsHCb,zsIA}. One key open challenge
emerged
as "Evo-Devo", the study of
phenotypic variation, the evolution of traits and how these can benefit
evolution for future outcomes given that it is essentially a myopic process. This
addresses several key observations, (i) the evolutionary process
has indeed no anticipation of future
outcomes, (ii) the process is randomized and weakly guided by the
responses to current conditions. Also, (iii) changes happen at the
genotypic level but in fact are observed at a phenotypic level which
depends on past selections. Modularity \citep{hhlmFM} plays a key role in these:
a certain form of intermediate level
organisation, modular and multivariate, might facilitate selections
through relevant combinations of modules that were successful in the
past. \textit{Stability} is also important as evolution is a still
poorly understood balance
between change and conservation \citep{swFA}. We complete
Evo-Devo with the capacity of evolution to be
(iv) agnostic, (v) adaptive and (vi) distribution-free in the machine learning jargon \citep{kvwEW,vE}:
\textit{agnostic} because parsimony and model complexity constraints on organisms
\citep{wsHCa} might just prevent selection to reach a perfectly
accurate and encodable organism even if carried out forever, \textit{adaptive} because the "evolution
of evolvability" \citep{wsHCa} shall require evolution to handle drifts in the
optimal organisms (due \textit{e.g.} to changes in external
conditions), and finally \textit{distribution-free} with respect to
conditions to cope with adaptability over long horizon with
various environmental conditions \citep{wsHCa}.

The evolvability
model of \cite{vE} is an excellent candidate to frame and formalize
such properties, but in the large body of work published on or before
\cite{wsHCa} in Valiant's evolvability model, it is quite remarkable
that no result frames substantial part of the constraints
above, and even less so comes up with a potentially
\textit{implementable} stochastic algorithm. This last question is
of practical importance at a time where gradient-free optimization is sparking new
interest in computer science and optimization \citep{nsRG}.
\begin{center}
\fbox{%
    \parbox{0.95\textwidth}{
\textbf{Our contribution, summarized,} is a proof of Valiant's evolvability for
finite-dimensional normed vector spaces using their simplest defining structure, a
generating set --- and with \textit{no} distributional assumptions
on conditions. In addition, we prove that the same algorithm, which
spans few lines of pseudocode, can be made powerful beyond Valiant's initial
requirements, including being agnostic, strictly monotonic, stable and
handling significant target drift.
}
}
\end{center}
 Our
proof of evolvability is \textit{constructive}. In the Evo-Devo scenario, we perform a genotype /
phenotype distinction by representing modular components of the phenotypes as
functions mapping observable conditions to real-valued vectors (\textit{e.g.}
relative size, height, concentration of certain proteins,
etc.). An organism is a linear a combination of
these functions, represented by a vector whose coordinates weight each
of these. Mutations are represented by a set of vectors in the same
space. Incidentally and interestingly, the vocabulary of linear algebra translates to
high-level characterizations of the evo-devo regime: for example,
\textit{pleiotropy}\footnote{Roughly, phenomenon by which a gene
  affects two or more traits.} may
arise from the non-sparsity of a
mutation vector \citep{sOH}, a small number of mutations --- in
particular not defining a generating set --- may
indicate parsimony pressure on evolution \citep{wsHCa}, a large number
of mutations may indicate genetic redundancy \citep{kspGR} and so on.

\begin{figure}[t]
\begin{center}
\includegraphics[trim=10bp 320bp 50bp
130bp,clip,width=0.6\columnwidth]{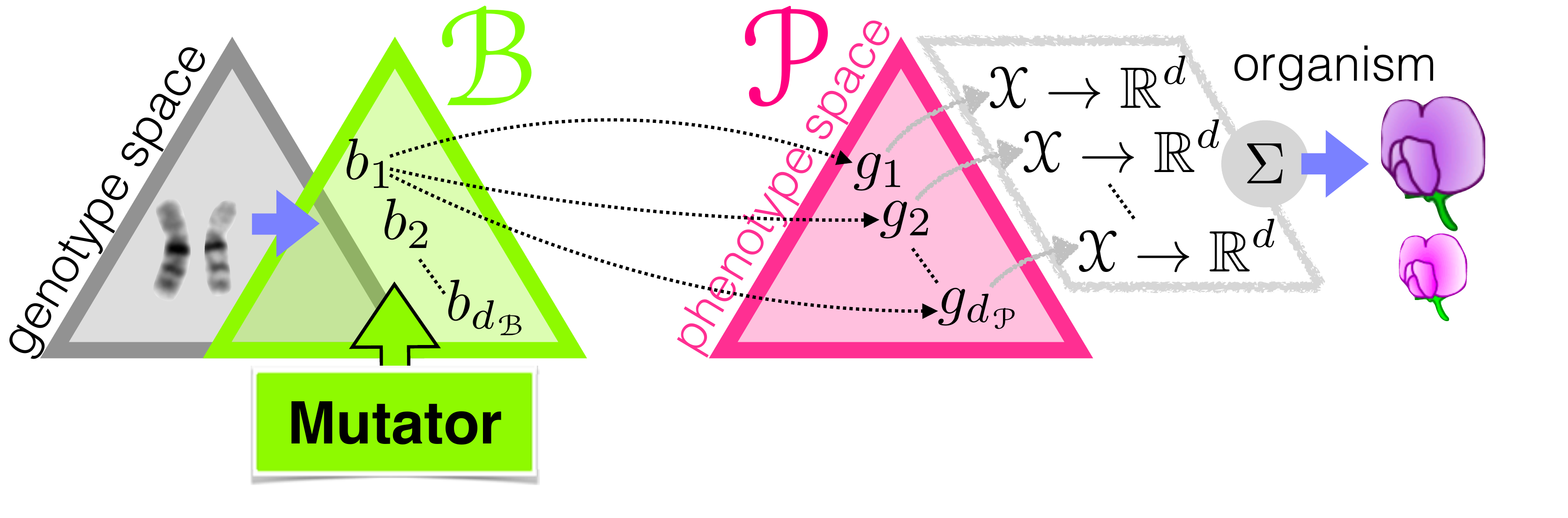}
\end{center}
\caption{In our model, a phenotype is a linear combination of functions
  mapping conditions from a set $\mathcal{X}$ to observables in
  $\mathbb{R}^d$ (right). An organism is a linear combination of these, and it
is evolved through mutations using a set of vectors $\mathcal{B}$
defining the effect of each mutations on each function of the
phenotype (left). The mutator is the Turing-computable mechanism performing
mutations in Valiant's evolvability model \citep{vE} (see text).}
\label{f-modglob}
\end{figure}

The rest of this paper is organised as follows: Section $\S$\ref{sec-rel} presents related work in Valiant's evolvability model; $\S$\ref{sec-evol}
details the evolvability model, $\S$\ref{sec-proof} states and gives a high-level proof of
evolvability, $\S$\ref{sec-agno}, $\S$ \ref{sec-stab} and $\S$\ref{sec-drift} respectively
state the agnostic, stability and drift-compliant evolvability results and
$\S$\ref{sec-toy-exp} sketches toy experiments. A last
$\S$\ref{sec-disc} discusses and concludes. 
For space considerations, an Appendix (starting $\S$\ref{app-overview}) provides all proofs and
complete experimental details.

\begin{table}[t]
\begin{center}
{\footnotesize
\begin{tabular}{r|cac||cccccc|c}
 & (A)& (B) & (C)& \textit{(D)} & \textit{(E)} & \textit{(F)} & \textit{(G)} & \textit{(H)} & \textit{(I)} & us\\ \hline
Any target$^\natural$ & \tickYes & \tickYes & \tickNo & \tickYes & \tickYes & \tickYes & \tickNo & \tickYes & \tickYes & \tickYes\\
Unrestricted loss$^\dagger$ & \tickNo & \tickNo & \tickNo & \tickNo & \tickYes & \tickNo & \tickNo & \tickNo &  \tickNo & \tickYes\\
Weak mutations$^\ddagger$ & \tickNo & \tickYes & \tickYes & \tickNo & \tickNo & \tickNo & \tickYes & \tickNo & \tickNo & \tickYes\\
Non-reflexive neighborhood$^\heartsuit$ & \tickNo & \tickNo & \tickNo & \tickNo & \tickNo & \tickNo & \tickNo & \tickNo & \tickNo & \tickYes\\
Optimal sized neighborhood$^\sharp$ & \tickYes & \tickYes & \tickYes &
\tickNo & \tickNo & \tickNo & \tickNo & \tickYes & \tickNo &
\tickYes\\
Optimal magnitude (mutations)$^\flat$ & \tickYes & \tickYes &
\tickYes & \tickNo & \tickNo & \tickNo & \tickNo & \tickYes & \tickNo &
\tickYes\\
Strictly monotonic evolution & \tickYes & \tickYes & \tickYes & \tickNo
& \tickNo/\tickYes & \tickYes & \tickYes & \tickYes & \tickYes & \tickYes\\
No distribution assumption & \tickNo & \tickYes & \tickNo & \tickYes & \tickYes & \tickNo & \tickYes & \tickNo & \tickYes & \tickYes\\
Unknown distribution & \tickNo & \tickYes & \tickYes & \tickNo &
\tickYes & \tickYes & \tickYes & \tickYes & \tickNo & \tickYes\\
Handles agnostic evolution & \tickNo & \tickYes & \tickNo & \tickNo &
\tickNo & \tickNo & \tickNo & \tickNo & \tickNo & \tickYes\\ 
\hline
\end{tabular}
}
\end{center}
\caption{{\small Main related evolvability results (not restricted to
  vector spaces) and their key properties in the literature (A =
  \cite{mEV}, B = \cite{fRO}, C = \cite{dtOE}, D = \cite{fEF}, E =
  \cite{vEOa}, F = \cite{akAEa}, G = \cite{fDI}, H = \cite{kvwEW}, I =
  \cite{fAC}). References indicated in \textit{italic} can be directly related to evolving vector
  spaces. ($\natural$) There is no assumption on the target function
  except its finiteness (encoding).
($\dagger$) ``Unrestricted'' means that evolution allows twice differentiable 
  non-boolean, non-well-behaved \citep{fDI} and non-quadratic losses. ($\ddagger$) Weak mutations:
  modifications do not depend on the current organism. ($\heartsuit$)
  Non-reflexive: neighborhood does not contain the current organism. ($\sharp$)
  Neighborhood size always $O(\dG)$. ($\flat$) Mutations magnitude is
  $\Omega(\epsilon)$. The gray column relates to the simplest form of organisms addressed, singletons
\citep{fRO}.}}
\label{t-comparison}
\end{table}

\section{Related work and comparison}\label{sec-rel}

In Leslie Valiant's model \citep{vE}, 
evolution has to come with high probability $\epsilon$-close to the
optimum after a polynomial number of iterations. Evolution makes
local modifications to a function that acts as an organism and weakly
minimizes a loss function through a mutator (shown in Figure \ref{f-modglob}). 
There has been a large amount of work in the evolvability model,
summarized in Table \ref{t-comparison} (the Table is discussed in
Section $\S$ \ref{sec-disc}). Row ``non-reflexive
neighborhood'' is a new feature that we have found nowhere else: the fact
that the current organism $f$ does not belong to the mutant set forces the
mutator to evolve $f$ without the safety net that reflexive mutations
belong to neutral neighbors, which therefore
somewhat artificially
contain ``worst case'' evolution.
Perhaps the work that is the closest to
ours with respect to the framework is that of Paul Valiant \citep{vEOa}, which evolves organisms encoding reals instead of binary number
as in \cite{vE} --- a setting arguably closer to natural biological
processes. In \cite{vEOa}, the problem corresponds to the restriction of Figure
\ref{f-modglob} for $\mathcal{X} =
\mathbb{R}^n$ and phenotypes being fixed-degree polynomials. There are two key contributions in the work of \cite{vEOa}.
The first one, which we relate to as the "indirect approach", works under
a broad setting which parallels ours : any convex loss and any
distribution on $\mathcal{X}$. We call it indirect because it relies on the beautiful recording
trick that representations can "hardcode" the optimisation steps of
evolution \citep{fEF,fRO,vEOa}. This trick comes however with some significant
downsides. 
First, the \textit{coding size} of representations
grows at \textit{each} generation and ultimately depend on the desired
accuracy for evolution. In particular, it is polynomial in
$1/\epsilon$, which can be huge. Second, the mutator simulates a weak
optimiser and so the \textit{time complexity} for \textit{each} mutation is also
big, more precisely of the order of the time complexity of the weak
optimisation algorithm it emulates \textit{times} the coding size of
the maximal performance. Third and worse, evolution comes with
\textit{restart}: at each generation, there is a chance that \textit{all} past evolution
history
is wiped out and the representation is initialized to a default one.

The second contribution of \cite{vEOa} is more direct since it trades
the complex mutator for a much simpler and randomized hill climber. However, the analysis is now significantly more restricted as evolution
is proven only for the quadratic loss and the distribution is restricted
to a ball on $\mathcal{X}$. Also, evolution still suffers
downsides as the \textit{coding size expands} at each
generation and the mutator is computationally quite ineffective
and biologically unplausible: the neighborhood to find new mutants size is huge --- polynomial in
$1/\epsilon$ and other factors --- and it resamples its stock of
available mutations at each generation. Finally, neither of \cite{vEOa}'s
schemes are known to be agnostic nor stable in any way --- we note
that stability is an important notion in biology \citep{swFA} but is
not a feature of Valiant's original evolvability model.

Our main result suffers none of these downsides: our mutator meets time
and space optimality properties (Section \ref{sec-disc}), we do not change the set of mutations (Section \ref{sec-proof}), we do not do
restart. Also, our evolvability scheme is agnostic (Section
\ref{sec-agno}), stable (Section \ref{sec-stab}) and
handles significant drift (Section \ref{sec-drift}). Finally, instead of fixed-degree polynomials, we consider any finite-valued function $f: \mathcal{X}
\rightarrow \mathbb{R}^d$. Thus, we can evolve
functions with infinite Taylor expansion, 
something \cite{vEOa} does not cover\footnote{It is also not
  clear whether a simple trick to extend \cite{vEOa} --- replacing
  variables by bounded functions --- is possible without
  endangering the distribution support assumption or the complexity
  parameters.}. Finally, our mutator yields an extremely simple and
provable evolutionary
scheme, implementable using
few lines of code as sketched in Algorithm \ref{algo:SNDS} (\sndS). 

\noindent As a brief comparison with other work, the mutator is not
organism-dependent like in \citep{kvwEW}, we have no distribution
assumptions like in \citep{akAEa,kvwEW} or a requirement to know this
distribution like in \citep{fEF}, and the same scheme can be made agnostic or
handle drift more significant than some allowed in more restricted
settings \citep{kvwEW}. We
insists on the no-distribution assumption: in some
work, this distribution is constrained, smooth and nice \citep{akAEa},
uniform \citep{mEV,vE}, spherically symmetric \citep{kvwEW}, a product
of Gaussians with polynomial variance \citep{kvwEW}, or with
support restricted to a ball \citep{vEOa}.
\begin{algorithm}[t]
\caption{\sndS($f_0$)}
\begin{algorithmic}
{\small
\label{algo:SNDS}
\STATE \textbf{Input}: initial representation $f_0, T$;
\STATE For $t=0, 1, ..., T-1$
\STATE \quad Step $t$.1: Sample $\mathcal{S}$ as per (\ref{defmm-short});
\STATE \quad Step $t$.2: Compute $\bene(f_t)$ and $\neut(f_t)$ using
$\neigh_\epsilon(f_t)$ as per (\ref{defneighbor})
\STATE \quad Step $t$.3: \textbf{if} $\bene(f_t) \neq \emptyset$ then sample uniformly
$f_{t+1} $ in $\bene(f_t)$ as per (\ref{defbene});
\STATE \quad \qquad\qquad \textbf{else if} $\neut(f_t) \neq \emptyset$ then
sample uniformly
$f_{t+1} $ in $\neut(f_t)$ as per (\ref{defneut});
\STATE \quad \qquad\qquad \textbf{else} sample uniformly
$f_{t+1} $ in $\neigh_\epsilon(f_t)$;
\STATE \textbf{Return} $f_{T}$;
}
\end{algorithmic}
\end{algorithm}

\section{Evolvability model}\label{sec-evol}

We define key components of the Evolvability model and then
define the model \citep{vE}.

\paragraph{Topology of representations ---}Organisms are represented by
functions $f:{\mathcal{X}} \rightarrow
{\mathbb{R}}^{d}$ of a set ${\hg}$ (Figure \ref{f-modglob}), called the representation class,
supposed to be polynomial-time Turing-evaluatable.  ${\mathcal{X}}$ is the set of
conditions or experiences.
For any $f \in
{\hg}$, a neighborhood function is defined, $\neigh_\epsilon(f) \subseteq
{\hg}$, that depends on an accuracy parameter
$\epsilon>0$. The size of the neighborhood is required to be
polynomial in $1/\epsilon$, $d$ and the dimension of ${\mathcal{X}}$,
$\mathrm{dim}({\mathcal{X}})$. 

\paragraph{Performances of representations ---}Performances are
measured with respect to an unknown but fixed distribution
${\mathcal{D}}$ over ${\mathcal{X}}$, relatively to an unknown target function
$t\in {\hg}$. The expected performance of some $f\in
{\hg}$ with respect to $t$ is $\perf_{t,\varphi}(f, {\mathcal{D}})
\defeq - \expect_{x \sim {\mathcal{D}}} [D_\varphi(f(x) \| t(x))]$,
where $D_\varphi$ is Bregman divergence with (twice differentiable) generator $\varphi: {\mathbb{R}}^d \rightarrow
{\mathbb{R}}$ \citep{bTR,bgwOT,bmdgCWj,bnnBV}. By extension, the empirical performance
realized by $f$ on an i.i.d. sample ${\mathcal{S}}$ is defined as
$\perf_{t,\varphi}(f, {\mathcal{S}}) \defeq - \expect_{x \sim {\mathcal{S}}}
[D_\varphi(f(x) \| t(x))]$. Our expected performance $\perf_{t,\varphi}$ generalizes Valiant's
which computes $\expect_{{\mathcal{D}}} [ft]$. In Valiant's setting,
$d=1$, the output of functions is $\{-1,1\}$ and $D_\varphi$ is
the square loss $D_{\sql}$, and so $\perf_{t,\sql}(f,
{\mathcal{D}}) = 2 (\expect_{{\mathcal{D}}} [f(x)t(x)] - 1)$. Notice
that unlike \cite{vEOa}, we do not assume to know the functions defining phenotypes in
Figure \ref{f-modglob} (grey parallelogram), we just observe the
combination of their outputs. We consider it very natural, some sort of ``Petri dish''
model of performance evaluation.

\paragraph{Selection by mutations ---}A mutator $\mutator: {\hg} \times {\mathbb{N}} \rightarrow  {\hg}\cup
\{\bot\}$ 
is a randomized
polynomial-time Turing machine that depends upon an accuracy $\epsilon
> 0$ and a tolerance $\tol>0$. Tolerance is required to be polynomial
in $\epsilon$, $1/d$ and $1/\mathrm{dim}({\mathcal{X}})$. The mutator
returns a so-called ``mutant'' of some
input $f\in  {\hg}$ based on a weak evaluation of the quality of the
elements of $\neigh_\epsilon(f)$. More precisely, it takes as
input a
sample size $m>0$, samples i.i.d. a set ${\mathcal{S}}$ of $m$ conditions, 
and outputs some $g\in \bene(f)$ at random if $\bene(f)\neq \emptyset$,
or else $g\in
\neut(f)$ at random if $\neut(f)\neq \emptyset$, using a
fixed distribution $\upmu(g,f)$ with support $\bene(f)$ or $\neut(f)$. 
Those two sets $\bene(f)$
and $\neut(f)$ are defined respectively by:
\begin{eqnarray}
\bene(f) & \defeq & \{g \in \neigh_\epsilon(f): \perf_{t,\varphi}(g,
{\mathcal{S}}) \geq \perf_{t,\varphi}(f,
{\mathcal{S}}) + \tol\}\:\:,\label{defbene}\\
\neut(f) & \defeq & \{g \in \neigh_\epsilon(f): |\perf_{t,\varphi}(g,
{\mathcal{S}}) - \perf_{t,\varphi}(f,
{\mathcal{S}}) |< \tol\}\:\:.\label{defneut}
\end{eqnarray}
If both sets $\bene(f)$ and $\neut(f)$ are
empty, the mutator outputs $\bot$, meaning 
evolution has failed. 

\paragraph{Representations ---} Our framework being non-boolean, we
define the models that we evolve. Let $[n] \defeq \{0, 1, ..., n\}$ and
$[n]_* \defeq \{1, ..., n\}$ where $n$ is a natural integer. First, we have a set of functions
$\{g_1, g_2, ..., g_{\dG}\} \subset {\hg}$, each of which is of the form $g_j : {\mathcal{X}} \rightarrow
{\mathbb{R}}^d$ for $j\in [\dG]$ for some $\dG>0$. Each $g_j$ can be
thought as encoding a specific part of trait(s) representation, 
such as a relative concentration in specific proteins under any
experimental condition --- for this reason we suppose without loss of
generality that its norm $\|g_j(.)\|_2^2$ is finite almost everywhere
with respect to $\mathcal{D}$. The set of functions that we evolve lies in
the span of $\{g_1, g_2, ..., g_{\dG}\}$ --- which we also denote as $\hg$
for simplicity ---, \textit{i.e.}, consists of linear combinations of
functions of $\hg$. 
However, because we want our model to be general, we do not evolve
directly $\hg$. For this reason, we define a set of vectors ${\mathcal{B}} \defeq
\{b_1,
b_2, ..., b_{\dF}\}$, with $b_i \in \mathbb{R}^{\dG}$, that will
represent our set of mutations. Each (column) vector, $b_i \defeq
[b_i^1 \cdots b_i^{\dG}]^\top$ maps to a function $b_i(x) \defeq \sum_j
b_i^j \cdot g_j(x)$. Their span, $\mathrm{span}(\mutbasis)$, defines a
subspace of the phenotype vector space. While we will
investigate first the case $\mathrm{span}(\mutbasis) =\hg$, we shall
also cover the "agnostic" evolvability case where 
$\mathrm{span}(\mutbasis) \subset \hg$. An organism $f$ that our mutator builds
has evolved from some initial $f_0$ and can therefore be represented
as $f = f_0 + \sum_i f^i b_i$, where $[f^1 \: f^2 \: ... \: f^{\dF}] \in
{\mathbb{N}}^{\dF}$. Finally, To avoid confusion with ${\mathbb{R}}^d$, we let
$\|.\|_{\hg}$ denote the $L_2$ norm computed with respect to $\{g_1,
g_2, ..., g_{\dG}\}$, \textit{i.e.} the norms of the coordinates in $\hg$.

\paragraph{Evolvability horizon ---}  
In the same way as PAC-learnability allows to be polynomial in the size of
the target concept, evolvability has to allow a time complexity that
depends on some complexity measure with respect to the target
organism, and not just the number of description variables, which
would be $\mathrm{dim}({\mathcal{X}})$ in our case. Evolvability results
involving complex representations alleviate this
distinction by putting constraints on representations \citep{kvwEW,vEOa}. 
We integrate this notion in the form of what we call the \textit{Evolvability
horizon}, $\darwin$. $\darwin$ quantifies the necessary number of generations to come up with an encoding
``close'' to that of $t$. Any evolution using
$o(\darwin)$ generations, using only $\mutbasis$, would be bound to fail in the worst case.
\begin{definition}
The \textit{Evolvability horizon} $\darwin(f_0,t)$ ($\darwin$ for short) of $f_0$ wrt target $t$ is
$\darwin \defeq \left\lceil \frac{\|t-f_0\|_{\hg}}{\max_i
   \|b_{i}\|_{\hg}} \right\rceil$.
\end{definition}

\paragraph{Evolvability ---} We define
evolvability, following \citep{fRO,vE}.

\begin{definition}\label{defevolvability}
Assume the following fixed, for any accuracy $\epsilon>0$:
representation class ${\hg}$, mutator
neighborhood $\neigh_\epsilon$ and distribution $\mu$, distribution ${\mathcal{D}}$, generator
$\varphi$, tolerance $\tol$. Then ${\hg}$ is distribution-free
evolvable by mutator $\mutator(.,.)$ iff for any initial
representation $f_0$ and target representation $t$ such that $\darwin
\ll \infty$ ("$\ll \infty$" means finite), there exist
polynomial functions
$m$ and $T$ (both polynomial in $d,\mathrm{dim}({\mathcal{X}}),1/\epsilon,
\darwin$) such that $\forall 0<\epsilon \leq 1$, 
then  with probability $\geq
1-\epsilon$, the sequence $f_0, f_1, ..., f_{T}$ with $f_j \defeq
\mutator(f_{j-1}, m)$, $\forall j$, satisfies:
\begin{eqnarray}
\perf_{t,\varphi}(f_{T},
{\mathcal{D}}) & \geq & - \epsilon\:\:.\label{condgen}
\end{eqnarray}
\end{definition}
Our model of evolution holds without initialisation \citep{fEF,vE}. We
have simplified the presentation of the model, in particular removing
the notion of ``evolution algorithm'' and specifying evolution
directly from the mutator. This does not weaken the results.

\section{Evolvability of vector spaces}\label{sec-proof}
The main notations are summarized in
Appendix, Subsection \ref{app-sub-not}.
\begin{definition}\label{defPERMUT}
A mutator is \textbf{permissible} iff the neighborhood
used by the mutator is defined as:
\begin{eqnarray}
\neigh_\epsilon(f) & \defeq & \{f\}\oplus \{\sigma \alpha b_i: \sigma \in
\{-1,+1\}, b_i \in {\mutbasis}\}\:\:,\label{defneighbor}
\end{eqnarray}
for some set ${\mutbasis} \defeq \{b_1, b_2, ...,
b_{\dF}\}$, where $\oplus$ is
Minkowski sum. 
$\alpha>0$ (fixed) is called the \textit{magnitude} of the
mutations and $\sigma$ is called the polarity of the
mutation.
\end{definition}
We have not detailed the distribution of the mutator, $\upmu$ (Section
\ref{sec-evol}). In fact,
it can be any distribution with full support and (at least) inversely
polynomial density, following \textit{e.g.}, \citep{kvwEW}. We shall
consider the simplest of all, the uniform distribution. 
We
remark that ${\mutbasis}$ is not necessarily a
basis, nor normal, nor orthogonal. Also, $\alpha$ is the key parameter
to be tuned for evolvability. We evolve vector spaces under three
assumptions united in a Singularity-Free (SF) setting.

\begin{definition}
The (\textbf{SF}) setting is defined by the following three
assumptions:\\
\noindent (i) any genome ``can be coded'': $\mathrm{span}(\mutbasis) = {\hg}$, \\
\noindent (ii) any target organism $t$
is ``unique'': $\arg\min_f \perf_{t,\varphi}(f,
{\mathcal{D}}) = \{t\}$ ($\forall t\in \hg$), and\\
\noindent (iii) any non-void genome gets ``expressed'': $\pr_{x\sim \mathcal{D}}[g(x) \neq 0_{{\mathbb{R}}^d}]
  > 0, \forall g\neq 0_\hg$.
\end{definition}
Each of (i-iii) allows to define parameters that will be useful to quantify
evolution. 
We now provide a concise
version of our main results, hiding the less important parameters in
the corresponding $\theta, \tilde{O}$ notations (the complete statement of the Theorem is in Theorem \ref{evolmut-long}).
\begin{theorem}\label{evolmut-short}
(evolvability of vector spaces, concise statement) Assume (\textbf{SF}) holds. Then ${\hg}$ is distribution-free evolvable by
any permissible mutator $\mutator$, with tolerance $\tol =
\theta(\epsilon^{2})$ and magnitude of mutations $\alpha =
\theta(\epsilon)$. 
The number of conditions sampled at each iteration satisfies:
\begin{eqnarray}
m & = & \tilde{O}\left(\frac{ \darwin^4 }{\epsilon^{2}}\log\left(\frac{\dF
      \darwin}{\epsilon}\right)\right)\label{defmm-short}\:\:.
\end{eqnarray}
Finally, the number of evolution steps $T$ sufficient to comply with
ineq. (\ref{condgen}) is $T = \tilde{O}\left(\darwin^4/\epsilon^{2}\right)$.
\end{theorem}
We sketch the key steps for the proof of Theorem
\ref{evolmut-short}, first introducing key
definitions \ref{defretprem} and \ref{defexen} below.
\begin{definition}\label{defretprem}
For any representation $f$, condition $x$, magnitude $\alpha$ and
polarity $\sigma$, we let
\begin{eqnarray}
\mathrm{R}_{f, i} (x) & \defeq & \langle \sigma b_i(x),
(\nabla\varphi \circ t)(x) - (\nabla\varphi \circ f)(x)\rangle\:\:, \label{defreturn} \mbox{ and }\\
\Pi_{f,i}(x) & \defeq &
\frac{1}{\alpha} D_\varphi(f (x)
- \alpha \cdot (-\sigma) b_i (x) \| f (x)) \label{defpremium}
\end{eqnarray}
denote respectively the mutator's \textbf{return} and \textbf{premium} on $x$ given $f$, omitting $\sigma$ and
$\alpha$ in notations.
\end{definition}
We give an equivalent
definition for the set of beneficial mutations, using returns and
premiums.
\begin{lemma}\label{lemmut}
$\bene(f) = \left\{f+\sigma \alpha b_i \in \neigh_\epsilon(f): \expect_{{\mathcal{S}}}
[\mathrm{R}_{f, i}(x)] -\expect_{{\mathcal{S}}}
[\Pi_{f,i}(x)] \geq (\tol/\alpha)\right\} $.
\end{lemma}
(Proof in Appendix, Subsection \ref{prooflemmut}) We call Lemma
\ref{lemmut} the \textit{mean-divergence} decomposition of beneficial
mutations in reference to portfolio theory: when $D_\varphi$ is Mahalanobis
divergence, the model simplifies to an equivalent of Markowitz model
\citep{mPS} in which $\Pi_{f,i}(x) = (\alpha/2) \cdot \langle b_i(x),
\matrice{m} b_i(x)\rangle$ ($\matrice{m}$ symmetric positive
definite), \textit{i.e.} the magnitude of mutations
$\alpha$ is exactly Arrow-Pratt measure of absolute risk
aversion --- since $\alpha > 0$, evolution is "risk
averse". Due to the lack of space, we close this analogy here and
notice that the mutator's premium 
quantifies the (local) risk of mutating (it does not depend on $t$). 
Hereafter, we let ``expected'' return and ``expected'' premium denote the
expectation of (\ref{defreturn}) and (\ref{defpremium}). There is a
quantity that turns out to be key in our analysis. It ties expression
(phenotype) and encoding (genotype), the \exen-ratio.
\begin{definition}\label{defexen}
The phenotype-to-genotype (\exen) ratio of $f \in {\hg}$
given distribution ${\mathcal{D}}$ is:
\begin{eqnarray}
\rho(f|{{\mathcal{D}}}) & \defeq & \frac{\expect_{{\mathcal{D}}}
[\|f(x)\|^2_2]}{\|f\|_{{\hg}}}  = \frac{\var_{{\mathcal{D}}}
[\|f(x)\|_2] + (\expect_{{\mathcal{D}}}
[\|f(x)\|_2])^2}{\|f\|_{{\hg}}} \:\:. \label{ratgp}
\end{eqnarray}
The \exen-divergence between $f \in {\hg}$ and $g \in
{\hg}$ given distribution ${\mathcal{D}}$ is $\rho(f, g|{{{\mathcal{D}}}})
\defeq \rho(f-g|{{\mathcal{D}}})$.
\end{definition}
A justification for the name of $\rho(f, g|{{{\mathcal{D}}}})$ comes
from the fact that 
$\rho(f|{{\mathcal{D}}}) =
\rho(f,0_{{\hg}}|{{{\mathcal{D}}}})$, $0_{{\hg}}$ representing a ``void genome''.
The proof of Theorem \ref{evolmut-short} relies on \textit{two arguments}. The first establishes
that, provided the mutator samples sufficient conditions $m$, the set of
beneficial mutations is never empty with high probability, as long as the current $f$ is
``far'' from the optimum, where this distance notion relies on the
\exen-divergence between $f$ and the target $t$. In fact, we
show a bit more, as in this case mutations
may be \textit{superior beneficial}: we call them
superior beneficial because while beneficial mutations shall be proven to yield an improvement
of $\Omega(\epsilon^2)$ in performance, those superior beneficial mutations yield
a greater increase of $\Omega(\epsilon)$. 
In the second argument, we show that when the first argument does not
hold anymore, the requirements of evolvability are met, and the number of evolution steps is polynomial
in all required parameters, so vector spaces are
evolvable. To formalize these two arguments, we need to define an
important basis in $\mutbasis$, called
$\mutbasis^*$.
\begin{definition}\label{defBSTAR}
Let $\mutbasis^*\subseteq \mutbasis$ be a basis that maximises
$\overline{B}_{\hg}(\mutbasis') \defeq B_{\hg}(\mutbasis') / \max_{b
\in \mutbasis'} \|b\|_\hg$ over all bases $\mutbasis' \subseteq \mutbasis$, where $B_{\hg}(\mutbasis') \defeq (1-\kappa_n(\mutbasis'))\cdot
(1 - \kappa_a(\mutbasis')) \cdot \frac{\sum_{b \in \mutbasis'}
  \|b\|_{\hg}}{\dG}$ and $\kappa_n(\mutbasis')$ and $\kappa_a(\mutbasis')$
denote non-negative reals such that $G(\mutbasis') = (1-\kappa_n(\mutbasis')^{\frac{2}{\dG}}
A(\mutbasis')$, $\cos(\thetaV) = 1 - (1-\kappa_a(\mutbasis'))^{\frac{1}{\dG-1}}$.
Here, $G(.)$ and $A(.)$ are the geometric and
arithmetic means\footnote{The \textit{G}eometric and \textit{A}rithmetic means of reals
$z_1, z_2, ..., z_m$ are $G \defeq (\prod_j z_j )^{1/m}, A \defeq (1/m)
\cdot \sum_j z_j$.} of the squared norms in $\mutbasis'$, and $\thetaV
\defeq \min_{b\neq  b' \in \mutbasis'}\min\{|\angle b, b'|, |\pi - \angle b,
b'|\} \in [0, \pi/2]$ is the minimal angle between two 
vectors. 
\end{definition}
Roughly, the larger $B_{\hg}(.)$, the
better for evolution. This parameter tends to be larger as the basis in argument becomes closer to
orthonormality, and so orthonormal bases represent the ``easiest'' cases for
evolution from this standpoint. They turn out to be the ones of \citep[Section 6]{kvwEW}.
\begin{lemma}
$\forall {\mutbasis}'$ a basis of $\hg$, $B_{\hg}(\mutbasis')$ and
$\overline{B}_{\hg}(\mutbasis')$ are strictly positive.
\end{lemma}
\begin{proof}
If $\kappa_n({\mutbasis}') = 1$, then one vector in ${\mutbasis}'$
is the null vector, if $\kappa_a({\mutbasis}') =1$, then two vectors in
${\mutbasis}'$ are collinear, in whichever case ${\mutbasis}'$
cannot be a basis. Thus, 
$(\kappa_n({\mutbasis}'), \kappa_a({\mutbasis}')) \in
[0,1)^2$. Finally, no basis vector can be the null vector, so
$B_{\hg}(\mutbasis') > 0$ and $\overline{B}_{\hg}(\mutbasis') > 0$, as claimed.
\end{proof}
Wlog, we are assume
$\max_{i\in [\dF]_*}
   \|b_{i}\|_{\hg} = \max_{b_i \in \mutbasis^*}
   \|b_{i}\|_{\hg}$, and $\sup_{i,x} \single(\mutbasis'\cup \{b_i\}, x) \leq \sup_{x}
\single(\mutbasis^*, x)$ for any basis $\mutbasis' \subseteq
\mutbasis$ and $b_i\in \mutbasis$, where $\single(\mutbasis', x) \defeq \sum_{b_i \in \mutbasis'}
 \|b_{i} (x)\|^2_2$ for any $\mutbasis' \subseteq \mutbasis$ and any $x\in {\mathcal{X}}$. These
 assumptions simplify derivations without restricting our results. We
 also assume $\sup_{x}
\single(\mutbasis^*, x)$ polynomial in all genome parameters,
$d,\mathrm{dim}(\mathcal{X})), \darwin$, in order not to laden the
polynomial dependences of the  evolvability model by one which takes into
account the maximal magnitude of expressions. 
We denote $\gendistnone \defeq \expect_{x \sim {\mathcal{D}}} [\matrice{g}^\top_x 
\matrice{g}_x]$ where $\matrice{g}^\top_x$ stacks all $g_i(x)$ in column, and $0\leq \hessianmin\leq \hessianmax\ll\infty$
(resp. $0\leq \expressmin\leq \expressmax
  \ll \infty$) the min/max eigenvalues of the Hessian of $\varphi$
  (resp. $\gendistnone$). 
\begin{lemma}\label{lemphi}
Under (\textbf{SF}), $\hessianmin > 0$ and $\expressmin > 0$.
\end{lemma}
Indeed, if $\hessianmin = 0$, then (ii) in (\textbf{SF}) would be
violated. If
$\expressmin = 0$, some genomes would get expressed only on conditions sets of
zero measure, violating (iii) in (\textbf{SF}). We define four
sets of organisms. The first, ${\hg}_t$, is the set complying with
evolvability requirements as in ineq. (\ref{condgen}). The second, ${\hg}_{t,
  {\mathcal{D}}}$, is the set of organisms "far enough" from target
relatively to the \exen-divergence.
\begin{definition}\label{defSet1}
Let ${\hg}_t \defeq \{f \in {\hg} : \perf_{t,\varphi}(f,
{\mathcal{D}})\geq -\epsilon\}$,  and 
\begin{eqnarray}
{\hg}_{t, {\mathcal{D}}} & \defeq & \left\{f \in
   {\hg} : \exists b_i \in \mutbasis^*, \rho(f, t|{{{\mathcal{D}}}}) \geq \frac{\sqrt{e \dG}}{\hessianmin B_{\hg}({\mutbasis}^*)}\left(\tau
  + \alpha
{\hessianmax} \cdot \|b_i\|_{\hg} \cdot \rho(b_i|{{{\mathcal{D}}}}) + \frac{\tol}{\alpha}
\right)
 \right\}\label{defftd}\:\:,
\end{eqnarray}
where $\alpha, \tau, \tol >0$ are fixed beforehand. Let
${\hg}_{t, {\mathcal{D}}} \supseteq \hgmon \defeq \{f_j : f_{j'} \in {\hg}_{t, {\mathcal{D}}},
\forall j'\in [j]\}$, where $f_j \defeq \mutator(f_{j-1}, m)$, and $\hgnotmon \defeq \{f_j\}_{j=1}^T \backslash
\hgmon$ the sequence ``following'' $\hgmon$.
\end{definition}
Hence, $\hgmon$ is the longest prefix sequence of evolved
organisms that are all in ${\hg}_{t, {\mathcal{D}}}$. Note that we do not
assume that $\hgmon\neq \emptyset$. A key property of
$\hgmon$ is that its sequence of organisms has strictly monotonically
increasing performances and yields with high probability non-empty
beneficial sets. This is our first argument to the proof of Theorem \ref{evolmut-short}.
\begin{theorem}\label{th00}
Assume (\textbf{SF}) holds. Suppose $\alpha, \tau, \tol >0$ fixed, and
mutator is run for $T>0$ iterations, sampling at each iteration a number of conditions
\begin{eqnarray}
m & = & \Omega\left(\frac{{\hessianmax}^2  \sup_{x} \single^2(\mutbasis^*, x)}{\tau^2}\left(\frac{\hessianmax
        \expressmax}{\hessianmin \expressmin} \cdot \frac{\darwin^2}{\overline{B}^2_{\hg}(\mutbasis^*)} +\alpha^2\right) \log\left(\frac{\dF
      T}{\epsilon}\right)\right) \label{numevom22}\:\:.
\end{eqnarray}
Then, recalling that $f_j \defeq \mutator(f_{j-1}, m)$, the following holds true with probability $\geq
1-\epsilon$:
\begin{eqnarray}
\bene(f_j)&  \neq & \emptyset\:\:,\forall f_j \in \hgmon\:\:. \label{defalphatauzerostart}
\end{eqnarray}
\end{theorem}
\begin{lemma}\label{lemconv1}
If $\hgnotmon\neq \emptyset$, then let $\hgnotmon\defeq
\{f_j\}_{j=j_{\star}}^T$ for some $j_{\star} \in [T]^*$. Then 
$f_{j_\star}
\in {\hg}_t$.
\end{lemma}
(Proofs respectively in Appendix, Subsection \ref{proofth00} and Subsection \ref{prooflemconv1}) Hence, the first element in $\hgnotmon$ complies with
evolvability requirements in eq. (\ref{condgen}).
What remains to be
shown is that as long as the current organism $f_j$
stays in $\hgmon$ ($\subseteq {\hg}_{t, {\mathcal{D}}}$), the performance increases by a
substantial amount, guaranteeing that the following scenario
occurs: either at some point it escapes $\hgmon$, in which case Lemma
\ref{lemconv1} guarantees that evolvability requirements are met, or it never escapes ${\hg}_{t, {\mathcal{D}}}$ and
after a polynomial number of iterations, it satisfies evolvability
requirements as well, achieving our second argument. This is shown in the following Lemma.
\begin{lemma}\label{lemmpack}
Fix $\tol$ and $\alpha$
as in (\ref{defTol}) and (\ref{defAlpha}), and in eqs
(\ref{defftd}, \ref{numevom22}) let $\tau \defeq \theta(\epsilon/U)$,
with $U \defeq (\hessianmax^{\frac{3}{2}}
\expressmax^{\frac{1}{2}}/(\hessianmin^{\frac{3}{2}}
\expressmin^{\frac{1}{2}})) \cdot (2\sqrt{e
  \dG}/\overline{B}_{\hg}({\mutbasis}^*)) \cdot
\darwin$. 
Then, with probability $\geq
1-\epsilon$, $(f_j \in \hgmon) \Rightarrow \perf_{t,\varphi}(f_{j+1},
{\mathcal{D}}) = \perf_{t,\varphi}(f_j,
{\mathcal{D}})  + \Omega(\epsilon^{2}), \forall j \in [T]_*$. 
Finally, the number of evolution steps sufficient for $\hgmon$ to comply with (\ref{condgen})
is $T = \tilde{O}\left(\darwin^4/\epsilon^{2}\right)$.
\end{lemma}
(Proof in Appendix, Subsection \ref{prooflemmpack}) This completes the
proof sketch of Theorem \ref{evolmut-short}. Section \ref{app-add-main} in the
Appendix (pages \pageref{app-add-main}-\pageref{app-add-res}) provides the complete proof.

\section{Agnostic evolvability of vector spaces} \label{sec-agno}

One important question
 is what happens when $t$ cannot be evolved from
$\mutbasis$, \textit{i.e.} when alleviating condition (i) in setting (\textbf{SF}). Ideally, we would like evolution to converge
to the ``best'' evolvable organism in terms of performances. To our
knowledge, few positive result exist in the agnostic / improper evolvability
model \citep{akAEa,fRO}, and the most unrestricted one holds for
extremely simple representations: singletons \citep{fRO}. Neither
evolving schemes of \cite{vEOa} are known to be agnostic. The proof of
the following Theorem appears in Appendix, Section \ref{proofthrobust}.
\begin{theorem}\label{throbust}
If we relax assumption (i) in (\textbf{SF}) but keep (ii, iii), then 
Theorem \ref{evolmut-short} holds mutatis mutandis with the replacement of
ineq. (\ref{condgen}) by:
\begin{eqnarray}
\perf_{t,\varphi}(f_{T},
{\mathcal{D}}) & \geq & \sup_{f
  \in \mathrm{span}(\mutbasis)} \perf_{t,\varphi}(f,
{\mathcal{D}}) - \epsilon\:\:.\label{condgenrob}
\end{eqnarray}
\end{theorem}

\section{Beyond Valiant's evolvability: stable evolvability of vector spaces} \label{sec-stab}

The evolution of phenotypes is an important but still poorly understood phenomenon,
mixing both the tendency to adapt to changing environments and the
"need" to remain the same \citep{swFA}. In Valiant's evolvability
model, one would expect the later
constraint to prevail when the organism is close enough to the optimum
--- a good mutator should not jump from a
near-optimal organism to a highly suboptimal one, except perhaps
with sufficiently small probability. Valiant model does
not take this into account: all that is required is to probably "hit"
the $\epsilon$-closedness "ball" around the optimum in polynomial
time (ineq. (\ref{condgen})) and with high probability. Nothing is
required for what happens "next". This is not a desirable
feature since it does not preclude one of the poorest mutators
--- a coin --- to be efficient: if there are two organisms in a set
$\mathcal{Z}$, say $A$
(optimal) and $b$ (highly suboptimal), then any coin "mutator" with
fixed bias (even favouring the choice of $b$) trivially "evolves" $\mathcal{Z}$ : $O(\log(1/\epsilon))$ mutations
suffice to hit $A$ with high probability. If
we were to require evolution to "stay" $\epsilon$-close to the
optimum for a certain number of iterations when it satisfies the
conditions of evolvability, this would considerably impede such poorly
informed mutators. To our knowledge, there is no such
\textit{stability} result in Valiant's model. We now provide such a result.
\begin{theorem}\label{thmtrap-short}
Assume (\textbf{SF}) holds. Let $N \in \mathbb{N}_*$. Then 
Theorem \ref{evolmut-short} holds mutatis mutandis with the replacement of
ineq. (\ref{condgen}) by: $\perf_{t,\varphi}(f_{T+j},
{\mathcal{D}}) \geq - \epsilon$, $\forall j \in [N]$.
\end{theorem}
(Proof in Appendix, Section \ref{proofthmtrap}) Note that stability
can be controlled by tuning $N$.

\section{Evolvability of vector spaces with target drift} \label{sec-drift}

Another important
question is what happens when the target organism drifts slowly
\citep{kvwEW}. 
In \citep{kvwEW}, there is a sequence of targets $t_0, t_1,
...$ and the objective is to replace the static requirement in
ineq. (\ref{condgen}) by one which takes into account the last target $t_T$, when
$t_i$ is allowed to slightly drift with respect to $t_{i-1}$ with
respect to its performances. In our
case, since we separate the encoding from computing performances, we
allow the encoding to drift, which is perhaps more natural ---
drift affects genotype before performances. Also, the evaluation of
beneficial and neutral mutations in $\bene$ and $\neut$ are done for $f_{j-1}$ with respect to $t_{j-1}$.
\begin{theorem}\label{thdrift-short}
Assume (\textbf{SF}) holds and the target organism sequence $t_0, t_1, ...$ drifts
according to $\|t_{i+1} - t_i\|_\hg \leq \tilde{O}\left(\epsilon^{4} /
  (\darwin^2 \max_{i'\in [\dF]_*}
   \|b_{i'}\|^2_{\hg})\right)$, $\forall i\geq 0$. 
Then 
Theorem \ref{evolmut-short} holds mutatis mutandis with the replacement of
ineq. (\ref{condgen}) by: $\perf_{t_T,\varphi}(f_{T},
{\mathcal{D}}) \geq - \epsilon$.
\end{theorem}
(Proof in Appendix, Section \ref{proofthdrift}) Up to factors that
depend upon $\hessianmin, \hessianmax, \expressmin,
\expressmax$, our model of drift is equivalent to the performance
drift model of \citep{kvwEW} (see Lemma \ref{lemfund1} below), yet, as
a function of $\epsilon$, we
tolerate drifts that are larger by factor $\theta(1/\epsilon^2)$ than
theirs, taking as reference their result on the weakest distribution
assumptions (product Gaussians), assumptions that we also
alleviate. Dependence on $\darwin$ is necessary up to some extent,
as otherwise worst-case drifts would defeat evolution by
artificially increasing the actual horizon to the last target
($\darwin$ is computed for $t_0$). It is important to note that
\cite[Theorem 8, Corollary 9]{kvwEW} show that the
\textit{strict monotonicity} of evolution --- that is, the fact that
performance satisfies some minimal strictly positive increment from one generation on to
the next --- is sufficient to grant
\textit{some} resistance against drift. From this standpoint, the
indirect approach (Section \ref{sec-rel} above) of \cite{vEOa} is not a good candidate since it does
restart. The direct approach however, which operates with the square loss and on a
ball-supported distribution, displays strict monotonicity\footnote{We
  already proved that our algorithm is also strictly monotonic, see
  Theorem \ref{th00}.}, yet a direct application of \cite[Theorem 8]{kvwEW} shows drift resistance with poor dependence on the degree of the
evolved polynomial,
\textit{i.e.} of order
$1/(\dim(\mathcal{X}))^{\mbox{\tiny{degree}}}$ in the worst case
\citep[Proof of Theorem 3.3]{vEOa}.

\section{Toy experiments}\label{sec-toy-exp}

\begin{table}[t]
\begin{center}
\begin{tabular}{cc}\hline\hline
\includegraphics[width=0.36\columnwidth]{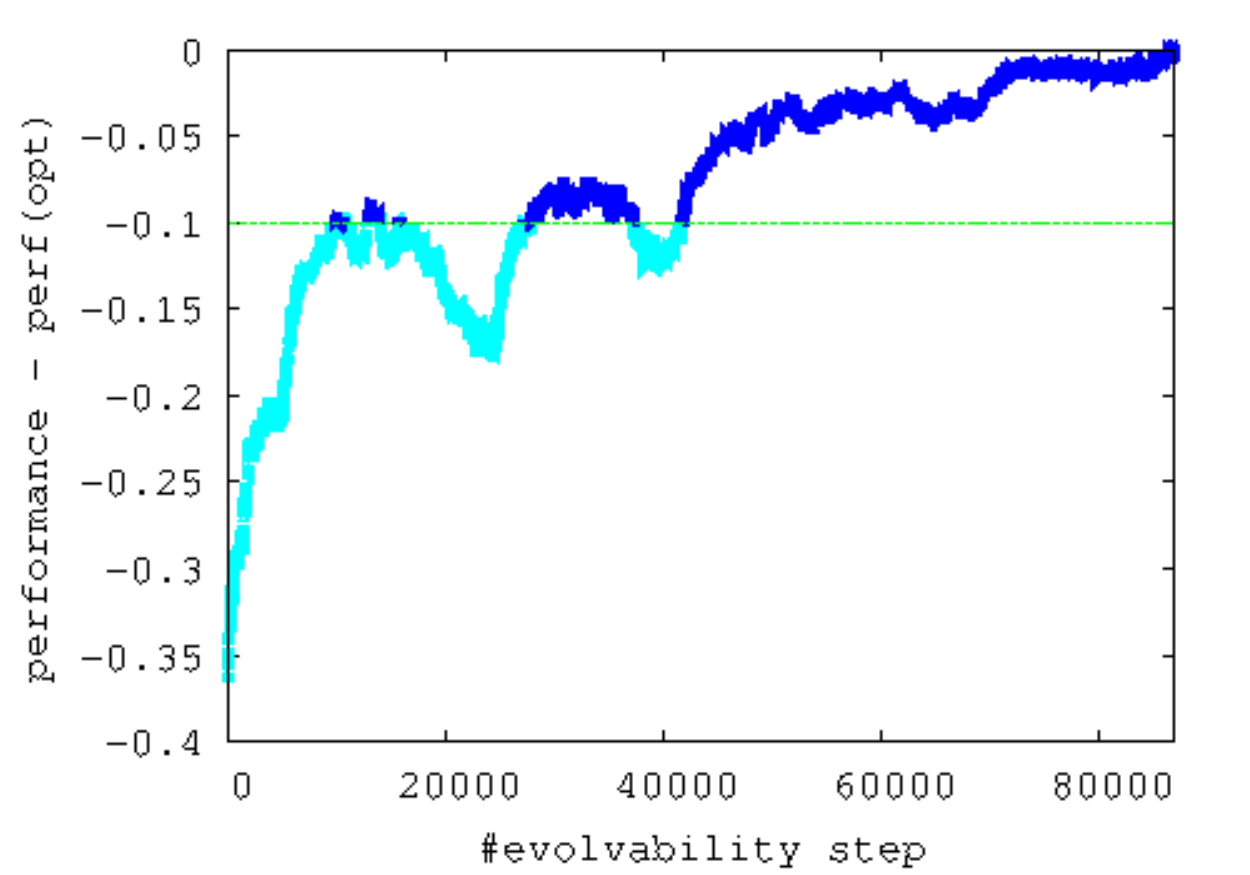} & \includegraphics[trim=0bp 130bp 0bp
150bp,clip,width=0.27\columnwidth]{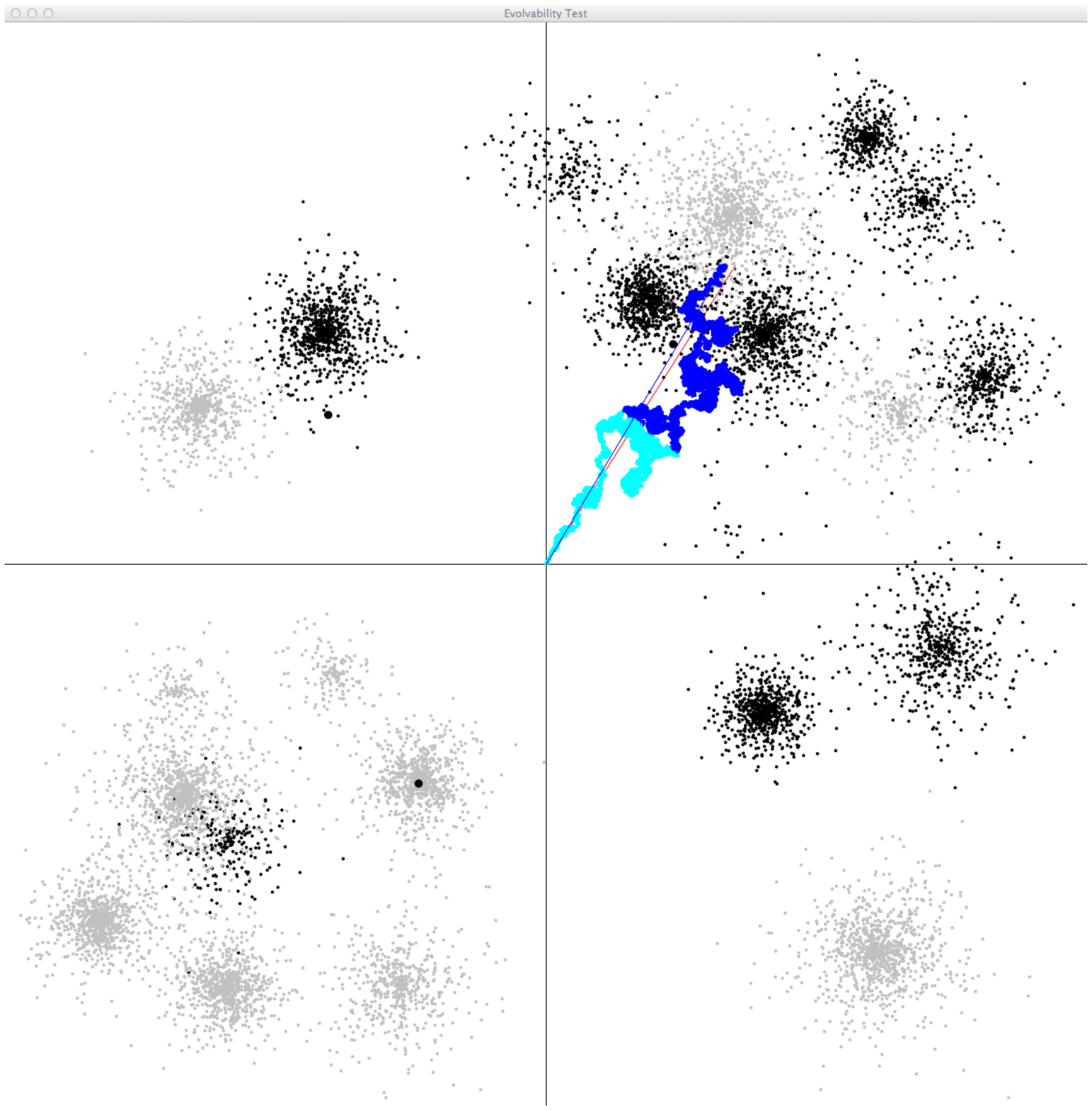}\\ \hline\hline
\end{tabular}
\end{center}
\caption{{\small Toy supervised learning experiment. Left: performances, indicating
  the $-\epsilon$ evolvability threshold in green. The data (shown on
  the right: black / grey = positive / negative class) are not linearly separable, so we compute the
  performance minus that of the optimal linear separator
  (eq. (\ref{condgenrob})). Organisms
  evolved are also displayed on the right (the cyan / dark blue path = path followed by the
  arrowhead of the vector learned, starting from the null vector). Dark blue
  means that the organism satisfies the condition of the evolvability model
  (eq. (\ref{condgenrob})). A red vector displays the target, almost
  confounded with the last vector evolved (dark blue line, best viewed in color).}}
\label{t-exp-short}
\end{table}
We complete our results with preliminary toy experiments in
supervised and unsupervised learning, displaying some promising
directions for Evolvability to spin out provable stochastic
gradient-free optimization algorithms, a field that has recently
started to spark new interest \citep{nsRG}. 
The algorithm we use is a
slightly more specific version of Algorithm \ref{algo:SNDS} above,
presented in Appendix (Section \ref{app-toy-exp-setting}). We sketch
here the supervised experiments and refer to Appendix, Section
\ref{app-toy-exp-setting} for a complete presentation of all experiments. Our supervised problem is the approximation of the best linear classifier on
\textit{non-separable} data. 
Table \ref{t-exp-short} presents the results obtained on a dataset
also displayed. While it has no
pretention whatsoever to bring significant experimental support for
the theory developed,
it displays interesting patterns of convergence: the
conditions for evolvability to be met are achieved quite early in the
process, and the
the presence of failures in the supervised case does not prevent
evolution to reach classifiers close to the optimal classifier.

\section{Discussion and conclusion}\label{sec-disc}

We split the discussion in several parts.\\

\noindent \textbf{Bregman divergences ---} We chose Bregman divergence
for several reasons. First, they generalize performance functions
previously used in seminal approaches \citep{kvwEW,vE} and they are in no way restrictive: the minimisation
of any differentiable convex function attaining its global minimum is trivially
equivalent to the minimisation of a Bregman divergence\footnote{The
  loss $\varphi(u)$ equals up to a constant the Bregman loss
  $D_\varphi(u\|\mathrm{opt}) = \varphi(u)- \varphi(\mathrm{opt})$,
  because $\nabla \varphi(\mathrm{opt}) = 0$.}. Moreover, they
\textit{exhaustively} define
fundamental classes of loss functions of both supervised and
unsupervised learning \citep{bmdgCWj,nnOT,rwCB} and their
properties are well understood: we use them to analyze our
mean-divergence model from Lemma \ref{lemmut} and devise a generalized
Pythagoras' Theorem trick for agnostic evolvability in Theorem
\ref{throbust}.\\

\noindent \textbf{Grid-convex minimization on a non-complete space
  ---} Evolvability can be analyzed from the standpoint of the actual
space in which evolution takes place from the mutator (not
\textit{every} representation may be built). For example, in the case of
\cite{vEOa}, the direct approach needs a complete vector space 
because the mutator samples in a topological closed ball. This is not
our case: any
permissible mutator (Definition \ref{defPERMUT}) evolves organisms on
a \textit{grid}\footnote{This also complies with the observation that
  biological evolution operates on a discrete code, DNA.} --- finer as $\epsilon$
decreases, but always a grid --- so in theory, we do not need a complete
vector space. For that reason, we do not need a performance convex everywhere: it just
needs to be convex on this grid. Our result is not the first to enjoy
such a property: a quick inspection shows that it is also the case for \cite[Section 6.1]{kvwEW}.
Our analysis is however significantly more general and would allow
the efficient minimization of highly non-convex functions whose
minima define a convex function supported on a grid, a
notorious example of which being Griewank function \citep{gGD}.\\

\noindent \textbf{Optimality of the mutator, frugality of evolution ---}
Evolution is fundamentally a resource-constrained process \citep{pklRA}, so it makes
a lot of sense to analyze evolvability in the light of time and
space requirements, in particular for the key algorithmic
device of evolution, the mutator. It is not hard
to check that our mutator is optimal essentially up to factor 2: suppose $\mutbasis$ defines a
basis for $\hg$. Any mutator running under the constraint
to be able to generate the whole space needs to work in time and space
$\Omega(\mathrm{Card}(\mutbasis))$, a lowerbound indeed matched by
ours. We also note the \textit{frugality} of our evolution process,
since our mutator never changes the set of mutations. All these
properties are enjoyed by \cite[Section 6.1]{kvwEW}; none of
them is enjoyed by the approaches of \cite{vEOa}. However, mention of the
possibility to
use of the \textit{canonical} basis to improve the mutator is given (footnote 2
in \cite{vEOa}), yet without proof but noting the difficulty of the
task. We note that we are significantly more general than this
possibility, since we consider not just orthonormal bases, but any
generating set.\\

\noindent \textbf{Random gradient-free optimization ---}
Derivative-free optimization is a big field intersecting mathematics, computer science and
engineering \citep{csvIT}, comprising a variety of methods like
genetic algorithms. Our approach to evolvability
certainly bears similarities with that class of algorithms, but it is
eventually closer to another popular set, Nelder-Mead simplex
methods \citep{nmSM,nsRG}. Briefly, such methods transform a
deformable non-degenerate simplex using a set of (initially
five) possible
moves, including reflection, expansion, etc. . When our generating set
is a basis, we can represent our algorithm with a rigid simplex modified by a set of
two operations only: "mirroring" the simplex to account for the polarity
of mutations, and "translating" it to follow mutations. While
our setting is not comparable to the deterministic setting of
\cite{nsRG}, we obtain dependences in $\epsilon$ similar to theirs on
non-strongly convex optimisation, without using the directional \textit{secant}
information that they use.
 \\

\noindent \textbf{Evolution on the efficient frontier ---} Our brief
analogy with Markowitz' portfolio theory \citep{mPS} following
Lemma \ref{lemmut} can be carried out much further to \textit{analytically} compute the
\textit{efficient frontier} of evolution. This appears to be an
important question because such a frontier has been
\textit{documented} in systems biology \citep{kVO}; interestingly, when
the current organism is ``far'' from the target (in a specific sense),
all superior beneficial mutations\footnote{Those mutations granting more than
just evolvability's minimal requirements, see Definition \ref{defexen} and following.} are close to the efficient frontier,
and therefore display an approximately optimal risk-return
tradeoff. This observation, formalized and proven in the Appendix
(Section \ref{app-add-res}), sheds interesting new light to the systems
biology's observations \citep{kVO} that go beyond the scope of our paper. 

To summarize, what we have shown is that, to evolve a vector space, one
merely needs a norm and a set that generates the space. The resulting
algorithm essentially has a single free parameter ($\alpha$,
Definition \ref{defPERMUT}) and is
straightforward to implement.
Our result generalizes in several directions another close result \citep{fDI} (Theorems 4.1, 4.4), \textit{i.e.}, outside the binary
classification framework, the realm of well-behaved
losses, single-dimensional outputs and non-agnostic evolvability. Our framework is also more
general. Feldman requires the target organism to have
minimal non-zero margin over \textit{all} conditions, which is
equivalent to replacing the non-zero probability by a unit probability in assumption (iii) of setting (\textbf{SF}), and therefore weakens
the general purpose of the result --- even when the minimal margin
assumption is reasonable in the restricted binary classification
setting. Finally, our analysis displays
better dependences on the key parameters $\alpha, T, m$: 
inspection of the bounds of Feldman shows that the guarantees on performance
increase may be very loose, namely as small as $\tilde{O}(\epsilon^a)$ for some potentially 
large constant $a$, which is significantly worse than our
$\tilde{\Omega}(\epsilon^2)$ guarantee in $\hgmon$. 

\bibliography{bibgen2+}

\clearpage

\section*{Appendix --- Table of contents}\label{sec:appendix}

\noindent \textbf{Overview of our results} \hrulefill Pg
\pageref{app-overview}\\
\noindent \textbf{Proof of Theorem \ref{evolmut-short}} \hrulefill Pg
\pageref{app-add-main}\\
\noindent \textbf{Proof of Theorem \ref{throbust}} \hrulefill Pg
\pageref{proofthrobust}\\
\noindent \textbf{Proof of Theorem \ref{thmtrap-short}} \hrulefill Pg
\pageref{proofthmtrap}\\
\noindent \textbf{Evolution on the efficient frontier} \hrulefill Pg
\pageref{app-add-res}\\
\noindent \textbf{Toy experiments (full)} \hrulefill Pg
\pageref{app-toy-exp-setting}

\newpage 

\section{Overview of our results}\label{app-overview}

Figure \ref{f-evolvability} presents a complete synthetic view of the main mechanisms
shown in our paper (including in the Appendix), with two properties never explicitly
documented before: the fact that the apparent weakness of the mutator
(which works regardless of the set that generates ${\hg}$) does not
prevent it to be able to ``compete'' with the best mutators when far from
the target, and the fact that evolution may just be trapped ``close''
to the target when it has succeeded. The agnostic evolution setting
relies on an analogue to Bregman orthogonal projection theorems \citep{anMO}
involving the performance function. 

\begin{figure}[t]
\begin{center}
\includegraphics[trim=20bp 50bp 0bp 20bp,clip,width=0.9\columnwidth]{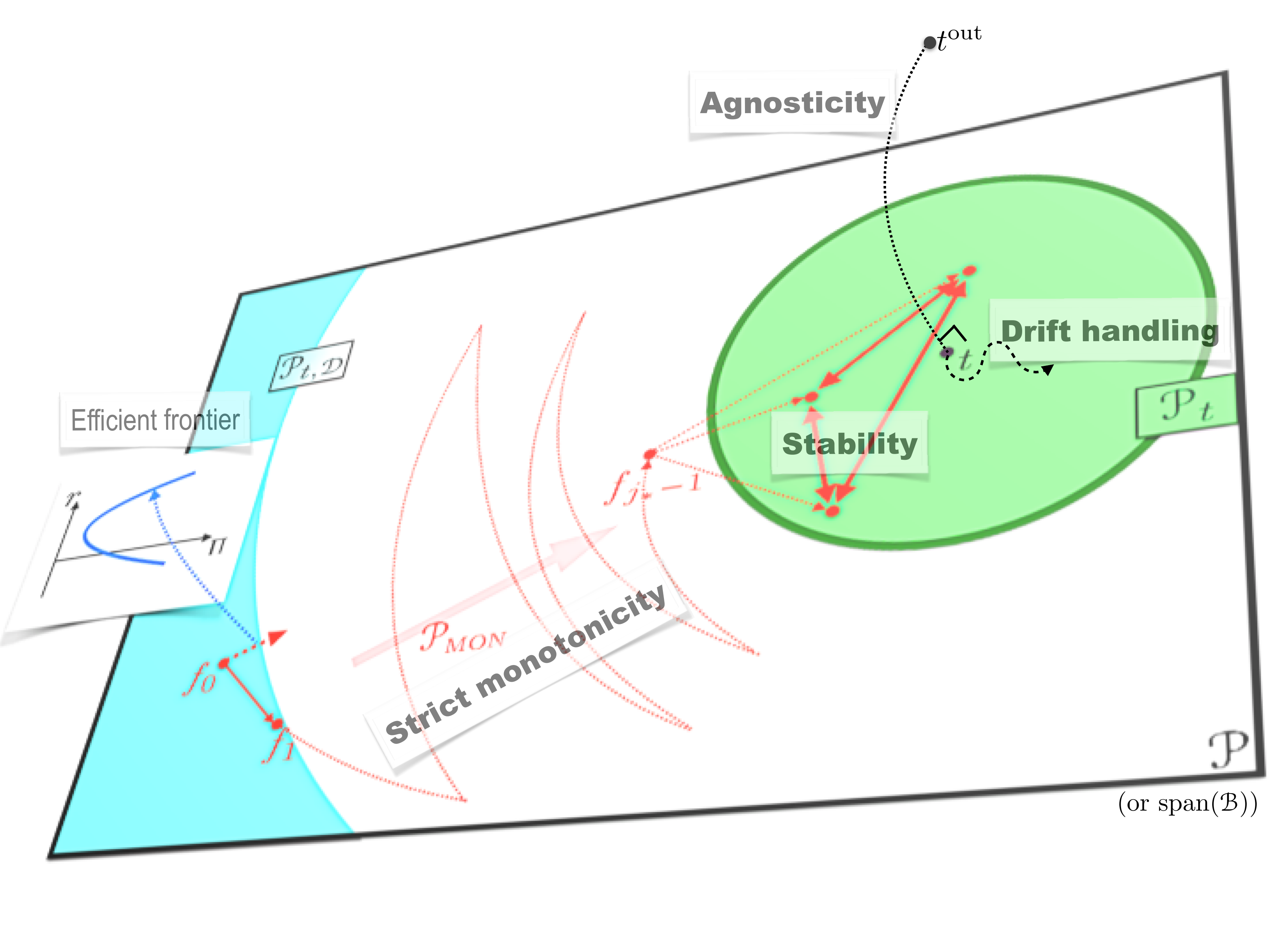}
\end{center}
\caption{Summary of our results --- Evolvability as carried out by our simple mutator (in red), with high
  probability: sequence $\hgmon$ monotonically converges to an organism
  that satisfies the conditions of evolvability (green area), after which
  monotonicity may not hold anymore, \textit{but} many subsequent organisms
  may satisfy the conditions of evolvability. When far from the target
  (blue), our mutator always has access to mutations close to the
  efficient frontier of evolution (diagram). If $t$ is not evolvable
  ($t^{\mathrm{out}}$), then evolution converges to the ``closest''
  encodable target. Stability also guarantees that when the evolved
  organism satisfies the conditions of evolvability, it will do so for
  at least a certain number of following generations. All notions
  developed in the main body of the paper are indicated in \textbf{bold faces}.
  The notion of efficient frontier is developed in the Appendix only (best viewed in color).}
\label{f-evolvability}
\end{figure}

\section{Proof of Theorem \ref{evolmut-short}}\label{app-add-main}

\subsection{Basic notations and helper Lemmata}\label{app-sub-not}

\begin{center}
\begin{tabular}{c|l}\hline
$t$ & target organism\\
$f_j$ & evolved organism\\
${\mathcal{B}}'$ & basis $\subseteq \mutbasis$\\
$\matrice{b}'$ & $\dG \times \dG$ transition matrix for basis
${\mathcal{B}}'$ in orthonormal ``gene'' basis $\{g_1, g_2, .., g_\dG\}$\\
$\delta(x)$ & $\defeq (t-f)(x) = \matrice{g}_x\matrice{b}^*(t-f)$, \\
 & difference of expressions between target $t$ and organism $f$\\
 & measured with respect to basis ${\mathcal{B}}^*$\\
$\delta^i$ & coordinate $i$ of vector $\delta$\\
$\|.\|_\hg$ & norm in canonical (``gene'') basis of $\hg$\\
$g_i$ & canonical basis vector of $\hg$\\
$g_i(x)$ & gene expression output (in ${\mathbb{R}}^d$) of $g_j$ on some $x\in {\mathcal{X}}$\\
$\matrice{g}_x $ & $\defeq [g_1(x)|g_2(x)|\cdots
|g_{\dG}(x)] \in {\mathbb{R}}^{d\times \dG}$\\
 & per-gene output matrix
on some $x\in {\mathcal{X}}$\\
$\gendistnone$ & $\defeq \expect_{x \sim {\mathcal{D}}} [\matrice{g}^\top_x 
\matrice{g}_x] \in {\mathbb{R}}^{\dG\times \dG}$\\
$\mathrm{ker}(.)$ & null space\\ \hline
\end{tabular}
\end{center}
Unless otherwise stated, all organisms are expressed in basis $\mutbasis^* \subseteq \mutbasis$, that is, 
\begin{eqnarray}
f(x) & = & \matrice{g}_x \matrice{b}^* f\:\:,
\end{eqnarray} 
and the norm of the encoding of $f$ expressed in basis $\mutbasis^*$ shall be
\begin{eqnarray}
\|f\|_{\hg} & \defeq & \sqrt{\langle \matrice{b}^*f,
  \matrice{b}^*f\rangle}\:\:.
\end{eqnarray}
We use in several places the following Lemmata.

\begin{lemma}\label{lemmareturndec}
Under setting (\textbf{SF}), there exists symmetric positive definite matrix
$\matrice{m}$ such that 
$\expect_{{\mathcal{D}}}[\mathrm{R}_{f, i} (x)] = \langle \matrice{b}^* b_i,
 \matrice{m}
 \matrice{b}^* \delta\rangle, \forall b_i \in {\mathcal{B}}$
 (coordinates of $\delta \defeq t-f, b_i$ expressed in basis ${\mathcal{B}}^*$).
\end{lemma}
\begin{proof}
Because $\varphi$ is twice differentiable and strictly convex, a Taylor expansion of
$\nabla \varphi$ around $t(x)$ yields
\begin{eqnarray}
(\nabla\varphi \circ t)(x) & = & (\nabla\varphi \circ f)(x) +
\textsc{h}(x) (t-f)(x)\nonumber\\
 & = & (\nabla\varphi \circ f)(x) + \textsc{h}(x)
 \matrice{g}_x\matrice{b}^* \delta\:\:, \label{taylor1}
\end{eqnarray}
for some value of the Hessian $\textsc{h}(x)$ of $\varphi$, therefore symmetric positive definite. We get
\begin{eqnarray}
\expect_{{\mathcal{D}}}[\langle b_i(x),
   (\nabla\varphi \circ t)(x) - (\nabla\varphi \circ f)(x)\rangle] & =
   & \expect_{{\mathcal{D}}}[\langle \matrice{g}_x\matrice{b}^*b_i, \textsc{h}(x)
 \matrice{g}_x\matrice{b}^* \delta\rangle] \nonumber\\
 & = & \expect_{{\mathcal{D}}}[\langle \matrice{b}^*b_i, \matrice{g}^\top_x\textsc{h}(x)
 \matrice{g}_x\matrice{b}^* \delta\rangle] \nonumber\\
 & = & \langle \matrice{b}^*b_i, \expect_{{\mathcal{D}}}[\matrice{g}^\top_x\textsc{h}(x)
 \matrice{g}_x]\matrice{b}^* \delta\rangle \nonumber\\
 & \defeq & \langle \matrice{b}^*b_i,
 \matrice{m}
 \matrice{b}^* \delta\rangle\:\:,\nonumber
\end{eqnarray}
with 
\begin{eqnarray}
\matrice{m} & \defeq &
\expect_{{\mathcal{D}}}[\matrice{g}^\top_x\textsc{h}(x)
 \matrice{g}_x]\:\:.\label{defMM}
\end{eqnarray} 
Because of (iii) in
(\textbf{SF}), for any $g\neq 0_\hg \in \hg$, for each $x\in
\mathcal{X}$ for which it is expressed, we have $g(x) = \matrice{g}_x
\matrice{b}^* g \neq 0_{{\mathbb{R}}^d}$ and so
\begin{eqnarray}
\langle \matrice{b}^* g,
\matrice{m} \matrice{b}^* g\rangle & = & \int_{\mathcal{X}} p(x) \langle g(x),
\textsc{h}(x)
 g(x)\rangle \mathrm{d}x \nonumber\\
& \geq& 
\pr_{x\sim \mathcal{D}}[g(x) \neq 0_{{\mathbb{R}}^d}] \cdot \min_{g(x)
\neq 0_{{\mathbb{R}}^d}} \langle g(x),
\textsc{h}(x)
 g(x)\rangle\nonumber\\
& > & 0\:\:,\nonumber
\end{eqnarray} 
showing $\matrice{m}$ is positive definite.
\end{proof}

\begin{lemma}\label{lemapproxbregdiv}
Under setting (\textbf{SF}), for any $f, t\in \hg$ (coordinates expressed in basis $\mutbasis^*$),
the following holds, for some $\matrice{m}'$ symmetric positive definite:
\begin{eqnarray}
\expect_{x \sim {\mathcal{D}}} [D_\varphi(f(x) \| t(x))] & = & \langle \matrice{b}^* (f-t), \matrice{m}'
 \matrice{b}^* (f-t)\rangle\:\:,\label{eqbregm1}\\
 \expect_{x \sim {\mathcal{D}}} [D_\varphi(f(x) \|
 t(x))] & \in & \left[\frac{\hessianmin}{2} \cdot \langle \matrice{b}^* (f-t), \gendistnone
 \matrice{b}^* (f-t)\rangle, \frac{\hessianmax}{2} \cdot \langle \matrice{b}^* (f-t), \gendistnone
 \matrice{b}^* (f-t)\rangle\right] \:\:.\label{eqbregm2}
\end{eqnarray}
\end{lemma}
\begin{proof}
Both results are a consequence of \citep{anMO}, that for
any twice differentiable $\varphi$, 
\begin{eqnarray}
D_\varphi(f(x) \| t(x)) & = & \frac{1}{2} \cdot \langle \matrice{g}_x \matrice{b}^* (f-t), \matrice{h}(x)
 \matrice{g}_x\matrice{b}^* (f-t)\rangle \label{eqbreghess} \:\:,
\end{eqnarray}
for some value of the Hessian $\matrice{h}(x)$ of $\varphi$, therefore symmetric
positive definite. Then, we define $\matrice{m}'$ as
\begin{eqnarray}
\matrice{m}' & \defeq &
\frac{1}{2} \cdot \matrice{m}\:\:,\label{defMM2}
\end{eqnarray} 
where $\matrice{m}$ is defined in eq. (\ref{defMM}), and so
$\matrice{m}' \succ 0$. We then use the definition of $\gamma$ to
obtain ineq. (\ref{eqbregm2}).
\end{proof}

\subsection{Complete statement of Theorem \ref{evolmut-short}}\label{proofevolmut-short}

We first provide a more complete statement of Theorem
\ref{evolmut-short}. 
We now define two key parameters:
\begin{eqnarray}
U & \defeq & \frac{\hessianmax^{\frac{3}{2}} \expressmax^{\frac{1}{2}}}{\hessianmin^{\frac{3}{2}} \expressmin^{\frac{1}{2}}} \cdot \frac{2\sqrt{e \dG}}{\overline{B}_{\hg}({\mutbasis}^*)}\cdot
\darwin\:\:,\label{defU}\\
V & \defeq & {\hessianmax} \cdot  \max_{i\in [\dF]_*} \{
\expect_{{\mathcal{D}}}
[\|b_i(x)\|^2_2]\} \:\:,\label{defV}
\end{eqnarray}
and triples of evolution ``knobs'', $(z_\tau, z_\alpha, z_\tol) \in {\mathbb{R}}_{+*}^3$, all absolute constants. 
\begin{definition}\label{defNER}
Set $\mathfrak{R} \subset {\mathbb{R}}^3$ is defined as the subset of
triples $(z_1,z_2,z_3)$ such that (i) $z_1,z_2,z_3 > 0$, (ii) $z_3 - z_1z_2 > 0$,
(iii) $z_2^2 - z_2(1-z_1) + z_3 \leq 0$.
\end{definition}
Remark that $\mathfrak{R} \neq \emptyset$, since for example $(1/9,
1/3, 2/27) \in \mathfrak{R}$. We now state our main result.
\begin{theorem}\label{evolmut-long}
(evolvability of vector spaces, complete statement) Assume (\textbf{SF}) holds, and fix any $(z_\tau, z_\alpha, z_\tol)\in
\mathfrak{R}$. 
${\hg}$ is distribution-free evolvable by
any permissible mutator $\mutator$, with tolerance:
\begin{eqnarray}
\tol & \defeq & \frac{z_{\tol}}{U^2\max\{1, V\}}\cdot \epsilon^{2}  = \theta(\epsilon^{2})\:\:,\label{defTol}
\end{eqnarray}
and magnitude of mutations:
\begin{eqnarray}
\alpha & \defeq & \frac{z_\alpha}{U\max\{1, V\}}\cdot
\epsilon = \theta(\epsilon)\:\:.\label{defAlpha}
\end{eqnarray}
The number of conditions sampled at each iteration satisfies:
\begin{eqnarray}
m & = & O\left(\frac{\hessianmax^6
        \expressmax^2 \sup_{x} \single^2(\mutbasis^*, x)}{\hessianmin^4 \expressmin^2} \cdot \frac{1}{\overline{B}^4_{\hg}(\mutbasis^*)} \cdot\frac{ \darwin^4 }{\epsilon^{2}}\log\left(\frac{\dF
      \darwin}{\epsilon}\right)\right)  = \tilde{O}\left(\frac{ \darwin^4 }{\epsilon^{2}}\log\left(\frac{\dF
      \darwin}{\epsilon}\right)\right)\label{defmm}\:\:.
\end{eqnarray}
Finally, the number of evolution steps $T$ sufficient to comply with
ineq. (\ref{condgen}) is $T = \tilde{O}\left(\darwin^4/\epsilon^{2}\right)$.
\end{theorem}
Remark that $z_\tau$ is defined but not used in the Theorem
statement. It shall be used in its proof below. 

\subsection{Proof of Lemma \ref{lemmut}}\label{prooflemmut}
By definition,
\begin{eqnarray}
\bene(f)  & = & \{g \in \neigh_\epsilon(f): \perf_{t,\varphi}(g,
{\mathcal{S}}) \geq \perf_{t,\varphi}(f,
{\mathcal{S}}) + \tol\} \nonumber\\
 & = & \{g \in \neigh_\epsilon(f): \expect_{{\mathcal{S}}}
[ \langle g(x) - f(x),
(\nabla\varphi \circ t)(x) - (\nabla\varphi \circ f)(x)\rangle - D_\varphi(g(x) \| f(x)) ] \geq \tol\} \nonumber\\
 & = & \left\{
\begin{array}{l}
f+\sigma \alpha b_i  \in \neigh_\epsilon(f): 
\expect_{{\mathcal{S}}}
\left[ \langle \sigma b_i(x),
(\nabla\varphi \circ t)(x) - (\nabla\varphi \circ f)(x)\rangle\right] \\
- \expect_{{\mathcal{S}}}
\left[ \frac{1}{\alpha} D_\varphi( (f
+ \alpha \cdot \sigma b_i)(x) \| f(x)) \right]
\geq \frac{\tol}{\alpha}\end{array}\right\}\nonumber\\
 & = & \left\{f+\sigma \alpha b_i \in \neigh_\epsilon(f): \expect_{{\mathcal{S}}}
[\mathrm{R}_{f, i}(x)] -\expect_{{\mathcal{S}}}
[\Pi_{f,i}(x)] \geq \frac{\tol}{\alpha}\right\} \:\:,\label{propbene}
\end{eqnarray}
as claimed (end of the proof of Lemma \ref{lemmut}).

\subsection{Proof of Theorem \ref{th00}}\label{proofth00}

The proof of the Theorem
consists of the following building blocks:
\begin{itemize}
\item  [\textbf{BB.1}] we show a result more general than
  eq. (\ref{defalphatauzerostart}), namely, 
over \textit{all} steps $j \in[T]_*$ and with high probability: 
\begin{eqnarray}
\left(\exists b_i \in \mutbasis^* : \rho(f_j, t|{{{\mathcal{D}}}}) \geq
\frac{\sqrt{e \dG}}{\hessianmin B_{\hg}(\mutbasis^*)}\left(\tau
  + \alpha
{\hessianmax} \cdot \|b_i\|_{\hg} \cdot \rho(b_i|{{{\mathcal{D}}}}) + \frac{\tol}{\alpha}
\right) \right)\nonumber\\
& &  \hspace{-5cm} \Rightarrow \bene(f_j) \neq \emptyset \:\:.\label{defalphatauzero}
\end{eqnarray}
This is more general since we show that
eq. (\ref{defalphatauzerostart}) holds for all organisms of the
evolution sequence that belong to ${\hg}_{t, {\mathcal{D}}}$, and not
just the ``first'' ones in $\hgmon$. To have this with high
probability it is sufficient to sample $m = \tilde{\Omega}(\max_j
\|t-f_j\|_{\hg}^2)$ conditions, which may be hard to upperbound
depending on $f_j$;
\item  [\textbf{BB.2}] we show that, in the subsequence
  $\hgmon$, $\max_j
\|t-f_j\|_{\hg}^2$ may be conveniently upperbounded.
\end{itemize}

\noindent $\hookrightarrow$ (Proof of [\textbf{BB.1}]) We temporarily
drop subscript $j$ in $f_j$ for clarity. The proof involves the
following three steps. First, we show that for any current representation $f$, there always exist a mutation whose
expected return is at least a (positive) fraction of the
\exen-divergence between $f$ and the target $t$. Its proof involves a
simple lowerbound on the volume induced by an
\textit{arbitrary} basis of vectors, which may be of independent interest. Second,
we show that, in the evolvability setting, this mutation is special:
whenever the current representation $f$ is in $\hgmon$, there is always $\sigma \in\{-1, 1\}, b_i \in
{\mutbasis}$ such that 
\begin{eqnarray}
\expect_{{\mathcal{D}}}
[\mathrm{R}_{f, i}(x)] - \expect_{{\mathcal{D}}}
[\Pi_{f, i}(x)]  & \geq & \frac{\tol}{\alpha} +
\tau\:\:.\label{supben} 
\end{eqnarray}
We shall see that this guarantees
equivalently $\expect_{{\mathcal{D}}}
[\mathrm{R}_{f, i}(x)] - \expect_{{\mathcal{D}}}
[\Pi_{f, i}(x)] = \Omega(\epsilon)$ --- we call this mutation
\textit{superior beneficial}, since
the right hand side exceeds the beneficial requirements (eq. (\ref{propbene})) by $\tau$, and
furthermore the left hand side is measured on ${\mathcal{D}}$. Third
and last, even when the mutation picked is not superior beneficial,
sampling a number of examples large enough is sufficient to guarantee
$\bene(f)\neq \emptyset$. More precisely,  when $m$ is large enough,
the sum of the two differences between $(\expect_{{\mathcal{D}}}
[\mathrm{R}_{f, i}(x)] - \expect_{{\mathcal{D}}}
[\Pi_{f, i}(x)])$ and $(\expect_{{\mathcal{S}}}
[\mathrm{R}_{f, i}(x)] - \expect_{{\mathcal{S}}}
[\Pi_{f, i}(x)])$ in absolute value is at most $\tau$ over each of the
$T$ iterations, with probability $\geq 1-\epsilon$. Using (\ref{supben}) then proves the
statement of the Theorem because of the definition of $\bene(f)$ in (\ref{propbene}).

\begin{lemma}\label{lemmax}
Assume (\textbf{SF}) holds. Then for any distribution ${\mathcal{D}}$ and any
representations $t, f \in
{\hg}$ with coordinates expressed in basis $\mutbasis^*$, 
\begin{eqnarray}
\max_{b_i \in \mutbasis^*} {|\expect_{{\mathcal{D}}}[\mathrm{R}_{f, i} (x)] |} & \geq
&K \cdot\rho(f, t|{{{\mathcal{D}}}})\:\:, \label{eqBB}
\end{eqnarray}
where 
\begin{eqnarray}
K & \defeq & \frac{\hessianmin}{\sqrt{e\dG}} \cdot B_{{\hg}}(\mutbasis^*)\label{defKKCOD}\:\:,
\end{eqnarray}
and $B_{{\hg}}(\mutbasis^*)$ is the corrected average norm in
Definition \ref{defBSTAR}.
\end{lemma}
\begin{proof}
The majority of the proof consists in showing first that ineq. (\ref{eqBB}) holds for
\begin{eqnarray}
K & \defeq & \frac{\hessianmin}{\sqrt{e}} \cdot \left(\frac{G(\mutbasis^*)^{\dG}}{A(\mutbasis^*)^{\dG-1}}
  \cdot \frac{(1-\cos(\thetaBprime))^{\dG - 1} \cdot (1 + (\dG - 1) \cdot \cos(\thetaBprime))}{\dG}\right)^{\frac{1}{2}}\label{defKK}\:\:,
\end{eqnarray}
where parameters $G(\mutbasis^*), A(\mutbasis^*), \thetaBprime$ are
defined in Definition \ref{defBSTAR}. Then, we show that
ineq. (\ref{eqBB}) holds for the expression of $K$ in
eq. (\ref{defKKCOD}). Define for short:
\begin{eqnarray}
\delta & \defeq & t - f = \sum_{i\in \dG} \delta^i b_i\:\:,\label{defallv} 
\end{eqnarray}
assuming without loss of generality that $\mutbasis^* \defeq \{b_1,
b_2, ..., b_{\dG}\}$. 
From Lemma \ref{lemmareturndec}, there exists symmetric positive definite matrix
$\matrice{m}$ such that 
$\expect_{{\mathcal{D}}}[\mathrm{R}_{f, i} (x)] = \langle b_i,
 \matrice{m}
 \matrice{b}^* \delta\rangle, \forall b_i \in {\mathcal{B}}^*$. This
 allows to get the last equality of:
\begin{eqnarray}
\langle \matrice{b}^* \delta,
 \matrice{m}
 \matrice{b}^* \delta\rangle & = & \sum_{i\in [\dG]} 
\delta^i \cdot \langle b_i,
 \matrice{m}
 \matrice{b}^* \delta\rangle\nonumber\\
 & \leq & \left(\sum_{i\in [\dG]} 
  {(\delta^i)^2}\right)^{\frac{1}{2}} \left(\sum_{i\in [\dG]} 
  { \langle b_i,
 \matrice{m}
 \matrice{b}^* \delta\rangle}\right)^{\frac{1}{2}}
\label{pp1}\\
 & \leq & \sqrt{\dG} \max_{i\in [\dG]}  {|\langle b_i,
 \matrice{m}
 \matrice{b}^* \delta\rangle|} \cdot \left(\sum_{i\in [\dG]} 
  {(\delta^i)^2}\right)^{\frac{1}{2}} \nonumber\\
 & & = \sqrt{\dG} \max_{i\in [\dG]}  {|\expect_{{\mathcal{D}}}[\mathrm{R}_{f, i} (x)]|} \cdot \left(\sum_{i\in [\dG]} 
  {(\delta^i)^2}\right)^{\frac{1}{2}}\:\:. \label{pp3}
\end{eqnarray}
Ineq. (\ref{pp1}) is
Cauchy-Schwartz inequality. Furthermore, the derivations of Lemma \ref{lemmareturndec}
and the definition of $\hessianmin$ yields $\langle \matrice{b}^* \delta,
 \matrice{m}
 \matrice{b}^* \delta\rangle = \expect_{{\mathcal{D}}}[\langle \matrice{b}^*\delta, \matrice{g}^\top_x\textsc{h}(x)
 \matrice{g}_x\matrice{b}^* \delta\rangle] \geq \hessianmin \cdot \expect_{{\mathcal{D}}}[\langle \matrice{g}_x\matrice{b}^*\delta, 
 \matrice{g}_x\matrice{b}^* \delta\rangle] \defeq \hessianmin
 \cdot \expect_{{\mathcal{D}}}[\|\delta(x)\|_2^2]$. Combining this with (\ref{pp3}) yields:
\begin{eqnarray}
\max_{i\in [\dG]} {|\expect_{{\mathcal{D}}}[\mathrm{R}_{f, i} (x)]|} &
 \geq & \frac{\hessianmin \expect_{{\mathcal{D}}}[\|\delta(x)\|_2^2]}{\sqrt{\dG}\left(\sum_{i\in [\dG]}
  {(\delta^i)^2}\right)^{\frac{1}{2}}}\:\:. \label{ff2}
\end{eqnarray}
Let us work on the $\left(\sum_j
  {(\delta^j)^2}\right)^{1/2}$ term, and relate it
to the norm $\|\delta\|_{\hg} \defeq \sqrt{\langle \matrice{b}^*\delta,
  \matrice{b}^*\delta\rangle}$, with 
$\matrice{b}^*$ be the transition matrix that collects
vectors from $\mutbasis^*$ in column, expressed in the orthonormal gene
basis of ${\hg}$. We now need the following Lemma.
\begin{lemma}\label{lposdef}
Under (\textbf{SF}), ${\matrice{b}^*}^\top \matrice{b}^*\succ \matrice{0}$.
\end{lemma}
\begin{proof}
Let
$\phi_{\matrice{b}^*} : {\hg} \rightarrow {\hg}$ the linear form that
$\matrice{b}^*$ represents. Condition (i) in (\textbf{SF}) implies that
$\phi_{\matrice{b}^*}$ is injective, and therefore
$\mathrm{ker}(\phi_{\matrice{b}^*}) = \{0_{\hg}\}$, and
thus $\langle \matrice{b}^* z ,\matrice{b}^* z \rangle = \langle z ,{\matrice{b}^*}^\top \matrice{b}^* z \rangle> 0$ whenever $z \neq
0_{\hg}$ (End of the proof of Lemma \ref{lposdef}).\end{proof}
We denote
$\{(\lambda_j, u_j)\}_{j=1}^{\dG}$ the pairs (strictly
positive eigenvalues in non-decreasing order, orthonormal
eigenvectors), with
$u_j \in {\mathbb{R}}^\dG$, so that the following decomposition holds:
${\matrice{b}^*}^\top \matrice{b}^* = \sum_i \lambda_i u_i
u_i^\top$. It comes:
\begin{eqnarray}
\|\delta\|_{\hg}^2 & \defeq & \langle \matrice{b}^*\delta,
  \matrice{b}^*\delta\rangle\nonumber\\
 & = & \sum_{i\in [\dG]}{\lambda_i
  \langle \delta, u_i\rangle^2} \nonumber\\
 & = & \left( \sum_{i\in [\dG]}
  {(\delta^i)^2}\right)\cdot \sum_{i\in [\dG]}{\lambda_i
  \cos^2(\delta, u_i)}\label{normal11}\\
 & \geq & \lambda_1 \cdot \left( \sum_{i\in [\dG]}
  {(\delta^i)^2}\right) \cdot \sum_{i\in [\dG]}{
  \cos^2(v, u_i)}\nonumber\\
&  & = \lambda_1 \cdot \left( \sum_{i\in [\dG]}
  {(\delta^i)^2}\right)\:\:,\label{eqq1}
\end{eqnarray}
where eq. (\ref{normal11}) comes from the fact that $u_j$s are 
normal and (\ref{eqq1}) comes from the fact that they are orthogonal. Let us define 
\begin{eqnarray}
{\matrice{b}^*}' & \defeq & \left(\frac{\dG}{\sum_{i\in [\dG]}
  {\|b_i\|_{\hg}^2}}\right)^{\frac{1}{2}} \cdot \matrice{b}^*\:\:. \label{defBprime}
\end{eqnarray}
${\matrice{b}^*}'$ satisfies:
\begin{eqnarray}
\trace{({\matrice{b}^*}'^\top {\matrice{b}^*}')} & = & \dG\:\:.\label{pbtilde}
\end{eqnarray}
Let $0< \tilde{\lambda}_1 \leq \tilde{\lambda}_2 \leq ... \leq
\tilde{\lambda}_{\dG}$ the eigenvalues of ${\matrice{b}^*}'^\top
{\matrice{b}^*}'$, all strictly positive because of Lemma \ref{lposdef}. The Arithmetic-Geometric-Harmonic means inequality brings:
\begin{eqnarray}
\prod_{i=2}^{\dG} \tilde{\lambda}_i & \leq & \left(\frac{1}{\dG-1}
  \sum_{i=2}^{\dG}\tilde{\lambda}_i  \right)^{\dG-1}\nonumber\\
 & & = \left(\frac{\trace{({\matrice{b}^*}'^\top {\matrice{b}^*}')} - \tilde{\lambda}_1}{\dG-1} \right)^{\dG-1}\nonumber\\
 & = & \left(\frac{\dG - \tilde{\lambda}_1}{\dG-1} \right)^{\dG-1}\nonumber\\
 & \leq & \left(\frac{\dG}{\dG-1} \right)^{\dG-1}\:\:,
\end{eqnarray}
Multiplying both sides by $\tilde{\lambda}_1$ and reorganising, we
get:
\begin{eqnarray}
\tilde{\lambda}_1 & \geq & \left(\frac{\dG-1}{\dG} \right)^{\dG-1}
\determinant{{\matrice{b}^*}'^\top {\matrice{b}^*}'} \nonumber \\
 & \geq & \frac{\determinant{{\matrice{b}^*}'^\top {\matrice{b}^*}'}}{e}\:\:,\label{btildel}
\end{eqnarray}
since function $u(n) \defeq ((n-1)/n)^{n-1}$ is strictly decreasing on
$n \in {\mathbb{N}}_*$
and has limit $\lim_{+\infty} u(n) = 1/e$.
Finally, using eq. (\ref{btildel}), we obtain from eq. (\ref{defBprime}):
\begin{eqnarray}
\lambda_1 & = & \frac{\sum_{i\in [\dG]}
  {\|b_i\|_{\hg}^2}}{\dG} \cdot \tilde{\lambda}_1 \nonumber\\
 & \geq & \frac{\sum_{i\in [\dG]}
  {\|b_i\|_{\hg}^2}}{\dG} \cdot \frac{\determinant{{\matrice{b}^*}'^\top {\matrice{b}^*}'}}{e}\nonumber\\
 &  & = \frac{1}{e} \cdot \determinant{\left(\frac{\sum_{i\in [\dG]}
  {\|b_i\|_{\hg}^2}}{\dG} \right)^\frac{1}{\dG}{\matrice{b}^*}'^\top {\matrice{b}^*}'}\nonumber\\
 & = & \frac{1}{e} \cdot \determinant{\left(\frac{\dG}{\sum_{i\in [\dG]}
  {\|b_i\|_{\hg}^2}} \right)^{\frac{\dG - 1}{2\dG}}{\matrice{b}^*}^\top \left(\frac{\dG}{\sum_{i\in [\dG]}
  {\|b_i\|_{\hg}^2}} \right)^{\frac{\dG - 1}{2\dG}}\matrice{b}^*}\label{bbtilde}\:\:.
\end{eqnarray}
Let us define
\begin{eqnarray}
\tilde{\matrice{b}} & \defeq & \left(\frac{\dG}{\sum_{i\in [\dG]}
  {\|b_i\|_{\hg}^2}} \right)^{\frac{\dG - 1}{2\dG}}\cdot \matrice{b}^*\:\:,
\end{eqnarray}
so that ineq. (\ref{bbtilde}) reads
\begin{eqnarray}
\lambda_1 & \geq & \frac{1}{e} \cdot \determinant{\tilde{\matrice{b}}^\top \tilde{\matrice{b}}} = \frac{1}{e} \cdot \volume{\tilde{\matrice{b}}}^2\label{blambda}\:\:,
\end{eqnarray}
where we let $\volume{\tilde{\matrice{b}}}^2 \defeq \determinant{\tilde{\matrice{b}}^\top
    \tilde{\matrice{b}}}$ denote the squared volume induced by set
  $\tilde{{\mutbasis}}$, since
the columns of $\tilde{\matrice{b}}$ also define a basis of $\hg$,
$\tilde{{\mutbasis}} \defeq \{\tilde{b}_1, \tilde{b}_2, ...,
\tilde{b}_{\dG}\}$, with
\begin{eqnarray}
\tilde{b}_{i} & \defeq & \left( \frac{\dG}{\sum_{i'\in [\dG]}
  {\|b_{i'}\|_{\hg}^2}} \right)^{\frac{\dG-1}{2\dG}} \cdot b_i
\:\:,\forall i \in [\dG]\:\:.\label{defBtildei}
\end{eqnarray}
Putting eq. (\ref{blambda}) and  (\ref{eqq1}) altogether, we obtain
\begin{eqnarray}
\left(\sum_{i\in [\dG]}
  {(\delta^i)^2}\right)^{\frac{1}{2}} & \leq &
\frac{\|v\|_{\hg}}{\sqrt{\lambda_1}}\nonumber\\
 & \leq & \sqrt{e} \cdot
 \frac{\|v\|_{\hg}}{\volume{\tilde{\matrice{b}}}}\:\:. \label{propbetaprime}
\end{eqnarray}
Combining eq. (\ref{ff2}) and this last inequality yields:
\begin{eqnarray}
\max_{i\in [\dG]}{|\expect_{{\mathcal{D}}}[\mathrm{R}_{f, i} (x)]|} &
 \geq & \frac{\hessianmin \volume{\tilde{\matrice{b}}}}{\sqrt{e\dG}} \cdot \frac{
   \expect_{{\mathcal{D}}}[\|\delta(x)\|_2^2]}{
   \|\delta\|_{\hg}} \nonumber\\
 &  & = \frac{\hessianmin \volume{\tilde{\matrice{b}}}}{\sqrt{e\dG}} \cdot
 \rho(f - t | {\mathcal{D}}) \nonumber\\
 &  & = \frac{\hessianmin \volume{\tilde{\matrice{b}}}}{\sqrt{e\dG}} \cdot
 \rho(f,t | {\mathcal{D}})\label{ffb2}\:\:.
\end{eqnarray}
Let us now work on $\volume{\tilde{\matrice{b}}}$. Denoting $S_{\dG}$ the
symmetric group of degree $\dG$, we have:
\begin{eqnarray}
\volume{\tilde{\matrice{b}}} & = & \left(\mathrm{det}(\tilde{\matrice{b}}^\top \tilde{\matrice{b}})\right)^{\frac{1}{2}}\nonumber\\
 & = & \left(\sum_{\sigma \in S_{\dG}} {\mathrm{sign}(\sigma) \prod_{i\in [\dG]}{\langle \tilde{b}_i,
 \tilde{b}_{\sigma(i)} \rangle}}\right)^{\frac{1}{2}}\nonumber\\
 & = & \left(\sum_{\sigma \in S_{\dG}} {\mathrm{sign}(\sigma) \prod_{i\in [\dG]}{\| \tilde{b}_i\|_\hg
 \|\tilde{b}_{\sigma(i)}\|_\hg \cos(\tilde{b}_i, \tilde{b}_{\sigma(i)})}}\right)^{\frac{1}{2}}\nonumber\\
 & = & \prod_{i\in [\dG]}{\|
   \tilde{b}_i\|_\hg} \left(\sum_{\sigma \in S_{\dG}} {\mathrm{sign}(\sigma)
   \prod_{i\in [\dG]}{\cos(\tilde{b}_i, \tilde{b}_{\sigma(i)})}}\right)^{\frac{1}{2}}\:\:. \label{llb}
\end{eqnarray}
Eq. (\ref{llb}) does not depend on the orientation of
the vectors in $\tilde{{\mutbasis}}$: changing $\tilde{b}_i$ to $-\tilde{b}_i$ keeps the
same expression (in each product, exactly two
cosines change of sign), and there is one such orientation of all
vectors such that all angles are in $[0,\pi/2]$. The quantity
(\ref{llb}) being then decreasing if all cosines increase, it is
minimized by the one in which all angles equal $\thetaB$, in which case
expression (\ref{llb}) admits the simplified lowerbound: 
\begin{eqnarray}
\prod_{i\in [\dG]}{\|
   \tilde{b}_i\|_\hg} \left(\sum_{\sigma \in S_{\dG}} {\mathrm{sign}(\sigma)
   \prod_{k=1}^{\dG}{\cos(\tilde{b}_k, \tilde{b}_{\sigma(k)})}}\right)^{\frac{1}{2}} & \geq & \prod_{i\in [\dG]}{\|
   \tilde{b}_i\|_\hg} \left(\sum_{\sigma \in S_{\dG}} {\mathrm{sign}(\sigma)
     \cos^{\upsilon(\sigma)}(\thetaBstar)}\right)^{\frac{1}{2}} \nonumber\\
 & & \defeq  A_l \sqrt{A_a}\:\:,\label{binfAB}
\end{eqnarray}
where $\upsilon(\sigma) = |\{i : \sigma(i) \neq i\}|$ counts
the number of integers whose position has changed through the
permutation $\sigma \in S_{\dG}$, or similarly it is the size of the
``deranged'' sub-permutation of $\sigma$. We now proceed through
finding simplified expressions for $A_l$ and $A_a$, the parts that respectively
depends on \emph{l}engthes and \emph{a}ngles.\\

\noindent We first compute $A_l$. We have from the definition of
$\tilde{b}_i$ in (\ref{defBtildei}):
\begin{eqnarray}
A_l & = & \prod_{i\in [\dG]}{\|
   \tilde{b}_i\|_{\hg}}\nonumber\\
 & = & \prod_{i\in [\dG]}{\left\|
   \left( \frac{\dG}{\sum_{i'\in [\dG]}
  {\|b_{i'}\|_{\hg}^2}} \right)^{\frac{\dG-1}{2\dG}} \cdot b_i\right\|_{\hg}}\nonumber\\
 & = & \left( \frac{\dG}{\sum_{i\in [\dG]}
  {\|b_i\|_{\hg}^2}} \right)^{\frac{\dG-1}{2}}
\prod_{i\in [\dG]} {\|
   b_i\|_{\hg}}\nonumber\\
 & = & 
\left(\prod_{i\in [\dG]} {\|
   b_i\|^2_{\hg}}^{\frac{1}{\dG}}\right)^{\frac{\dG}{2}} \bigg/ \left( \frac{\sum_{i\in [\dG]}
  {\|b_i\|_{\hg}^2}}{\dG} \right)^{\frac{\dG-1}{2}}\nonumber\\
 & = & \frac{G_{\mutbasis}^\frac{\dG}{2}}{A_{\mutbasis}^\frac{\dG-1}{2}}\label{valAl}\:\:.
\end{eqnarray}

\noindent We now compute $A_a$. Let us define polynomial $P_{\dG}(z)$ by:
\begin{eqnarray}
P_{\dG}(z) & \defeq & \sum_{\sigma \in S_{\dG}} {\mathrm{sign}(\sigma)
     (1-z)^{\upsilon(\sigma)}}\:\:.\label{defPZ}
\end{eqnarray}
Notice that $A_a = P_{\dG}(1-\cos(\thetaBstar))$. While the max degree of
$P_{\dG}(z)$ is $\dG$, we now show that $P_{\dG}(z)$ involves only two monomials, of degree $\dG$ and $\dG-1$.
\begin{lemma}\label{eqpdg}
$P_{\dG}(z) = z^{\dG - 1} \cdot (\dG - (\dG - 1) \cdot z)$.
\end{lemma}
\begin{proof}
Denote $\kappa_k$ the coefficient of $(-z)^k$
in $P_{\dG}(z)$, for $k\in [\dG]$. It
satisfies:
\begin{eqnarray}
\kappa_k & = & \sum_{\sigma \in S_{\dG} : \upsilon(\sigma)\geq k}
{\mathrm{sign}(\sigma) {\upsilon(\sigma) \choose k}}\label{pKappak0}\\\
 & = & \sum_{j=k}^{\dG} {{\dG \choose j}{j \choose k} \sum_{\sigma \in
     S_{j} : \upsilon(\sigma) = j} {\mathrm{sign}(\sigma)}}\label{pKappak}\:\:.
\end{eqnarray}
The inner sum is the sum of all signs of all derangements of a set of
$j$ elements. To compute its expression, define $U_{j} \in \{0,1\}^{j\times
j}$ the matrix whose
diagonal elements are 0 and off-diagonal elements are 1. It satisfies
\begin{eqnarray}
\determinant{U_j} & = & \sum_{\sigma \in
     S_{j} : \upsilon(\sigma) = j} {\mathrm{sign}(\sigma)}\:\:. \label{propdetU}
\end{eqnarray}
$U_j$ has eigenvalue $-1$ of order $j-1$ (all vectors with $1$ and
$-1$ in two consecutive coordinates and zero elsewhere are
corresponding eigenvectors), and since its trace is zero, it also
admits $j-1$ as eigenvalue of order 1. It follows from eq. (\ref{propdetU}):
\begin{eqnarray}
\sum_{\sigma \in
     S_{j} : \upsilon(\sigma) = j} {\mathrm{sign}(\sigma)} & = & (-1)^{j-1}(j-1)\:\:,\label{ppp1}
\end{eqnarray}
out of which we obtain
\begin{eqnarray}
\kappa_k & = & \sum_{j=k}^{\dG} {{\dG \choose j}{j \choose k} \sum_{\sigma \in
     S_{j} : \upsilon(\sigma) = j} {\mathrm{sign}(\sigma)}}\nonumber\\
 & = & \sum_{j=k}^{\dG} {{\dG \choose j}{j \choose k}
   (-1)^{j-1}(j-1)}\nonumber\\
 & = & \frac{\dG !}{k!} \cdot \sum_{j=k}^{\dG} {\frac{(-1)^{j-1}(j-1)}{(\dG-j)!(j-k)!}}\nonumber\\
 & = & \frac{\dG !}{k!} \cdot \sum_{j=0}^{\dG-k} {\frac{(-1)^{j+k-1}(j+k-1)}{(\dG-k-j)!j!}}\nonumber\\
 & = & (-1)^{k-1}{\dG \choose k}\cdot \sum_{j=0}^{\dG-k} {{\dG-k \choose j}
   (-1)^j (j+k-1)}\nonumber\\
 & = & (-1)^{k-1}{\dG \choose k}\cdot \Bigg[(k-1) \cdot \underbrace{\sum_{j=0}^{\dG-k}{{\dG-k
       \choose j} (-1)^j}}_{\Sigma_k} + \underbrace{\sum_{j=1}^{\dG-k}{{\dG-k
       \choose j} (-1)^j j}}_{\Sigma'_k}\Bigg] \:\:.\label{lasst}
\end{eqnarray}
We observe that
\begin{eqnarray}
\Sigma_k & = & (-1+1)^{\dG-k} = 0 \:\:, \forall k\leq \dG - 1\:\:.\label{vsk}
\end{eqnarray}
Furthermore,
\begin{eqnarray}
\Sigma'_k & = & \sum_{j=1}^{\dG-k}{{\dG-k
       \choose j} (-1)^j j}\nonumber\\
 & = & -(\dG - k) \cdot \sum_{j=1}^{\dG-k}{{\dG-k-1
       \choose j-1} (-1)^{j-1} }\nonumber\\
 & = & -(\dG - k) \cdot \sum_{j=0}^{\dG-k-1}{{\dG-k-1
       \choose j} (-1)^{j} }\nonumber\\
 & = & -(\dG - k) (-1+1)^{\dG-k-1} \nonumber\\
 & = &  0 \:\:, \forall k\leq \dG - 2\:\:.\label{vsk2}
\end{eqnarray}
We thus get 
\begin{eqnarray}
\kappa_k & = & 0 \:\:, \forall k\leq \dG - 2\:\:.\label{vsk3}
\end{eqnarray}
Furthermore, eqs (\ref{pKappak0}) and (\ref{ppp1}) also yields:
\begin{eqnarray}
 \kappa_{\dG-1} & = & \sum_{\sigma \in
     S_{\dG} : \upsilon(\sigma) = \dG-1} {\mathrm{sign}(\sigma)} + \dG  \cdot \sum_{\sigma \in
     S_{\dG} : \upsilon(\sigma) = \dG} {\mathrm{sign}(\sigma)}\nonumber\\
 & = &  \dG \cdot \sum_{\sigma \in
     S_{\dG-1} : \upsilon(\sigma) = \dG-1} {\mathrm{sign}(\sigma)} + \dG  \cdot \sum_{\sigma \in
     S_{\dG} : \upsilon(\sigma) = \dG} {\mathrm{sign}(\sigma)}\nonumber\\
 & = & (-1)^{\dG-1} \cdot (- \dG (\dG - 2)) + (-1)^{\dG-1} \dG (\dG - 1) \nonumber\\
 & = & (-1)^{\dG-1} \dG\:\:, \label{vsk4}\\
\kappa_{\dG} & = & \sum_{\sigma \in
     S_{\dG} : \upsilon(\sigma) = \dG} {\mathrm{sign}(\sigma)}
   \nonumber\\
 & = &
   (-1)^{\dG-1}(\dG-1)\:\:. \label{vsk5}
\end{eqnarray}
Plugging eqs (\ref{vsk2}),  (\ref{vsk4}),  (\ref{vsk5}) in
(\ref{defPZ}), we obtain the simplified expression:
\begin{eqnarray}
P_{\dG}(z) & = & z^{\dG - 1} \cdot (\dG - (\dG - 1) \cdot z)\:\:,\nonumber
\end{eqnarray}
as claimed (end of the proof of Lemma \ref{eqpdg}).\end{proof}
We obtain
\begin{eqnarray}
A_a & = & (1-\cos(\thetaBstar))^{\dG - 1} \cdot (1 + (\dG - 1) \cdot \cos(\thetaBstar))\:\:.\label{valAa}
\end{eqnarray}
There remains to put together eqs (\ref{ffb2}), (\ref{llb}), (\ref{binfAB}),
(\ref{valAl}) and (\ref{valAa}) to obtain the statement of
(\ref{eqBB}) with the expression of $K$ in eq. (\ref{defKK}), and
finish the main part of the proof of Lemma \ref{lemmax}. 

To
obtain (\ref{defKKCOD}) and finish the proof of Lemma \ref{lemmax}, we
just have to remark that it follows from 
eq. (\ref{valAl}) and Definition \ref{defBSTAR} that
\begin{eqnarray}
A_l & \geq & \sqrt{A(\mutbasis^*)} (1-\kappa_n (\mutbasis^*))\nonumber\\
 & & =  \frac{1-\kappa_n (\mutbasis^*)}{\sqrt{\dG}}\cdot \left(\sum_{i\in [\dG]}
  {\|b_i\|_{\hg}^2}\right)^{\frac{1}{2}}\nonumber\\
 & \geq & (1-\kappa_n (\mutbasis^*)) \cdot \frac{\sum_{i\in [\dG]}
  {\|b_i\|_{\hg}}}{\dG}\label{ll1}\:\:,
\end{eqnarray}
where ineq.
(\ref{ll1}) comes from $p$-norm inequalities. Finally, because
$\cos(\thetaBstar)\geq 0$, Definition \ref{defBSTAR} yields:
\begin{eqnarray}
A_a & \geq & (1-\cos(\thetaBstar))^{\dG - 1}\nonumber\\
 & \geq & (1 - \kappa_a (\mutbasis^*))\label{ll2}\:\:,
\end{eqnarray} 
because of the definition of $\kappa_a (\mutbasis^*)$. 
Putting
altogether ineqs. (\ref{llb}) and (\ref{binfAB}) with the lowerbounds
on ineqs. (\ref{ll1}) and (\ref{ll2}), we obtain:
\begin{eqnarray}
\volume{\tilde{B}} & \geq &  (1-\kappa_n (\mutbasis^*)) \cdot  (1-\kappa_a (\mutbasis^*))
\cdot \frac{\sum_{i\in [\dG]}
  {\|b_i\|_{\hg}}}{\dG} \nonumber\\
 & & = B_{\hg}(\mutbasis^*)\:\:. \label{llb2}
\end{eqnarray}
Finally, combining ineq. (\ref{ffb2}) and (\ref{llb2}) yields the statement of (\ref{eqBB})
with the expression of $K$ in eq. (\ref{defKKCOD}), as claimed (end of
the proof of Lemma \ref{lemmax}).
\end{proof}\\
The following Lemma now shows that for any current representation $f$,
there always exists a mutation with guaranteed lowerbound on its
expected return minus its expected premium, where the lowerbound
depends on the $\exen$-ratio of the mutation and the
$\exen$-divergence between $f$ and target $t$.
\begin{lemma}\label{llem}
Let $\mutbasis^*\subseteq \mutbasis$ be any basis of $\hg$, and assume (\textbf{SF}) holds. Then for any distribution ${\mathcal{D}}$ and any
representations $t, f \in
{\hg}$ with coordinates expressed in basis $\mutbasis^*$, $\exists \sigma \in\{-1, 1\}, \exists  b_i \in
\mutbasis^*$ such that:
\begin{eqnarray}
\expect_{{\mathcal{D}}}
[\mathrm{R}_{f, i}(x)] - \expect_{{\mathcal{D}}}
[\Pi_{f, i}(x)]  & \geq &
\frac{\hessianmin}{\sqrt{e \dG}}
  \cdot B_{\hg}(\mutbasis^*) \cdot\rho(f, t|{{{\mathcal{D}}}}) - \alpha
{\hessianmax} \cdot \|b_i\|_{\hg} \cdot \rho(b_i|{{{\mathcal{D}}}})\:\:,\label{defess}
\end{eqnarray}
where ${\mutbasis}^* \subseteq {\mutbasis}$ is defined in Definition \ref{defBSTAR}.
\end{lemma}
\begin{proof}
We use
Lemma \ref{lemmax} and Definition (\ref{defreturn}), with which we obtain that
there exists, at any call of the mutator, polarity $\sigma \in
\{-1,+1\}$ and basis vector $b_i \in
\mutbasis^*$ such that:
\begin{eqnarray}
\max_{b_i \in \mutbasis^*} {|\expect_{{\mathcal{D}}}[\mathrm{R}_{f, i} (x)] |} & \geq
& 
 \frac{\hessianmin}{\sqrt{e \dG}}
 \cdot B_{\hg}(\mutbasis^*) \cdot\rho(f, t|{{{\mathcal{D}}}})\label{eqBB4}\:\:.
\end{eqnarray}
On the expected premium's side, we have for any $b_i \in \mutbasis^*$
whose coordinates are given in $\mutbasis^*$,
\begin{eqnarray}
\expect_{{\mathcal{D}}}
[\Pi_{f,i}(x)] & \defeq &
\frac{1}{\alpha} \expect_{{\mathcal{D}}}
[D_\varphi(f (x)
- \alpha \cdot (-\sigma) b_i (x) \| f (x))] \nonumber\\
 & \leq & \frac{\alpha \hessianmax}{2} \cdot \langle \matrice{b}^* b_i, \gendistnone
  \matrice{b}^* b_i\rangle \label{eqcoord3}\\
 &  & = \frac{\alpha \hessianmax}{2}  \cdot \expect_{{\mathcal{D}}}
[\|b_i (x)\|^2_2]\label{eqcoord2}\\
 &  = & \frac{\alpha \hessianmax}{2}  \cdot \|b_i\|_{\hg} \cdot \rho(b_i|{{{\mathcal{D}}}})\label{lastbpp}\:\:, \forall \sigma \in \{-1,1\},
\forall b_i \in \mutbasis^*\:\:.
\end{eqnarray}
Eq. (\ref{eqcoord3}) uses the definition of $\matrice{m}$ in
(\ref{defMM}) and Lemma \ref{lemapproxbregdiv}.  Eq. (\ref{eqcoord2}) come from the fact that $\matrice{b}^* b_i,$ gives
the coordinates of $b_i$ in the orthonormal gene basis of $\hg$. 
Putting altogether (\ref{eqBB4}) and (\ref{lastbpp}) with
$1/2$ factor dropped yields the
statement of the Lemma. 
\end{proof}\\
Notice that terms $B_{\hg} \cdot \rho(f, t|{{{\mathcal{D}}}})$ and
$\|b_i\|_{\hg} \cdot \rho(b_i|{{{\mathcal{D}}}})$ are homogeneous in (\ref{defess}), as both quantify a
\exen-ratio or divergence, weighted by a (corrected) length of the encoding.

It comes from Lemma \ref{llem} that in the (\textbf{SF}) setting, for
some evolution step $j\in [T]_*$, \textit{if} mutant $f_j$ satisfies,
for some $\sigma \in\{-1, 1\}$ and $b_i \in
\mutbasis^*$,
\begin{eqnarray}
\rho(f_j, t|{{{\mathcal{D}}}}) & \geq &
\frac{\sqrt{e \dG}}{\hessianmin B_{\hg}(\mutbasis^*)}\left(\tau
  + \alpha
{\hessianmax} \cdot \|b_i\|_{\hg} \cdot \rho(b_i|{{{\mathcal{D}}}}) + \frac{\tol}{\alpha}
\right)\:\:,
\end{eqnarray}
then, chaining with ineq. (\ref{defess}), we get
\begin{eqnarray}
\expect_{{\mathcal{D}}}
[\mathrm{R}_{f_j, i}(x)] - \expect_{{\mathcal{D}}}
[\Pi_{f_j, i}(x)]  & \geq &
\frac{\hessianmin}{\sqrt{e \dG}}
  \cdot B_{\hg}(\mutbasis^*) \cdot\rho(f_j, t|{{{\mathcal{D}}}}) - \alpha
{\hessianmax} \cdot \|b_i\|_{\hg} \cdot
\rho(b_i|{{{\mathcal{D}}}})\nonumber\\
 & \geq & \frac{\tol}{\alpha} + \tau \:\:, \label{conctrue}
\end{eqnarray}
\textit{i.e.} the mutation involving $b_i \in \mutbasis^*$ is superior
beneficial. This does not show however that the mutator will \textit{pick} one of these
mutations, and it does not show that the mutator will pick
\textit{some} $b_i \in \mutbasis^*$. In fact, this is not even enough to show that $\bene(f)$ is
not empty since it involves estimates (Lemma \ref{lemmut}). To show
that $\bene(f)$ is
not empty, we now show that with high probability the estimates $\expect_{{\mathcal{S}}}
[\mathrm{R}_{f_j, i}(x)]$ and $\expect_{{\mathcal{S}}}
[\Pi_{f_j, i}(x)]$ are both within $\tau/2$ of their true values,
implying in this case from ineq. (\ref{conctrue})
\begin{eqnarray}
\expect_{{\mathcal{S}}}
[\mathrm{R}_{f_j, i}(x)] - \expect_{{\mathcal{S}}}
[\Pi_{f_j, i}(x)] & \geq & \left(\frac{\tol}{\alpha} +
  \tau\right) - \tau = \frac{\tol}{\alpha}\label{gmut}\:\:,
\end{eqnarray} 
and thus $\bene(f_j)$ as defined in (\ref{propbene}) is indeed not
empty. Remark this relies on the sole assumption that $f_j \in
{\hg}_{t, {\mathcal{D}}}$; $f_j$ may \textit{not} be in $\hgmon$. The
constraint that $f_j\in \hgmon$ shall be used to compute the
number of conditions needed for ineq. (\ref{gmut}) to hold with high
probability in the sequence $\hgmon$, which shall then be used
to prove evolvability. 
\begin{lemma}\label{lemdev}
Let $\mutbasis^*\subseteq \mutbasis$ be any basis of $\hg$, and assume (\textbf{SF}) holds. Then for any distribution ${\mathcal{D}}$ and any
representations $t, f \in
{\hg}$ with coordinates expressed in basis $\mutbasis^*$,  the following inequalities hold
over the i.i.d. sampling of ${\mathcal{S}}$ (of size
$m$) according to ${\mathcal{D}}$:
\begin{eqnarray}
\lefteqn{\pr_{\mathcal{S}\sim \mathcal{D}}\left[\exists i, \sigma: |\expect_{{\mathcal{S}}}
[\mathrm{R}_{f, i}(x)] - \expect_{{\mathcal{D}}}
[\mathrm{R}_{f, i}(x)]| \geq \tau \right]}\nonumber\\
 & \leq &  2\dF
\exp\left( -\frac{B_{\hg}(\mutbasis^*)^2 m
\tau^2}{2e \hessianmax^2 \|t-f\|^2_{\hg} \sup_x \single^2(\mutbasis^*, x)} \right)
\:\:, \label{erd}
\end{eqnarray}
\begin{eqnarray}
\lefteqn{\pr_{\mathcal{S}\sim \mathcal{D}}\left[\exists i, \sigma: |\expect_{{\mathcal{S}}}
[\Pi_{f, i}(x)] - \expect_{{\mathcal{D}}}
[\Pi_{f, i}(x)]|  \geq \tau \right]}\nonumber\\
 &
\leq &  2\dF \exp\left( -\frac{2m \tau^2}{\alpha^2 {\hessianmax}^2
\sup_x \single^2(\mutbasis^*, x)} \right)  \:\:.\label{pd}
\end{eqnarray}
\end{lemma}
\begin{proof}
Both bounds are direct applications of the independent bounded
differences inequality (IBDI, \citep{mdC}). We first prove
(\ref{erd}). Again, we let  ${\mathcal{B}}^*\defeq \{b_1, b_2, ...,
b_{\dG}\}$ without loss of generality, and $t - f = \delta \defeq
\sum_{i\in [\dG]} \delta^i
b_i$. Take any $\sigma \in \{-1,1\}$,
$b_i \in \mutbasis$ (the coordinates of this being in
$\mutbasis^*$). Using notations from Lemma \ref{lemmareturndec} (eq. (\ref{taylor1})), we have
\begin{eqnarray}
\mathrm{R}_{f, i}(x) & \defeq & \langle \sigma b_i(x),
   (\nabla\varphi \circ t)(x) - (\nabla\varphi \circ f)(x)\rangle
   \nonumber\\
& = &
\sigma\cdot \langle \matrice{g}_x\matrice{b}^* b_i,  \matrice{h}(x)
 \matrice{g}_x\matrice{b}^* \delta \rangle\nonumber\\
 & \leq & \hessianmax \cdot |\langle \matrice{g}_x\matrice{b}^* b_i, 
 \matrice{g}_x\matrice{b}^* \delta \rangle|\label{tayl3}\\
 & \leq & \hessianmax \cdot
\sqrt{\langle \matrice{g}_x\matrice{b}^*b_i, \matrice{g}_x\matrice{b}^*b_i\rangle}\cdot \sqrt{\langle
  \matrice{g}_x\matrice{b}^* \delta,
 \matrice{g}_x\matrice{b}^* \delta\rangle}\label{tayl4}\\
 & & = \hessianmax \cdot \|b_i (x)\|_2  \cdot \|t(x)-f(x)\|_2\:\:, \forall
 x \in {\mathcal{X}}\label{lastb}\:\:.
\end{eqnarray}
Ineq. (\ref{tayl3}) follows from the definition of $\hessianmax$, ineq. (\ref{tayl4}) is Cauchy-Schwartz. 
Finally, we know that $\forall x \in \mathcal{X}$, letting $\delta \defeq
t - f$,
\begin{eqnarray}
\|t(x)-f(x)\|^2_2 & \defeq & \langle \matrice{g}_x\matrice{b}^* \delta,
\matrice{g}_x\matrice{b}^* \delta\rangle\nonumber\\
 & = & \left\langle \sum_{i\in [\dG]} \delta^i \cdot \matrice{g}_x
   \matrice{b}^*b_i, \sum_{i\in [\dG]}\delta^i \cdot \matrice{g}_x \matrice{b}^*b_i\right\rangle\:\:.\label{transit1}
\end{eqnarray}
For any reals $a_i$ and vectors $u_i$
 ($i\in [\dG]$), we have $\langle \sum_{i\in [\dG]} a_i \cdot u_i, \sum_{i\in
  [\dG]} a_i \cdot u_i\rangle \leq \sum_{i, i'\in [\dG]} a_i \|u_i\|
a_{i'}\|u_i\| =(\sum_{i\in [\dG]} a_i \|u_i\|)^2\leq  (\sum_{i\in
  [\dG]} a^2_i )\cdot (\sum_{i\in
  [\dG]} \|u_i\|^2 )$ (from Cauchy-Schwartz inequality), and so
eq. (\ref{transit1}) yields
\begin{eqnarray}
\|t(x)-f(x)\|^2_2 & \leq & \left( \sum_{i\in [\dG]}
  (\delta^i)^2\right) \cdot \sum_{i\in [\dG]} \langle \matrice{g}_x\matrice{b}^*
b_i, \matrice{g}_x \matrice{b}^*b_i\rangle\nonumber\\
 & & = \left( \sum_{i\in [\dG]}
  (\delta^i)^2\right) \cdot \sum_{i\in [\dG]} \|b_i (x)\|^2_2\nonumber\\
& \leq &  \frac{e \|t-f\|^2_{\hg}}{
   B_{\hg}(\mutbasis^*)^2}\cdot \sum_{i\in [\dG]} \|b_i (x)\|^2_2\label{eqqq1}\:\:.
\end{eqnarray}
Ineq. (\ref{eqqq1}) follows from ineqs 
(\ref{propbetaprime}) and (\ref{llb2}).
Putting altogether ineqs (\ref{lastb}) and (\ref{eqqq1}),
we get
\begin{eqnarray}
\mathrm{R}_{f, i}(x) & \leq & {\hessianmax} \|b_i (x)\|_2\cdot
\frac{\sqrt{e} \|t-f\|_{\hg}}{B_{\hg}(\mutbasis^*)} \cdot \left(\sum_{i'\in [\dG]} \|b_{i'} (x)\|^2_2\right)^{\frac{1}{2}} \nonumber\\
 & \leq & \frac{\sqrt{e} \hessianmax
   \|t-f\|_{\hg}}{B_{\hg}(\mutbasis^*)} \cdot \sum_{i'\in [\dG]\cup\{i\}}
 \|b_{i'} (x)\|^2_2 \nonumber\\
 & & = \sqrt{e} \hessianmax
   \|t-f\|_{\hg} \cdot \frac{\single(\mutbasis^*\cup \{b_i\}, x)}{B_{\hg}(\mutbasis^*)} \nonumber\\
 & \leq & \sqrt{e} \hessianmax
   \|t-f\|_{\hg} \cdot \frac{\single(\mutbasis^*, x)}{B_{\hg}(\mutbasis^*)} \:\:, \label{eqqq3}
\end{eqnarray}
by the properties of $\mutbasis^*$.
Hence, between two sets ${\mathcal{S}}$ and ${\mathcal{S}}'$ of
the same size $m$ and that
would differ from a single element, we have:
\begin{eqnarray}
\lefteqn{|\expect_{{\mathcal{S}}}
[\mathrm{R}_{f, i}(x)] - \expect_{{\mathcal{S}}'}
[\mathrm{R}_{f, i}(x)]|}\nonumber\\
 & \leq & \frac{2}{m}\cdot
\sqrt{e} \hessianmax \|t-f\|_{\hg}
\cdot \frac{\sup_x \single(\mutbasis^*, x)}{B_{\hg}(\mutbasis^*)} \defeq
\gamma_\Delta(\mutbasis^*)\:\:.\label{defgammad}
\end{eqnarray} 
The following bound follows from
\citep{mdC} (Theorem 3.1) and the union bound over ${\mutbasis}$:
\begin{eqnarray*}
\pr_{\mathcal{S}\sim \mathcal{D}}\left[\exists i: |\expect_{{\mathcal{S}}}
[\mathrm{R}_{f, i}(x)] - \expect_{{\mathcal{D}}}
[\mathrm{R}_{f, i}(x)]| \geq \tau \right] & \leq &  2\dF \exp\left( -\frac{2 \tau^2}{m 
\gamma^2_{\Delta}(\mutbasis^*) } \right) 
\:\:.
\end{eqnarray*}
We replace $\gamma_\Delta (\mutbasis^*)$ by its
expression in ineq. (\ref{defgammad}) and obtain (\ref{erd}) .

We proceed in the same way for the expected mutator's premium
(\ref{pd}). We know from eq. (\ref{eqbreghess}) that, for some value
of the Hessian $\matrice{h}$ of $\varphi$, we have
\begin{eqnarray}
\Pi_{f,i}(x) & = & \frac{\alpha}{2} \cdot \langle \matrice{g}_x
\matrice{b}^* b_i, \matrice{h}
 \matrice{g}_x\matrice{b}^* b_i\rangle\nonumber\\
 & \leq & \frac{\alpha \hessianmax}{2} \cdot |\langle \matrice{g}_x
\matrice{b}^* b_i,
 \matrice{g}_x\matrice{b}^* b_i\rangle|\nonumber\\
 & & = \frac{\alpha \hessianmax}{2} \cdot \|b_i(x)\|_2^2\:\:,
\end{eqnarray}
where the inequality comes from Lemma \ref{lemphi}. So, the variation in average premium between two sets ${\mathcal{S}}$ and ${\mathcal{S}}'$ of
the same size $m$ and that
would differ from a single condition satisfies:
\begin{eqnarray}
|\expect_{{\mathcal{S}}}
[\Pi_{f, i}(x)] - \expect_{{\mathcal{S}}'}
[\Pi_{f, i}(x)]| & \leq & \frac{1}{m} \cdot \alpha {\hessianmax}
\cdot \sup_x \single(\{i\}, x)\nonumber\\
 & \leq & \frac{1}{m} \cdot \alpha {\hessianmax}
  \cdot \sup_{i', x} \single(\{i'\}, x)\nonumber\\
 & \leq & \frac{1}{m} \cdot \alpha {\hessianmax}
  \cdot \sup_x \single(\mutbasis^*, x)\:\:,
\end{eqnarray} 
by the properties of $\mutbasis^*$,
which yields, out of the IBDI \citep{mdC} and the union bound
over ${\mutbasis}$, the following bound:
\begin{eqnarray*}
\lefteqn{\pr_{\mathcal{S}\sim \mathcal{D}}\left[\exists i, \sigma: |\expect_{{\mathcal{S}}}
[\Pi_{f, i}(x)] - \expect_{{\mathcal{D}}}
[\Pi_{f, i}(x)]|  \geq \tau \right]}\nonumber\\
 &
\leq & 2\dF \exp\left( -\frac{2m \tau^2}{\alpha^2 {\hessianmax}^2
\sup_x \single^2(\mutbasis^*, x)} \right) \:\:,
\end{eqnarray*}
as claimed.
\end{proof}\\
Suppose now that the mutator samples a sufficient
number $m$ of examples that would ensure that the right hand-sides of
(\ref{erd}) and (\ref{pd}) are no more than $\epsilon/T$ for deviation
$\tau/2$ (and not $\tau$). By the
union bound, the probability that there exists a run (among the $T$) of
the mutator, and some $i,\sigma$ such that one of the averages on ${\mathcal{S}}$ of $\mathrm{R}_{\varphi,
  i}(x)$ or $\Pi_{f, i}(x)$ deviates from its respective
expectation on ${\mathcal{D}}$ by more than $\tau/2$ is no more than
$T\cdot (\epsilon/T) = \epsilon$. Hence, with probability $\geq
1-\epsilon$, we shall have at all $j\in [T]_*$ runs of
the mutator and for the corresponding $\sigma$, $b_i$ picked at each
call of the mutator:
\begin{eqnarray}
\lefteqn{|(\expect_{{\mathcal{S}}}
[\mathrm{R}_{f_j, i}(x)] -\expect_{{\mathcal{S}}}
[\Pi_{f_j,i}(x)]) - (\expect_{{\mathcal{D}}}
[\mathrm{R}_{f_j, i}(x)] -\expect_{{\mathcal{D}}}
[\Pi_{f_j,i}(x)])|}\nonumber\\
 & \leq & |\expect_{{\mathcal{S}}}
[\mathrm{R}_{f_j, i}(x)] - \expect_{{\mathcal{D}}}
[\mathrm{R}_{f_j, i}(x)]| + |\expect_{{\mathcal{S}}}
[\Pi_{f_j,i}(x)] - \expect_{{\mathcal{D}}}
[\Pi_{f_j,i}(x)]| \nonumber\\
 & \leq & (\tau/2) + (\tau/2) = \tau\:\:,\label{eqdiv}
\end{eqnarray}
and hence ineq. (\ref{gmut}) holds over all $T$ iterations, thus
whenever $f_j \in {\hg}_{t, {\mathcal{D}}}$, even when the mutator
does not pick superior beneficial mutations, or even mutations from
$\mutbasis^*$, it still has non-empty $\bene(f)$, as claimed in (\ref{defalphatauzero}).

Using Lemma \ref{lemdev}, the sufficient number of conditions $m$ that ensures this is
found to be any integer that satisfies:
\begin{eqnarray}
m & = & \Omega\left(\frac{{\hessianmax}^2  \sup_{x} \single^2(\mutbasis^*, x)}{\tau^2}\left(\frac{\max_j \|t-f_j\|_{\hg}^2}{B_{\hg}(\mutbasis')^2}+\alpha^2\right) \log\left(\frac{\dF
      T}{\epsilon}\right)\right)\label{constm2}\:\:.
\end{eqnarray}

\noindent $\hookrightarrow$ (Proof of [\textbf{BB.2}]) The bound in
ineq. (\ref{constm2}) may be problematic as little tells us about
$\|t-f_j\|_{\hg}$ and how big it can be. Fortunately, it is sufficient
for evolution that we focus on subset $\hgmon \subseteq
{\hg}_{t, {\mathcal{D}}}$. In this subset, we can bound
$\|t-f_j\|_{\hg}$, in a very simple way. Recall that sampling a number
of example that complies with ineq. (\ref{constm2})  is
sufficient to guarantee with high probability that $\expect_{{\mathcal{D}}}
[\mathrm{R}_{f_j, i}(x)] - \expect_{{\mathcal{D}}}
[\Pi_{f_j, i}(x)] \geq (\tol/\alpha) - \tau$ where $b_i$ is the mutation
\textit{picked} by the mutator, which implies in particular
\begin{eqnarray}
\lefteqn{\perf_{t,\varphi}(f+\sigma \alpha b_i,
{\mathcal{D}})}\nonumber\\
 & \geq & \perf_{t,\varphi}(f,
{\mathcal{D}}) + \alpha\left(\frac{\tol}{\alpha} - \tau\right) \nonumber\\
 & & = \perf_{t,\varphi}(f,
{\mathcal{D}}) + \tol - \alpha \tau \:\:.
\end{eqnarray}
But
\begin{eqnarray}
\tol - \alpha \tau & = & \frac{z_{\tol}}{U^2\max\{1, V\}}\cdot \epsilon^{2} -
\frac{z_{\alpha}z_{\tau}}{U^2\max\{1, V\}}\cdot \epsilon^{2}\nonumber\\
 & = & \frac{z_\tol - z_\alpha z_\tau}{U^2\max\{1, V\}}\cdot \epsilon^{2}\nonumber\\
 & > & 0\:\:,
\end{eqnarray}
because $(z_\tau, z_\alpha, z_\tol)\in
\mathfrak{R}$ (Definition \ref{defNER}). Hence, as long as $f \in {\hg}_{t,
  {\mathcal{D}}}$ (and so, $f\in \hgmon$), we are guaranteed that 
\begin{eqnarray}
\perf_{t,\varphi}(f_j,
{\mathcal{D}}) & \geq & \perf_{t,\varphi}(f_0, {\mathcal{D}})\:\:.\label{mperf}
\end{eqnarray} 
To obtain
our bound on $\|t-f_j\|_{\hg}$, we need the following Lemma.
\begin{lemma}\label{lemfund1}
Under setting (\textbf{SF}), $-(1/2) \hessianmax \expressmax
 \|t-f\|^2_{\hg}\leq \perf_{t,\varphi}(f, {\mathcal{D}})\leq - (1/2) \hessianmin \expressmin
 \|t-f\|^2_{\hg}$, $\forall f,t\in \hg$.
\end{lemma}
\begin{proof}
Let $\delta \defeq f - t$. We get from Lemma \ref{lemapproxbregdiv}
and the definition of $\gendistnone$,
\begin{eqnarray}
\expect_{x \sim {\mathcal{D}}} [D_\varphi(f(x) \| t(x))] & \leq & \frac{\hessianmax}{2}\cdot \langle
 \matrice{b}^* \delta, \gendistnone\matrice{b}^* \delta \rangle\nonumber\\
 & \leq & \frac{\hessianmax \expressmax}{2}\cdot \langle
 \matrice{b}^* \delta, \matrice{b}^* \delta \rangle\nonumber\\
 & & = \frac{\hessianmax \expressmax}{2}\cdot \|\delta\|^2_{\hg}\:\:.
\end{eqnarray} 
We would obtain similarly $\expect_{x \sim {\mathcal{D}}}
 [D_\varphi(f(x) \| t(x))] \geq (\hessianmin \expressmin)/2
 \|t-f\|^2_{\hg}$.
\end{proof}
Using
Lemma \ref{lemfund1} brings
\begin{eqnarray}
\|t-f_j\|^2_{\hg} & \leq & -\frac{2}{\hessianmin \expressmin}\cdot
\perf_{t,\varphi}(f_j, {\mathcal{D}}) \nonumber\\
 & \leq & -\frac{2}{\hessianmin \expressmin}\cdot
\perf_{t,\varphi}(f_0, {\mathcal{D}}) \label{ingmon}\\
 & \leq & \frac{\hessianmax \expressmax}{\hessianmin \expressmin}\cdot
\|t-f_0\|^2_{\hg} \nonumber\\
 &  & = \frac{\hessianmax \expressmax}{\hessianmin \expressmin}\cdot
\darwin^2 \max_{i\in [\dF]_*}
   \|b_{i}\|^2_{\hg} \:\:, \forall f_j \in \hgmon\:\:,\label{seqdec}
\end{eqnarray}
because of the definition of $\darwin$. Ineq. (\ref{ingmon}) comes
from ineq. (\ref{mperf}).
Using the last inequality (\ref{seqdec}), we
obtain that a sufficient condition for $m$ to meet (\ref{constm2}) is:
\begin{eqnarray}
m & = & \Omega\left(\frac{{\hessianmax}^2  \sup_{x} \single^2(\mutbasis^*, x)}{\tau^2}\left(\frac{\hessianmax
        \expressmax}{\hessianmin \expressmin} \cdot \frac{\max_{i\in [\dF]_*}
   \|b_{i}\|^2_{\hg}}{B^2_{\hg}(\mutbasis')} \cdot \darwin^2 +\alpha^2\right) \log\left(\frac{\dF
      T}{\epsilon}\right)\right) \nonumber\:\:,
\end{eqnarray}
but since $\max_{i\in [\dF]_*}
   \|b_{i}\|_{\hg} = \max_{b_i \in \mutbasis^*}
   \|b_{i}\|_{\hg}$ by the properties of $\mutbasis^*$, then it is sufficient that
\begin{eqnarray}
m & = & \Omega\left(\frac{{\hessianmax}^2  \sup_{x} \single^2(\mutbasis^*, x)}{\tau^2}\left(\frac{\hessianmax
        \expressmax}{\hessianmin \expressmin} \cdot \frac{\darwin^2}{\overline{B}^2_{\hg}(\mutbasis^*)} +\alpha^2\right) \log\left(\frac{\dF
      T}{\epsilon}\right)\right) \nonumber\:\:,
\end{eqnarray}
which is eq. (\ref{numevom22}). This achieves the proof of Theorem \ref{th00}.

\subsection{Proof of Lemma \ref{lemconv1}}\label{prooflemconv1}
Because of the definition of ${\hg}_{t, {\mathcal{D}}}$, whenever $f
\not\in {\hg}_{t, {\mathcal{D}}}$, 
\begin{eqnarray}
 \rho(f, t|{{{\mathcal{D}}}}) < \frac{\sqrt{e \dG}}{\hessianmin B_{\hg}({\mutbasis}^*)}\left(\tau
  + \alpha
{\hessianmax} \cdot \max_{b_i\in \mutbasis^*} \{\|b_i\|_{\hg} \cdot \rho(b_i|{{{\mathcal{D}}}})\} + \frac{\tol}{\alpha}
\right)\:\:.\label{deffjstar}
\end{eqnarray}
If $f_{j_\star-1}$ is in $\hgmon$ but $f_{j_\star}$ is not in 
$\hgmon$, then ineq. (\ref{deffjstar}) is satisfied by $f_{j_\star}$, along
with, because of ineq. (\ref{seqdec}) on $f_{j_\star-1}$ and the triangle inequality,
\begin{eqnarray}
\|t-f_{j_\star}\|_{\hg} & \leq & \|t-f_{j_\star-1}\|_{\hg} +
\|f_{j_\star-1}-f_{j_\star}\|_{\hg}\nonumber\\
 & \leq & \left(\frac{\hessianmax \expressmax}{\hessianmin \expressmin}\right)^{\frac{1}{2}}\cdot
\darwin \max_{i\in [\dF]_*}
   \|b_{i}\|_{\hg} + \|f_{j_\star-1}-f_{j_\star}\|_{\hg}\nonumber\\
 & & = \left(\frac{\hessianmax \expressmax}{\hessianmin \expressmin}\right)^{\frac{1}{2}}\cdot
\darwin  \max_{i\in [\dF]_*}
   \|b_{i}\|_{\hg} +\alpha \|b\|_{\hg}\label{defbmut}\\
 & \leq & 2\cdot \left(\frac{\hessianmax \expressmax}{\hessianmin \expressmin}\right)^{\frac{1}{2}}\cdot
\darwin  \max_{i\in [\dF]_*}
   \|b_{i}\|_{\hg} \:\:,\label{eqfjstar}
\end{eqnarray}
since $U\geq 1$ and for any $(z_\tau, z_\alpha, z_\tol) \in
\mathfrak{R}$, we can show that we have $z_\alpha\leq 1$, implying
$\alpha \leq 1$. We have also let $b\defeq f_{j_\star}-f_{j_\star-1}
\in \mutbasis$ in eq. (\ref{defbmut}).
Multiplying ineq (\ref{deffjstar}) by
${\hessianmax}\|t-f_{j_\star}\|_{\hg}$ and using ineq. (\ref{eqfjstar}) yields
\begin{eqnarray}
\lefteqn{{\hessianmax}\|t-f_{j_\star}\|_{\hg} \cdot
\rho(f_{j_\star}, t|{{{\mathcal{D}}}})}\nonumber\\
 & < & {\hessianmax}\|t-f_{j_\star}\|_{\hg} \cdot
\frac{\sqrt{e \dG}}{\hessianmin B_{\hg}({\mutbasis}^*)}\left(\tau
  + \alpha
{\hessianmax} \cdot \max_{b_i\in \mutbasis^*}\{\|b_i\|_{\hg} \cdot \rho(b_i|{{{\mathcal{D}}}})\} + \frac{\tol}{\alpha}
\right) \nonumber\\
& \leq & \frac{\hessianmax^{\frac{3}{2}}
  \expressmax^{\frac{1}{2}}}{\hessianmin^{\frac{3}{2}}
  \mu^{\frac{1}{2}}_{\hg}} \cdot 2\sqrt{e \dG} \cdot \frac{\max_{b_i\in \mutbasis^*}
   \|b_{i}\|_{\hg} }{B_{\hg}({\mutbasis}^*)}\cdot
\darwin \cdot \left(\tau
  + \alpha
{\hessianmax} \cdot \max_{b_i\in \mutbasis^*}\{\|b_i\|_{\hg} \cdot \rho(b_i|{{{\mathcal{D}}}})\} + \frac{\tol}{\alpha}
\right) \label{bfond1}\\
&  & =  \underbrace{\frac{\hessianmax^{\frac{3}{2}} \expressmax^{\frac{1}{2}}}{\hessianmin^{\frac{3}{2}} \mu^{\frac{1}{2}}_{\hg}} \cdot \frac{2\sqrt{e \dG} }{\overline{B}_{\hg}({\mutbasis}^*)}\cdot
\darwin}_{=U}\cdot \left(\tau
  + \alpha
{\hessianmax} \cdot \max_{i\in [\dF]_*} \{\|b_i\|_{\hg} \cdot \rho(b_i|{{{\mathcal{D}}}})\} + \frac{\tol}{\alpha}
\right) \label{bfond12}\\
 & \leq & U\tau
  + \alpha U \max\{1, V\} + \frac{U \tol}{\alpha} \nonumber\\
 & = & z_\tau\cdot
\epsilon+ z_\alpha\cdot
\epsilon+\frac{z_\tol}{z_\alpha}\cdot
\epsilon\:\:,
\end{eqnarray}
by definition of $U,\alpha,V,\tol$ in eqs. (\ref{defU}, \ref{defV},
\ref{defTol}, \ref{defAlpha}).
Ineq. (\ref{bfond1}) holds because of ineq. (\ref{seqdec}). Hence,
we obtain that
\begin{eqnarray}
\expect_{x \sim {\mathcal{D}}} [D_\varphi(f_{j_\star}(x) \| t(x))] & \leq & {\hessianmax}\|t-f_{j_\star}\|_{\hg} \cdot
\rho(f_{j_\star}, t|{{{\mathcal{D}}}})\nonumber\\
 & \leq & z_\tau\cdot
\epsilon+ z_\alpha\cdot
\epsilon+\frac{z_\tol}{z_\alpha}\cdot
\epsilon \nonumber\\
 & & = \left( z_\tau + z_\alpha + \frac{z_\tol}{z_\alpha}\right) \cdot
 \epsilon\label{condconv}\\
 & \leq & \epsilon\:\:,\label{boundzzp}
\end{eqnarray}
because $(z_\tau, z_\alpha, z_\tol) \in
\mathfrak{R}$ and condition (iii) in Definition \ref{defNER}. Ineq. (\ref{boundzzp}) is
equivalent to:
\begin{eqnarray}
\perf_{t,\varphi}(f_{j_\star}, {\mathcal{D}}) & \geq & - \epsilon\:\:,\nonumber
\end{eqnarray}
and so $f_{j_\star}\in \hg_t$, as claimed.

\subsection{Proof of Lemma \ref{lemmpack}}\label{prooflemmpack}

We first provide the complete expression of $\tau$,
\begin{eqnarray}
\tau & \defeq & \frac{z_\tau}{U}\cdot
\epsilon\:\:.\label{defTau}
\end{eqnarray}
We also remove subscripts in organisms for clarity.
Under the conditions of Theorem \ref{th00}, when the current organism
$f \in \hgmon \subseteq {\hg}_{t,
  {\mathcal{D}}}$, we have $\bene(f) \neq \emptyset$ and thus the next
mutation picks polarity $\sigma \in \{-1,1\}$ and $b_i\in
{\mutbasis}$ such that $\expect_{{\mathcal{S}}}
[\mathrm{R}_{f, i}(x)] - \expect_{{\mathcal{S}}}
[\Pi_{f, i}(x)]   \geq  \tol/\alpha$. The left-hand side is an average
computed over the mutator's sample ${\mathcal{S}}$. Its difference
with its true value (expectation over ${\mathcal{D}}$) in absolute
value is no more than $\tau$ over all $T$ iterations with probability
$\geq 1-\epsilon$ when $m$ meets bound (\ref{numevom22}). Hence, with
probability $\geq 1-\epsilon$, the mutation will always exhibit $\expect_{{\mathcal{D}}}
[\mathrm{R}_{f, i}(x)] - \expect_{{\mathcal{D}}}
[\Pi_{f, i}(x)] \geq (\tol/\alpha) - \tau$, and so, using the
definition of $\bene$ in eq. (\ref{propbene}) and the expressions of $\tol,
\alpha, \tau$ in eqs (\ref{defTol}, \ref{defAlpha}, \ref{defTau}),
\begin{eqnarray}
\lefteqn{\perf_{t,\varphi}(f+\sigma \alpha b_i,
{\mathcal{D}})}\nonumber\\
 & \geq & \perf_{t,\varphi}(f,
{\mathcal{D}}) + \alpha\left(\frac{\tol}{\alpha} - \tau\right) \nonumber\\
 & & = \perf_{t,\varphi}(f,
{\mathcal{D}}) + \tol - \alpha \tau \nonumber\\
 & = & \perf_{t,\varphi}(f,
{\mathcal{D}}) + \frac{z_{\tol}}{U^2V}\cdot \epsilon^{2} -
\frac{z_{\alpha}z_{\tau}}{U^2V}\cdot \epsilon^{2}  \nonumber\\
 & = & \perf_{t,\varphi}(f,
{\mathcal{D}}) + \frac{z_\tol - z_\alpha z_\tau}{U^2V}\cdot \epsilon^{2} \label{defincre}\\
 & = & \perf_{t,\varphi}(f,
{\mathcal{D}})  + \Omega(\epsilon^{2})\:\:, \nonumber
\end{eqnarray}
since $(z_\tau, z_\alpha, z_\tol)\in
\mathfrak{R}$ (condition (ii), Definition \ref{defNER}).  
Hence, as long as $f \in \hgmon$, we have a guaranteed
increase in performances given by ineq. (\ref{defincre}). If $f$ never
leaves $\hgmon$, how long would it
take for evolvability conditions to be met ?
To compute it, we need the following Lemma.
\begin{lemma}\label{lemmaxx22}
Assume (\textbf{SF}) holds. The initial representation $f_0$ satisfies 
\begin{eqnarray}
\expect_{{\mathcal{D}}}
[D_\varphi(f_0(x)||t(x))] & \leq & \frac{\hessianmax \expressmax}{2} \max_{b_i\in \mutbasis^*}
   \|b_{i}\|^2_{\hg}\cdot \darwin^2 \label{eqtobeat} \:\:.
\end{eqnarray}
\end{lemma}
\begin{proof}
We have from the definition of $\darwin$,
\begin{eqnarray}
\expect_{x \sim {\mathcal{D}}} [D_\varphi(f_0(x) \| t(x))] & \defeq & -\perf_{t,\varphi}(f_0, {\mathcal{D}})
\nonumber\\
 & \leq &
\frac{\hessianmax \expressmax}{2} \|f_0 - t\|_\hg^2\label{ineq111}\\
 & \leq & \frac{\hessianmax \expressmax}{2} \max_{i\in [\dG]_*}
   \|b_{i}\|^2_{\hg} \cdot \darwin^2\nonumber\\
 & & = \frac{\hessianmax \expressmax}{2} \max_{b_i\in \mutbasis^*}
   \|b_{i}\|^2_{\hg} \cdot \darwin^2\:\:.\nonumber
\end{eqnarray}
Ineq. (\ref{ineq111}) comes from Lemma \ref{lemfund1}. The last
identity comes from the properties of $\mutbasis^*$.\end{proof}
The way we use Lemma
\ref{lemmaxx22} is the following: when the number of iterations $T$
times the minimal performance variation in ineq. (\ref{defincre}) exceeds ineq
(\ref{eqtobeat}), then with high probability $f \in {\hg}_t$ and
thus meets
the condition of convergence for evolvability in
ineq. (\ref{condgen}) (Definition \ref{defevolvability}).
So, we want
\begin{eqnarray}
T & \geq & \frac{U^2V \hessianmax \expressmax \max_{i\in [\dG]_*}
   \|b_{i}\|^2_{\hg}}{z_\tol - z_\alpha z_\tau} \cdot \frac{\darwin^2}{2\epsilon^{2}}\:\:.\label{consttt1}
\end{eqnarray}
Taking into account the expression of $U$ and $V$, it is sufficient that
\begin{eqnarray}
T & \geq & \frac{1}{z_\tol - z_\alpha z_\tau} \cdot
\frac{\hessianmax^5 \expressmax^2 \max_{i\in [\dF]_*} \{
\expect_{{\mathcal{D}}}
[\|b_i(x)\|^2_2]}{\hessianmin^3 \expressmin} \cdot \frac{e \dG\max_{i\in [\dF]_*}
   \|b_{i}\|^4_{\hg} }{B^2_{\hg}({\mutbasis}^*)}\cdot
\frac{\darwin^4}{\epsilon^{2}}\label{numiterevol}\:\:,
\end{eqnarray}
that is, disregarding absolute constants and all other parameters in the $\tilde{O}$
notation, it is sufficient that
\begin{eqnarray}
T & = & \tilde{O} \left(\frac{\darwin^4}{\epsilon^{2}}\right)\:\:, \label{bsupm11}
\end{eqnarray}
as claimed.

\paragraph{Remark ---} To pack the proof of Theorem \ref{evolmut-long}, we finally
need to check that the number of examples $m$ in ineq. (\ref{defmm}) is
indeed sufficient, which is upperbounded by the number of iterations
to satisfy evolvability requirements while staying in
$\hgmon$, \textit{i.e.} using $T$ as in ineq
. (\ref{numiterevol}). Considering the other parameters, $\tau$ in
eq. (\ref{defTau}), the fact that $\alpha$ is $O(1)$ in
eq. (\ref{defAlpha}), we obtain from ineq. (\ref{numevom22}) that it is sufficient to sample
\begin{eqnarray}
m & = & O\left(\frac{\hessianmax^6
        \expressmax^2 \sup_{x} \single^2(\mutbasis^*, x)}{\hessianmin^4 \mu^2_{\hg}} \cdot \frac{1}{\overline{B}^4_{\hg}(\mutbasis^*)} \cdot\frac{ \darwin^4 }{\epsilon^{2}}\log\left(\frac{\dF
      \darwin}{\epsilon}\right)\right) \nonumber\\
 & = & \tilde{O}\left(\frac{ \darwin^4 }{\epsilon^{2}}\log\left(\frac{\dF
      \darwin}{\epsilon}\right)\right)\label{bsupm22}\:\:,
\end{eqnarray}
as claimed. Notice that we have essentially hidden in the $\tilde{O}$
notation of ineqs (\ref{bsupm11}) and (\ref{bsupm22}) the eventual
(polynomial) dependences
of $\single(.,.)$ in $\darwin$.

\section{Proof of Theorem \ref{throbust}}\label{proofthrobust}

We consider the complete statement of Theorem \ref{evolmut-short},
that is, the one in Theorem Theorem \ref{evolmut-long}.
Let $\mutbasis^*$ a basis for $\mathrm{span}(\mutbasis)$. Suppose
$\mathrm{dim}(\mathrm{span}(\mutbasis^*))< \dG$ and denote
$\mutbasis^\bot$ any basis for the supplementary space of
$\mathrm{span}(\mutbasis^*)$ in $\hg$, and $\matrice{b}^\bot$ its
matrix in the orthonormal basis of $\hg$. Whenever $t\not\in
\mathrm{span}(\mutbasis^*)$, there uniquely exists two vectors
$t^{\mathrm{in}} \in \mathrm{span}(\mutbasis^*)$ and
$t^{\mathrm{out}} \in \mathrm{span}(\mutbasis^\bot)$ such that $t =
t^{\mathrm{in}} + t^{\mathrm{out}}$.

For the sake of readability, we represent $t^{\mathrm{in}}$ in
$\mutbasis^*$ and $t^{\mathrm{out}}$ in
$\mutbasis^\bot$, such that $\hg \ni t = \matrice{b}^*
t^{\mathrm{in}}+\matrice{b}^\bot t^{\mathrm{out}}$.
This being defined, for any $t_* \in \mathrm{span}(\mutbasis^*)$, we then have, for some $\matrice{m}(t,t^*)\succ 0$:
\begin{eqnarray}
\lefteqn{\expect_{x \sim {\mathcal{D}}} [D_\varphi(f(x) \|
  t(x))]}\nonumber\\
 & = &
\expect_{x \sim {\mathcal{D}}} [D_\varphi(f(x) \| t_*(x))] +
\expect_{x \sim {\mathcal{D}}} [D_\varphi(t_*(x) \| t(x))]  +
\expect_{x \sim {\mathcal{D}}} [\langle (t_* - f)(x), \nabla\varphi
\circ t(x) -
\nabla\varphi \circ t_* (x)\rangle]\nonumber\\
 & = & \expect_{x \sim {\mathcal{D}}} [D_\varphi(f(x) \| t_*(x))] +
\expect_{x \sim {\mathcal{D}}} [D_\varphi(t_*(x) \| t(x))] \nonumber\\
 & & + \langle \matrice{b}^*(t_* - f), \matrice{m}(t,t^*)
(\matrice{b}^\bot t^{\mathrm{out}} +\matrice{b}^* t^{\mathrm{in}} - \matrice{b}^* t_*)\rangle \label{eqreflem}\\
 & = & \expect_{x \sim {\mathcal{D}}} [D_\varphi(f(x) \| t_*(x))] +
\expect_{x \sim {\mathcal{D}}} [D_\varphi(t_*(x) \| t(x))]
\nonumber\\
& & + \underbrace{\langle t_* - f, {\matrice{b}^*}^\top \matrice{m}(t,t^*)
\matrice{b}^\bot t^{\mathrm{out}}\rangle}_{=0} + \langle
\matrice{b}^*(t_* - f), \matrice{m}(t,t^*)(\matrice{b}^*
t^{\mathrm{in}} - \matrice{b} t_*)\rangle\label{eqrefort}\\
& = & \expect_{x \sim {\mathcal{D}}} [D_\varphi(f(x) \| t_*(x))] +
\expect_{x \sim {\mathcal{D}}} [D_\varphi(t_*(x) \| t(x))] + \langle
t_* - f, {\matrice{b}^*}^\top \matrice{m}(t,t^*)\matrice{b}^*
(t^{\mathrm{in}} - t_*)\rangle\label{eqreflast}\:\:.
\end{eqnarray}
The first equality is the Bregman triangle equality
\citep{anMO}. Eq. (\ref{eqreflem}) comes from Lemma
\ref{lemmareturndec} (matrix $\matrice{m}(t,t^*)$ is defined as in
(\ref{defMM}); our notation puts in
emphasis the fact that the matrix depends on $t$ and $t^*$, but not on $f$). We indeed have $\langle t_* - f, {\matrice{b}^*}^\top \matrice{m}(t,t^*)
\matrice{b}^\bot t^{\mathrm{out}}\rangle = 0$ in eq. (\ref{eqrefort})
since we can always choose $\mutbasis^\bot$ such that:
\begin{eqnarray}
{\matrice{b}^*}^\top \matrice{m}(t,t^*)
\matrice{b}^\bot & = & \matrice{0}\:\:.\label{propbot}
\end{eqnarray}
 To see that this holds, remark that ${\matrice{b}^*}^\top \matrice{m}(t,t^*)
\matrice{b}^\bot = (\matrice{p}^\top \matrice{b}^*)^\top
(\matrice{p}^\top \matrice{b}^\bot)$ for some transfer matrix
$\matrice{p}$. Since it is a transfer matrix, $\mathrm{rank}(\matrice{p}^\top \matrice{b}^*) =\mathrm{rank}(\matrice{b}^*) =
\mathrm{dim}(\mathrm{span}(\mutbasis^*))$, so we can find $d\times
(\dG - \mathrm{dim}(\mathrm{span}(\mutbasis^*)))$ full rank matrix
$\tilde{\matrice{m}}$ such that $(\matrice{p}^\top \matrice{b}^*)^\top \tilde{\matrice{m}}
= \matrice{0}$, which allows to pick $\matrice{b}^\bot =
(\matrice{p}^\top)^{-1} \tilde{\matrice{m}}$, define accordingly
$\mathcal{B}^\bot$ with the column vectors, and therefore ensures
eq. (\ref{propbot}) satisfied. 

We need however to check that
$\mutbasis^*$ and $\mutbasis^\bot$ are indeed supplementary in $\hg$.
Suppose that some $h$
belongs to both spans of $\mutbasis^*$ and $\mutbasis^\bot$ as defined
here. Let $h^*$ and $h^\bot$ be its (unique) coordinates in both
sets, therefore satisfying $\matrice{b}^* h^* = \matrice{b}^\bot
h^\bot$. We obtain $\matrice{b}^* h^* = (\matrice{p}^\top)^{-1}
\tilde{\matrice{m}} h^\bot$, or equivalently $(\matrice{p}^\top
\matrice{b}^*)^\top \matrice{p}^\top \matrice{b}^*
h^* = (\matrice{p}^\top \matrice{b}^*)^\top \tilde{\matrice{m}} h^\bot =
0_{\hg}$, implying
\begin{eqnarray}
h^* & \in & \mathrm{ker}
((\matrice{p}^\top
\matrice{b}^*)^\top \matrice{p}^\top \matrice{b}^*)\nonumber\\
 & & = \mathrm{ker}
({\matrice{b}^*}^\top\matrice{m}(t,t^*){\matrice{b}^*})\:\:,\nonumber
\end{eqnarray} 
and
therefore $h^* = 0_{\mathcal{B}^*}$ (since
$\matrice{m}(t,t^*) \succ 0$ and $\matrice{b}^*$ has full rank), and so $h^\bot =
0_{\mathcal{B}^\bot}$, and finally $h = 0_{\hg}$, implying
$\mathrm{span}(\mutbasis^*) \cap \mathrm{span}(\mutbasis^\top) =
\{0_{\hg}\}$ and since
$\mathrm{dim}(\mathrm{span}(\mutbasis^*))+\mathrm{dim}(\mathrm{span}(\mutbasis^\bot))
= \dG$, $\mutbasis^*$ and $\mutbasis^\bot$ are supplementary in $\hg$, as claimed.

In eq. (\ref{eqreflast}), $\langle
t_* - f, {\matrice{b}^*}^\top \matrice{m}(t,t^*)\matrice{b}^*
(t^{\mathrm{in}} - t_*)\rangle$ zeroes over all $f$ iff
$t^{\mathrm{in}} = t_*$ since again $\mathrm{ker}
({\matrice{b}^*}^\top\matrice{m}(t,t^*){\matrice{b}^*}) =
\{0_{{\mathcal{B}}^*}\}$. So 
\begin{eqnarray}
\expect_{x \sim {\mathcal{D}}} [D_\varphi(f(x) \|
  t(x))] & = & \expect_{x \sim {\mathcal{D}}} [D_\varphi(f(x) \| t^{\mathrm{in}}(x))] +
\expect_{x \sim {\mathcal{D}}} [D_\varphi(t^{\mathrm{in}}(x) \|
t(x))], \forall f \in \hg\:\:,\label{deffright}
\end{eqnarray}
from which we get, from the non-negativity of Bregman divergences and (ii) in setting (\textbf{SF}),
\begin{eqnarray}
  t^{\mathrm{in}} & = & \arg  \sup_{f
  \in \mathrm{span}(\mutbasis)} \perf_{t,\varphi}(f,
{\mathcal{D}})\:\:.\nonumber
\end{eqnarray}
Then, since since the rightmost
expectation in eq. (\ref{deffright}) does not depend on $f$,
we can equivalently reformulate the definitions of $\bene$ and
$\neut$ in eqs (\ref{defbene}) and (\ref{defneut}) by:
\begin{eqnarray}
\bene(f) & = & \{g \in \neigh_\epsilon(f): \perf_{t^{\mathrm{in}},\varphi}(g,
{\mathcal{S}}) \geq \perf_{t^{\mathrm{in}},\varphi}(f,
{\mathcal{S}}) + \tol\}\:\:,\label{defbenerob}\\
\neut(f) & = & \{g \in \neigh_\epsilon(f): |\perf_{t^{\mathrm{in}},\varphi}(g,
{\mathcal{S}}) - \perf_{t^{\mathrm{in}},\varphi}(f,
{\mathcal{S}}) |\leq \tol\}\:\:.\label{defneutrob}
\end{eqnarray}
Replacing $t$ by $t^{\mathrm{in}}$ in $\darwin$, we get that Theorem
\ref{evolmut-long} can now be applied with all three conditions in
(\textbf{SF}) and guarantees this time from ineq. (\ref{condgen}):
\begin{eqnarray}
\perf_{t^{\mathrm{in}},\varphi}(f_{T},
{\mathcal{D}}) & \geq & - \epsilon\:\:,\label{condgenrobust}
\end{eqnarray}
and so, from eq. (\ref{deffright}),
\begin{eqnarray}
\perf_{t,\varphi}(f_{T},
{\mathcal{D}}) & = & \perf_{t^{\mathrm{in}},\varphi}(f_{T},
{\mathcal{D}}) + \sup_{f
  \in \mathrm{span}(\mutbasis)} \perf_{t,\varphi}(f,
{\mathcal{D}})\nonumber\\
 & \geq & \sup_{f
  \in \mathrm{span}(\mutbasis)} \perf_{t,\varphi}(f,
{\mathcal{D}}) - \epsilon \:\:,\nonumber
\end{eqnarray}
as claimed.

\section{Proof of Theorem
  \ref{thdrift-short}}\label{proofthdrift}

\paragraph{Complete statement of Theorem \ref{thdrift-short} ---}  We provide below the complete statement of Theorem
\ref{thdrift-short}, using notations from Section \ref{app-add-main}.
\begin{theorem}\label{thdrift-long}
Assume (\textbf{SF}) holds and the target organism sequence $t_0, t_1, ...$ drifts
according to:
\begin{eqnarray}
\|t_{i+1} - t_i\|_\hg & \leq & \frac{  z_\tol -
   z_\alpha z_\tau}{2U^2V(2 + \darwin^2 \max_{i'\in [\dF]_*}
   \|b_{i'}\|^2_{\hg})}\cdot \epsilon^{4}\:\:, \forall i\geq 0 \:\:.\label{constdriftXX-long}
\end{eqnarray}
Then 
Theorem \ref{evolmut-long} holds mutatis mutandis with the replacement of
ineq. (\ref{condgen}) by:
\begin{eqnarray}
\perf_{t_T,\varphi}(f_{T},
{\mathcal{D}}) & \geq & - \epsilon\:\:.\label{condgendrift-long}
\end{eqnarray}
\end{theorem}

\paragraph{Proof of Theorem \ref{thdrift-long}}

We provide here the proof of Theorem \ref{thdrift-long}.
First, we reformulate the definitions of $\bene$ and
$\neut$ in eqs (\ref{defbene}) and (\ref{defneut}) to fit to the model
\citep{kvwEW}:
\begin{eqnarray}
\bene(f_j) & = & \{g \in \neigh_\epsilon(f_j): \perf_{t_j,\varphi}(g,
{\mathcal{S}}) \geq \perf_{t_j,\varphi}(f_j,
{\mathcal{S}}) + \tol\}\:\:,\label{defbenedrift}\\
\neut(f_j) & = & \{g \in \neigh_\epsilon(f_j): |\perf_{t_j,\varphi}(g,
{\mathcal{S}}) - \perf_{t_j,\varphi}(f_j,
{\mathcal{S}}) |\leq \tol\}\:\:.\label{defneutdrift}
\end{eqnarray}
We also replace ${\hg}_{t, {\mathcal{D}}}$ by a sequence ${\hg}_{t_j,
  {\mathcal{D}}}$, so $\hgmon$ is now the prefix sequence of $f_0,
f_1, ...$ such that $f_j \in {\hg}_{t_j,
  {\mathcal{D}}}$. The definition of ${\hg}_{t_j,
  {\mathcal{D}}}$ is the same as in (\ref{defftd}).

The proof consists of three steps: first (part (i)), we show that Lemma
\ref{lemmpack} still holds when $t$ is allowed to drift following
ineq. (\ref{constdriftXX-long}). Then (part (ii)), we show that the bound
in $m$ is the same as in (\ref{numevom22}). Finally (part (iii)), we show that Lemma
\ref{lemconv1} also holds, completing the proof.

\noindent\textbf{Part (i) ---}To prove the first part, let $\beta(i)$ denote the index of
the $b_.\in \mutbasis$ chosen by mutator at step $i$. 
We reuse ineq. (\ref{defincre}), and get this time
\begin{eqnarray}
 \perf_{t_{i+1},\varphi}(f_{i+1},
{\mathcal{D}}) & \defeq & \perf_{t_{i+1},\varphi}(f_i+\sigma \alpha b_{\beta(i)},
{\mathcal{D}})\nonumber\\
 & \geq & \perf_{t_{i+1},\varphi}(f_i,
{\mathcal{D}}) + \alpha\left(\frac{\tol}{\alpha} - \tau\right) \nonumber\\
 &  & = \perf_{t_{i+1},\varphi}(f_i,
{\mathcal{D}}) + \frac{z_\tol - z_\alpha z_\tau}{U^2V}\cdot \epsilon^{2} \:\:.\label{defincredrift}
\end{eqnarray}
Now, we use the Bregman triangle equality \citep{anMO}, which yields
\begin{eqnarray}
\lefteqn{\expect_{x \sim {\mathcal{D}}} [D_\varphi(f_i (x) \|
  t_{i}(x))]}\nonumber\\
  & = & \expect_{x \sim {\mathcal{D}}} [D_\varphi(f_i (x) \|
  t_{i+1}(x))] + \expect_{x \sim {\mathcal{D}}} [D_\varphi(t_{i+1} (x) \|
  t_{i}(x))] \nonumber\\
 & &+ \expect_{x \sim {\mathcal{D}}} [\langle (t_{i+1} - f_i)(x), \nabla\varphi
\circ t_i(x) -
\nabla\varphi \circ t_{i+1} (x)\rangle]\:\:,\nonumber
\end{eqnarray}
and so, reorganizing,
\begin{eqnarray}
\perf_{t_{i+1},\varphi}(f_i,
{\mathcal{D}}) & = & \perf_{t_{i},\varphi}(f_i,
{\mathcal{D}}) + \expect_{x \sim {\mathcal{D}}} [D_\varphi(t_{i+1} (x) \|
  t_{i}(x))] \nonumber\\
 & & + \expect_{x \sim {\mathcal{D}}} [\langle (t_{i+1} - f_i)(x), \nabla\varphi
\circ t_i(x) -
\nabla\varphi \circ t_{i+1} (x)\rangle]\:\:.\label{tog1}
\end{eqnarray}
Lemma \ref{lemmareturndec} yields, for some symmetric positive
definite $\matrice{m}$ defined in the same way as in (\ref{defMM}),
\begin{eqnarray}
\lefteqn{\expect_{x \sim {\mathcal{D}}} [\langle (t_{i+1} - f_i)(x), \nabla\varphi
\circ t_i(x) -
\nabla\varphi \circ t_{i+1} (x)\rangle]}\nonumber\\
 & = & \langle
\matrice{b}^*(t_{i+1} - f_i), \matrice{m}
\matrice{b}^* (t_i - t_{i+1})\rangle\nonumber\\
 & = & -\hessianmax|\langle
\matrice{b}^*(t_{i+1} - f_i), \gendistnone
\matrice{b}^* (t_i - t_{i+1})\rangle|\label{bsupii}\\
 & \geq & -\hessianmax \expressmax \|t_{i+1} - f_i\|_\hg \|t_i -
 t_{i+1}\|_\hg\label{bsupscal}\\
 & \geq & -\hessianmax \expressmax \left(\|t_0 - f_i\|_\hg + \sum_{k=1}^{i+1} \|t_{k+1} -
 t_{k}\|_\hg \right)  \cdot \|t_{i+1} -
 t_{i}\|_\hg\label{bsupscal2}\\
 &  & \defeq  -\hessianmax \expressmax \left(\|t_0 - f_i\|_\hg +
   \sum_{k=1}^{i+1} \nu_k \right)  \cdot \nu_{i+1}\:\:,\label{bsupscal4}
\end{eqnarray}
where we denote $\nu_k\defeq \|t_{k+1} -
 t_{k}\|_\hg $.
Ineqs (\ref{bsupii}) and (\ref{bsupscal}) hold because of the
definition and properties of $\gendistnone$ and Cauchy-Schwartz inequality. Inequality
(\ref{bsupscal2}) holds because of the triangle inequality. 
Now, we also have $\expect_{x \sim {\mathcal{D}}} [D_\varphi(t_{i+1} (x) \|
  t_{i}(x))]  = -\perf_{t_i,\varphi}(t_{i+1}, {\mathcal{D}}) \geq  (1/2) \hessianmin \expressmin
 \cdot \nu_{i+1}$ because of Lemma \ref{lemfund1}, so if we fold
 ineq (\ref{bsupscal4}) and
eq. (\ref{tog1}) into ineq. (\ref{defincredrift}), then we get,
$\forall f_i \in \hgmon$,
\begin{eqnarray}
\lefteqn{\perf_{t_{i+1},\varphi}(f_{i+1},
{\mathcal{D}})}\nonumber\\
 & \geq & \perf_{t_{i},\varphi}(f_i,
{\mathcal{D}}) + \frac{\hessianmin \expressmin}{2}\cdot \nu_{i+1} -\hessianmax \expressmax \left(\|t_0 - f_i\|_\hg +
   \sum_{k=1}^{i+1} \nu_k \right)  \cdot \nu_{i+1} + \frac{z_\tol -
   z_\alpha z_\tau}{U^2V}\cdot \epsilon^{2}\nonumber\\
 & = & \perf_{t_{i},\varphi}(f_i,
{\mathcal{D}}) - \frac{\nu_{i+1}}{2}\cdot \left(\hessianmax \expressmax \left(\|t_0 - f_i\|_\hg +
   \sum_{k=1}^{i+1} \nu_k \right) - \hessianmin
  \expressmin \right) + \frac{z_\tol -
   z_\alpha z_\tau}{U^2V}\cdot \epsilon^{2}\nonumber\\
 & \geq & \perf_{t_{i},\varphi}(f_i,
{\mathcal{D}}) - \frac{\nu_{i+1}}{2}\cdot \left(\hessianmax \expressmax \left(\sqrt{\frac{\hessianmax \expressmax}{\hessianmin \expressmin}}\cdot
\darwin \max_{i\in [\dF]_*}
   \|b_{i}\|_{\hg}+
   \sum_{k=1}^{i+1} \nu_k \right) - \hessianmin
  \expressmin \right) \nonumber\\
 & &  + \frac{z_\tol -
   z_\alpha z_\tau}{U^2V}\cdot \epsilon^{2}\label{ineqdiff}\\
 & \geq &  \perf_{t_{i},\varphi}(f_i,
{\mathcal{D}}) - \frac{\nu_{i+1}}{2}\cdot \left(\hessianmax
  \expressmax \left(\underbrace{1 + \max \left\{1, \frac{2\hessianmax \expressmax}{\hessianmin \expressmin}\cdot
\darwin^2 \max_{i\in [\dF]_*}
   \|b_{i}\|^2_{\hg}\right\}}_{\defeq W} +
   \sum_{k=1}^{i+1} \nu_k \right) - \hessianmin
  \expressmin \right) \nonumber\\
 & &  + \frac{z_\tol -
   z_\alpha z_\tau}{U^2V}\cdot \epsilon^{2}\label{nottight}\:\:,
\end{eqnarray}
where ineq. (\ref{ineqdiff}) follows from ineq. (\ref{seqdec}).
Now, remark that $\hessianmax \expressmax W - \hessianmin
\expressmin > 0$, and suppose we 
ensure that, if  $f_i \in \hgmon$, $\forall i\in
[T]_*$, then
\begin{eqnarray}
\nu_{i+1} & \leq & \eta \cdot \frac{z_\tol -
   z_\alpha z_\tau}{U^2V(\hessianmax \expressmax W - \hessianmin \expressmin)\sqrt{T}}\cdot \epsilon^{2} \:\:,\label{constnu}
\end{eqnarray}
for some $\eta > 0$. In this case, assuming that the outermost
parenthesis is non negative, and letting $X \defeq (z_\tol -
   z_\alpha z_\tau)/(U^2V)$, we can assert
\begin{eqnarray}
\perf_{t_{i+1},\varphi}(f_{i+1},
{\mathcal{D}}) & \geq & \perf_{t_{i},\varphi}(f_i,
{\mathcal{D}}) - \left(\eta\cdot \frac{X}{2 \sqrt{T}}\cdot \epsilon^{2} +
\eta^2 \cdot \frac{\hessianmax \expressmax X^2}{2 (\hessianmax \expressmax W - \hessianmin
  \expressmin)}\cdot \epsilon^{4}\right)  + X\cdot \epsilon^{2}\label{assert1}\:\:.
\end{eqnarray}
Now, we want to fing $\eta$ such that:
\begin{eqnarray}
\eta\cdot \frac{X}{2 \sqrt{T}}\cdot \epsilon^{2} +
\eta^2 \cdot \frac{\hessianmax \expressmax X^2}{2 (\hessianmax \expressmax W - \hessianmin
   \expressmin)}\cdot \epsilon^{4} & \leq & \frac{X}{2} \cdot \epsilon^{2}\nonumber\:\:.
\end{eqnarray}
Reorganizing, we find that it is sufficient that
\begin{eqnarray}
\eta^2 \cdot \frac{\hessianmax \expressmax X}{\hessianmax \expressmax W - \hessianmin
   \expressmin}\cdot \epsilon^{2} + \frac{\eta}{\sqrt{T}} - 1 & \leq & 0 \nonumber\:\:,
\end{eqnarray}
and so we need
\begin{eqnarray}
\eta & \leq & \frac{(\hessianmax \expressmax W - \hessianmin
   \expressmin)}{2 \hessianmax \expressmax X\sqrt{T}}\cdot \left( \sqrt{1 + \frac{4T\hessianmax \expressmax X}{\hessianmax \expressmax W - \hessianmin
   \expressmin}\cdot \epsilon^{2}} - 1\right)\nonumber\:\:.
\end{eqnarray}
Since $\sqrt{1+x} \geq 1 + (x/2) - x^2/8\geq 0$ for $x\in [0, 2(1+\sqrt{3})]$, it is sufficient that
\begin{eqnarray}
\eta & \leq & \sqrt{T} \cdot \epsilon^{2} \cdot \left( 1 - \frac{T \hessianmax \expressmax X}{\hessianmax \expressmax W - \hessianmin
   \expressmin}\cdot \epsilon^{2}\right)\:\:,
\end{eqnarray}
and we want
\begin{eqnarray}
T & < & \frac{\hessianmax \expressmax W - \hessianmin
   \expressmin}{\hessianmax \expressmax X} \cdot \frac{1}{2\epsilon^2}\:\:.
\end{eqnarray}
In this case, we can check that $4T\hessianmax \expressmax X/(\hessianmax \expressmax W - \hessianmin
   \expressmin) < 2 < 2(1+\sqrt{3})$. Replacing $X$ and $W$ by their expressions, we want equivalently
\begin{eqnarray}
T & \leq & \frac{U^2V \hessianmax \expressmax}{z_\tol -
   z_\alpha z_\tau} \cdot \left(1 + \max \left\{1, \frac{2\hessianmax \expressmax}{\hessianmin \expressmin}\cdot
\darwin^2 \max_{i\in [\dF]_*}
   \|b_{i}\|^2_{\hg}\right\}- \frac{\hessianmin
   \expressmin}{\hessianmax \expressmax}\right) \cdot \frac{1}{2\epsilon^2}\:\:.\label{lastconst}
\end{eqnarray}
Now, we have 
\begin{eqnarray}
1 + \max\{1, 2xy\} - \frac{1}{x} & \geq & 2y\:\:, \forall x\geq 1,
\forall y\geq 0\:\:,
\end{eqnarray}
and so to ensure ineq. (\ref{lastconst}), it is sufficient to ensure
\begin{eqnarray}
T & \leq & \frac{U^2V \hessianmax \expressmax \max_{i\in [\dF]_*}
   \|b_{i}\|^2_{\hg}}{z_\tol -
   z_\alpha z_\tau} \cdot \frac{\darwin^2}{\epsilon^2}\:\:.\label{lastconst2}
\end{eqnarray}
In this case, we
can fix
\begin{eqnarray}
\eta & = & \frac{\sqrt{T}}{2}\cdot \epsilon^2\:\:,
\end{eqnarray}
which replaces ineq. (\ref{constnu}) by
\begin{eqnarray}
\nu_{i+1} & \leq & \frac{z_\tol -
   z_\alpha z_\tau}{2 U^2V(\hessianmax \expressmax W - \hessianmin \expressmin)}\cdot \epsilon^{4} \:\:,\nonumber
\end{eqnarray}
which holds if
\begin{eqnarray}
\nu_{i+1} & \leq & \frac{  z_\tol -
   z_\alpha z_\tau}{2U^2V(2 + \darwin^2 \max_{i\in [\dF]_*}
   \|b_{i}\|^2_{\hg})}\cdot \epsilon^{4} \:\:.\label{constdrift1}
\end{eqnarray}
In this case, ineq. (\ref{assert1}) becomes
\begin{eqnarray}
\perf_{t_{i+1},\varphi}(f_{i+1},
{\mathcal{D}}) & \geq & \perf_{t_{i},\varphi}(f_i,
{\mathcal{D}}) + \frac{z_\tol -
   z_\alpha z_\tau}{2U^2V} \cdot \epsilon^{2}\label{assert2}\:\:.
\end{eqnarray}
So if the drift is bounded as in ineq. (\ref{constdrift1}), then we lose by a
factor at most 2 over the improvement without drift as guaranteed in $\hgmon$
by Lemma \ref{lemmpack}. We then need to check that the number of
iterations in ineq. (\ref{consttt1}) now becomes
\begin{eqnarray}
T & \geq & \frac{U^2V \hessianmax \expressmax \max_{i\in [\dG]_*}
   \|b_{i}\|^2_{\hg}}{z_\tol - z_\alpha z_\tau} \cdot \frac{\darwin^2}{\epsilon^{2}}\:\:,\label{consttt2}
\end{eqnarray}
whose right-hand side matches ineq. (\ref{lastconst2}). Since
ineq. (\ref{nottight}) is never tight, picking $T$ of the order of the
right-hand side of ineq. (\ref{consttt2}) allows for $\hgmon$ to
comply with ineq. (\ref{condgendrift-long}) with a number of steps of the
same order as for $T$ in Lemma
\ref{lemmpack}. Hence, Lemma \ref{lemmpack} still holds.\\

\noindent\textbf{Part (ii) ---} Notice that as long as $f_j \in
\hgmon$, we can still bound $\|t_j-f_j\|^2_{\hg}$ in the same way as
we do in ineq. (\ref{seqdec}), because ineq. (\ref{ingmon}) can still
be used via the fact that $\perf_{t_{i+1},\varphi}(f_{i+1},
{\mathcal{D}}) \geq \perf_{t_{i},\varphi}(f_i,
{\mathcal{D}})$, as shown by ineq. (\ref{assert2}), so the order of
the number of
examples in ineq. (\ref{numevom22}) does not change because the order
of $T$ does not change. \\

\noindent\textbf{Part (iii) ---} Because of the definition of
${\hg}_{t_j, {\mathcal{D}}}$, if we let $f_{j_\star}$ to denote the
first organism out of the prefix sequence $\hgmon$ ($f_{j_\star-1}$ is
in $\hgmon$ but $f_{j_\star}$ is not), and $t_{j_\star}$ the target at the
same index in the target sequence, then ineq. (\ref{eqfjstar})
still holds with $t_{j_\star}$, and so we shall observe again, in place of
ineq. (\ref{boundzzp})
\begin{eqnarray}
\expect_{x \sim {\mathcal{D}}} [D_\varphi(f_{j_\star}(x) \| t_{j_\star}(x))] & \leq & {\hessianmax}\|t_{j_\star}-f_{j_\star}\|_{\hg} \cdot
\rho(f_{j_\star}, t_{j_\star}|{{{\mathcal{D}}}})\nonumber\\
 & \leq & \epsilon\:\:,\label{boundzzpdrift}
\end{eqnarray}
which means
\begin{eqnarray}
\perf_{t_{j_\star},\varphi}(f_{j_\star}, {\mathcal{D}}) & \geq & - \epsilon\:\:,\nonumber
\end{eqnarray}
and so $f_{j_\star}$ satisfies ineq. (\ref{condgendrift-long}), as claimed. 

\section{Proof of Theorem  \ref{thmtrap-short}}\label{proofthmtrap}

The trick we use is the following one: theorem \ref{evolmut-long} relies on the existence of a
monotonic sequence (with respect to performances) which leads to satisfying the conditions of
evolution. Provided we constrain a bit more set
$\mathfrak{R}$, we can do more than the requirements of evolvability:
when we escape this monotonic sequence, the mutated
organism is going to \textit{stay} within the evolvability
requirements, over a number of iterations / mutations steps that we
can control.

\begin{definition}\label{defSER}
Fix $N\in {\mathbb{N}}_*$. Set
$\mathfrak{R}_N\subset {\mathbb{R}}^3$ is the subset of 
triples $(z_1,z_2,z_3)$ such that (i) $z_1,z_2,z_3 > 0$, (ii)
$z_3 - z_1z_2 > 0$, and (iii) $(a+b)z_2^2 - z_2(1-bz_1) +bz_3 \leq 0$, where $a
\defeq N/U$ and $b\defeq 2\hessianmax/\hessianmin$.
\end{definition}
It is worthwhile remarking that $\mathfrak{R}_N \subset \mathfrak{R}$,
and furthermore $\mathfrak{R}_N\neq \emptyset$ since we can choose for example:
\begin{eqnarray}
z_1 = \frac{1}{4b}\:\:, z_2 = \frac{1}{32(a+b)}\:\:, z_3=
\frac{1}{64b(a+b)}\:\:.
\end{eqnarray}

\begin{theorem}\label{thmtrap-long}
Assume (\textbf{SF}) holds, and $(z_\tau, z_\alpha, z_\tol)\in
\mathfrak{R}_N$ for some $N\in {\mathbb{N}}_*$, and all other
parameters are fixed according to Theorem \ref{th00}. Let $f_\star$ be the first organism in the sequence
$f_0, f_1, ..., f_T$ to hit $\hg_t$. Then, with probability $\geq 1 -
\epsilon$, $f_{j_\star + j}
\in \hg_t, \forall j \in [N]$.
\end{theorem}
\begin{proof}
\begin{figure}[t]
\begin{center}
\includegraphics[trim=10bp 200bp 340bp 200bp,clip,width=0.8\columnwidth]{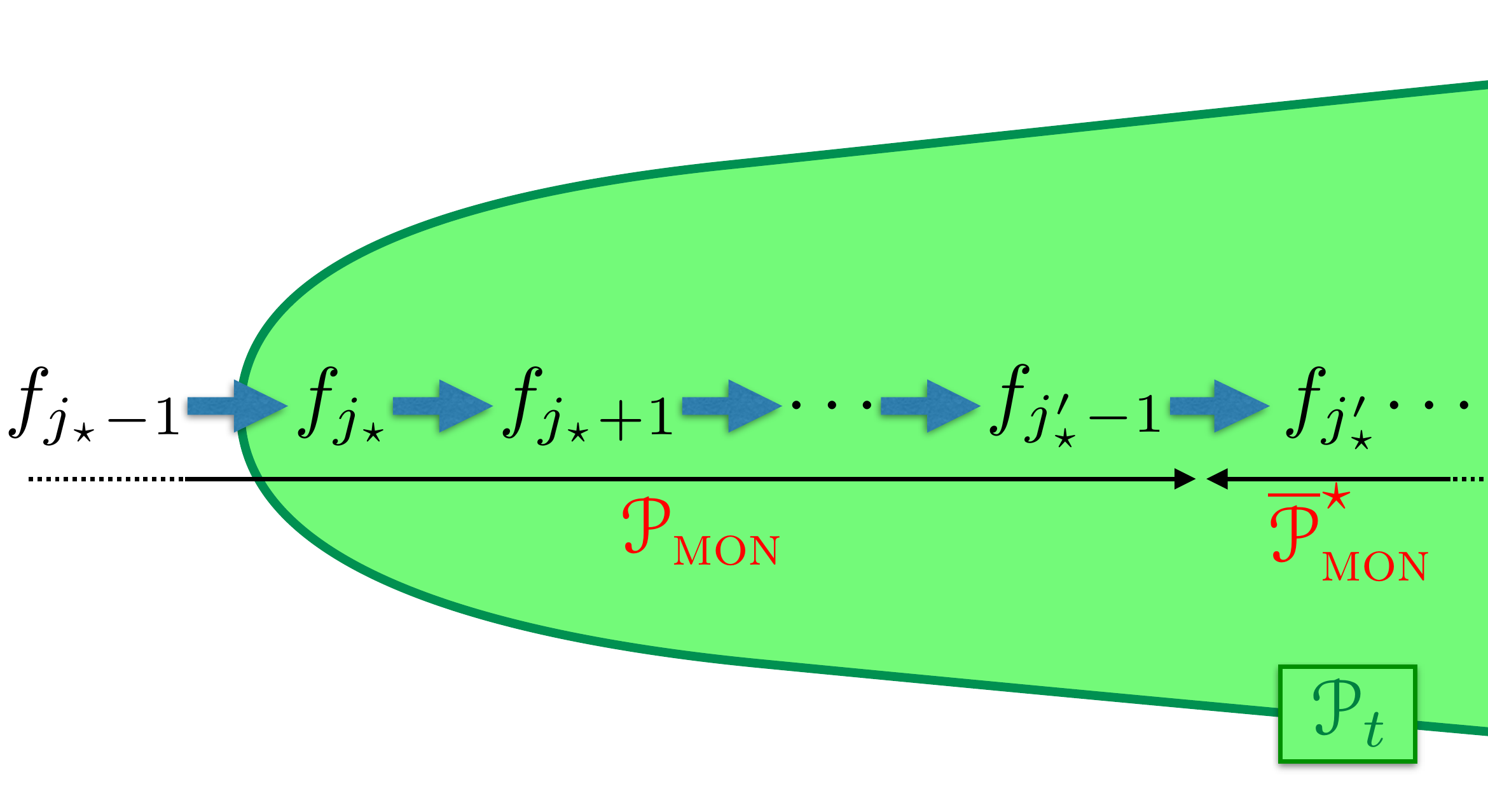}
\end{center}
\caption{Proof of Theorem \ref{thmtrap-long}: we show that the sequence
  $\hgnotmon^\star$ satisfies
  $\mathrm{Card}(\hgnotmon^\star \cap \hg_t)\geq N$ (See text for details).}
\label{f-trap}
\end{figure}
Let $f_{j_\star} \in \hg_t$ be the \textit{first} organism in the sequence $f_0,
f_1, ..., f_T$ to hit $\hg_t$. Because of Theorem \ref{th00} and Lemma \ref{lemconv1}, with
high probability, all organisms before $f_{j_\star}$ belong to
$\hgmon$. 
We have two cases, either $f_{j_\star}
\in \hgmon$ or $f_{j_\star}
\in \hgnotmon$. We know already that with high probability, as long as
the mutated organism $f_j$ stays in $\hgmon$, its performance cannot
decrease; therefore, if $f_{j_\star}
\in \hgmon$, then all subsequent mutated organism in
$\hgmon$ also belong to $\hg_t$. What will be sufficient to show
Theorem \ref{thmtrap-long} will be to show that the sequence $\hgnotmon$,
clamped to its first element not before $f_{j_\star}$ (we call it $\hgnotmon^\star$), satisfies
\begin{eqnarray}
\mathrm{Card}(\hgnotmon^\star \cap \hg_t) & \geq & N\:\:.
\end{eqnarray}
Let $f_{j'_\star}$ denote the first element of $\hgnotmon^\star$, with therefore $j'_\star \geq j_\star$ and $f_{j'_\star
- 1}\in \hgmon$.
Let us define $\delta \defeq t - f_{j_\star}$, $\delta'' \defeq \sum_{k=1}^{N}\sigma_k \alpha b_{\beta(k)}$
 and $\delta' \defeq \delta -
\delta''$, $\sigma_k \in \{-1,1\}$
and $\beta : [N]\rightarrow [\dF]$ gives the mutations chosen to
evolve further $f_{j'_\star}$ for $N$ steps. All coordinates for $\delta,
\delta', \delta'', b_{\beta(k)}$ are expressed in basis
$\mutbasis^*$. 
Using eq. (\ref{eqbregm1}) in Lemma
\ref{lemapproxbregdiv} and the expression of $\matrice{m}\succ 0$ in
(\ref{defMM}), and Lemma \ref{lemapproxbregdiv}, 
we get, for any $N' \in [N]$,
\begin{eqnarray}
\expect_{x \sim {\mathcal{D}}}
  \left[D_\varphi\left(\left(\left. f_{j'_\star}+\sum_{k\in [N']} \sigma_k\alpha
  b_{\beta(k)}\right)(x) \right\| t(x)\right)\right] & = & \frac{1}{2} \cdot \langle \matrice{b}^* \delta', \matrice{m}
 \matrice{b}^* \delta'\rangle\nonumber\\
 & \leq & \frac{\hessianmax}{2}  \cdot  \langle \matrice{b}^* \delta', \gendistnone
 \matrice{b}^* \delta'\rangle\label{Neq2}\:\:.
\end{eqnarray}
We then observe
\begin{eqnarray}
 \frac{\hessianmax}{2} \cdot  \langle \matrice{b}^* \delta', \gendistnone
 \matrice{b}^* \delta'\rangle & \leq & \hessianmax \cdot  \langle \matrice{b}^* \delta, \gendistnone
 \matrice{b}^* \delta\rangle + \alpha \hessianmax \cdot \langle \matrice{b}^* \delta'', \gendistnone
 \matrice{b}^* \delta''\rangle\label{Neq3}\\
& \leq &\hessianmax \cdot  \langle \matrice{b}^* \delta, \gendistnone
 \matrice{b}^* \delta\rangle + N' \alpha \hessianmax \cdot \sum_{k\in
 [N']}\langle \matrice{b}^* b_{\beta(k)}, \gendistnone
 \matrice{b}^* b_{\beta(k)}\rangle\label{Neq4}\\
&  & = \hessianmax \cdot  \langle \matrice{b}^* \delta, \gendistnone
 \matrice{b}^* \delta\rangle + N' \alpha \hessianmax \cdot \sum_{k\in [N']}\expect_{x \sim {\mathcal{D}}} [\|b_{\beta(k)}(x)\|_2^2]\nonumber\\
& \leq & \frac{2\hessianmax}{\hessianmin} \cdot \expect_{x \sim {\mathcal{D}}} [D_\varphi(f_{j'_\star}(x) \| t(x))] + N'\alpha \hessianmax \cdot
 \sum_{k\in [N']}\expect_{x \sim {\mathcal{D}}} [\|b_{\beta(k)}(x)\|_2^2]\label{eqbd2}\\
 & \leq &\frac{2 \hessianmax}{\hessianmin} \cdot \expect_{x \sim
   {\mathcal{D}}} [D_\varphi(f_{j'_\star}(x) \| t(x))] + N' \alpha\cdot  \max\{1, V\} \label{bsupV1}\\
 & & = \frac{2\hessianmax}{\hessianmin} \cdot \expect_{x \sim {\mathcal{D}}} [D_\varphi(f_{j'_\star}(x) \| t(x))] + \frac{N'z_\alpha}{U}\epsilon \label{bsupAlpha1}\\
 & \leq &\frac{2\hessianmax}{\hessianmin}\cdot  \left(
   z_\tau + z_\alpha + \frac{z_\tol}{z_\alpha}\right) \cdot \epsilon +
 \frac{N'z_\alpha}{U}\epsilon \label{defgstar}\\
 & & = \left(\frac{2\hessianmax}{\hessianmin}\cdot\left(
   z_\tau + z_\alpha + \frac{z_\tol}{z_\alpha}\right) +
 \frac{N'z_\alpha}{U}\right) \cdot \epsilon \nonumber\\
 & \leq &\left(\frac{2\hessianmax}{\hessianmin}\cdot\left(
   z_\tau + z_\alpha + \frac{z_\tol}{z_\alpha}\right) +
 \frac{Nz_\alpha}{U}\right) \cdot \epsilon\label{deffact}\\
 & \leq & \epsilon\:\:.\label{deflasteqEPS}
\end{eqnarray}
Ineqs (\ref{Neq3}) and (\ref{Neq4}) hold because for any inner
product $\langle .,.\rangle$ and set $\{u_1, u_2, ..., u_M\}$,
\begin{eqnarray}
\left\langle \sum_{k\in [M]} u_k, \sum_{k\in [M]} u_k\right\rangle &
\leq & M \cdot \sum_{k\in [M]} \langle u_k, u_k\rangle\:\:,
\end{eqnarray}
since right hand side minus left hand side is $\sum_{k\neq k'} \langle u_k - u_{k'}, u_k -
u_{k'}\rangle \geq 0$. Ineq. (\ref{eqbd2}) comes from Lemma
\ref{lemapproxbregdiv}.  Ineq. (\ref{bsupV1}) holds because of the definition of $V$ in
eq. (\ref{defV}). Eq. (\ref{bsupAlpha1}) holds because of the definition of $\alpha$ in
eq. (\ref{defAlpha}). Finally, ineq. (\ref{defgstar}) holds because we
know from ineq. (\ref{condconv}) that since $f_{j'_\star} \not\in \hgmon$,
then $f_{j'_\star} \not\in
{\hg}_{t, {\mathcal{D}}}$, and so
\begin{eqnarray}
\expect_{x \sim {\mathcal{D}}} [D_\varphi(f_{j'_\star}(x) \| t(x))] & \leq & \left( z_\tau + z_\alpha + \frac{z_\tol}{z_\alpha}\right) \cdot
 \epsilon\label{condconv2}\:\:,
\end{eqnarray}
and ineq. (\ref{deflasteqEPS}) holds because $(z_\tau, z_\alpha, z_\tol)\in
\mathfrak{R}_N$ (Definition \ref{defSER}). Indeed, constraint (iii)
yields equivalently
\begin{eqnarray}
\frac{2\hessianmax}{\hessianmin} \cdot \left(z_\alpha^2 + z_\alpha z_\tau +
  z_\tol\right) + \frac{Nz^2_\alpha}{U} & \leq & z_\alpha\nonumber\:\:,
\end{eqnarray}
which, after dividing by $z_\alpha > 0$ yields that the factor in
front of $\epsilon$ in eq. (\ref{deffact}) is $\leq 1$.

Hence, the moment the sequence $f_0, f_1, ...$ hits $\hg_t$, it shall
stay inside $\hg_t$ with high probability for at least $N$ further evolution
steps. 
This ends the proof of Theorem \ref{thmtrap-long}.
\end{proof}

\section{Evolution on the efficient frontier}\label{app-add-res}

The fact that beneficial mutations involve a mean-divergence
decomposition of the expression asks for the nature of the efficient
frontier, that is, the set of mutations that would minimize
$\expect_{\mathcal{D}} [\Pi_{f,i}(x)]$ subject to a fixed
$\expect_{\mathcal{D}} [\mathrm{R}_{f, i} (x)]$. This 
would give ``nature's best bet'' against our, perhaps, modest and
small ${\mathcal{B}}$.
To answer this
question, we consider the restricted case of Mahalanobis divergence. 
Let us alleviate the constraint that $b$ involves only elements from
${\mutbasis}$, and, for notational convenience, define $b(x) \defeq
\sum_{i=1}^{\dG} b^i g_i(x)$. In this case, finding the efficient frontier is solving, similarly to
that of the efficient
portfolio \citep{mAA},
\begin{eqnarray}
\arg\min_{b\in {\mathbb{R}}^{\dG}} \Pi(b) & \defeq & \frac{\alpha}{2} \cdot
\expect_{\mathcal{D}}\left[\langle b(x), \matrice{m}
  b(x)\rangle\right]\label{defoptf}\\
\mbox{s.c.} & & \left\{
\begin{array}{rcl}
\expect_{\mathcal{D}}\left[\langle b(x), \matrice{m}
  (t-f)(x)\rangle\right] & = & r\\
\langle 1_{\dG}, b\rangle & = & n
\end{array}\:\:.
\right.
\end{eqnarray}
We let the ``efficient
frontier'' denote the equation that gives $r$ as a function of
$\Pi(.)$, and for that purpose generalize $\gendistnone$ to $\gendist{\textsc{m}}\defeq \expect_{x \sim {\mathcal{D}}} [\matrice{g}^\top_x \matrice{m}
\matrice{g}_x]$, for any $\matrice{m} \in
{\mathbb{R}}^{d\times d}$. 
\begin{theorem}\label{theffront}
Under setting (\textbf{SF}), the equation of the efficient frontier is
\begin{eqnarray}
r & = & 
\left\{\begin{array}{lcl}
\frac{n\cdot \langle 1_{\dG}, t-f\rangle}{\langle 1_{\dG},
  \distortion^{-1} 1_{\dG} \rangle}\cdot \left(1 \pm \sqrt{(\xi(t-f)  - 1)(\xi(b) - 1)}\right) & \mbox{
 if } & \langle 1_{\dG}, t-f\rangle \neq 0\\
\pm n\cdot \sqrt{\langle 1_{\dG}, \distortion^{-1} 1_{\dG} \rangle \cdot \langle t-f, \distortion
(t-f)\rangle\cdot
  \left(\xi(b) - 1\right)} & \mbox{
 if } & \langle 1_{\dG}, t-f\rangle = 0
\end{array}\right. \label{defEFth}
\end{eqnarray}
where $\xi(u) \defeq \langle 1_{\dG}, \distortion^{-1} 1_{\dG} \rangle
  \cdot \langle u, \distortion u\rangle/(\langle 1_{\dG}, u\rangle)^2\geq 1$
  (defined for $\langle 1_{\dG}, u\rangle\neq 0$).
\end{theorem}
\begin{proof}
Let $\delta \defeq t
- f$ for short. We solve
\begin{eqnarray}
\min \Pi(b) & \defeq & \frac{\alpha}{2} \cdot\langle b, \distortion b\rangle\nonumber\\
\mbox{s.c.} & & \left\{
\begin{array}{rcl}
\langle b, \distortion \delta\rangle & = & r\\
\langle 1_{\dG}, b\rangle & = & n
\end{array}\:\:,
\right.
\end{eqnarray}
Letting $\lambda_r$ and $\lambda_n$ the two Lagrange multipliers for
the two constraints, we obtain the first order condition $\alpha \distortion b -
\lambda_r \distortion \delta - \alpha \lambda_n 1 = 0$, \textit{i.e.},
\begin{eqnarray}
b & = & \frac{\lambda_r }{\alpha} \cdot \delta + \frac{\lambda_n }{\alpha}  \cdot
\distortion^{-1} 1_{\dG}\label{eqbfo}\:\:,
\end{eqnarray}
from which we obtain, using the constraints, the following system:
\begin{eqnarray}
\left\{
\begin{array}{rcl}
\alpha r & = & \lambda_r\cdot \langle \delta,
\distortion\delta\rangle + \lambda_n \cdot \langle 1_{\dG}, \delta\rangle\\
\alpha n & = & \lambda_r \cdot \langle 1_{\dG}, \delta\rangle +\lambda_n \cdot \langle 1_{\dG}, \distortion^{-1}
1_{\dG}\rangle
\end{array}
\right. \:\:.\nonumber
\end{eqnarray}
\begin{lemma}\label{simplelemMM}
For any symmetric
positive definite $\matrice{m}$, $\gendist{\textsc{m}}\succ 0$.
\end{lemma}
The lemma is a direct consequence of assumption (iii) in
(\textbf{SF}). Lemma \ref{simplelemMM} yields that $\distortion$ is invertible. 
To simplify this system, let us denote for short:
\begin{eqnarray}
P_b & \defeq & \frac{\alpha}{2} \cdot\langle b, \distortion b\rangle\nonumber\:\:,\\
P_\delta & \defeq & \frac{\alpha}{2} \cdot\langle \delta, \distortion
\delta\rangle\nonumber\:\:,\\
S & \defeq & \langle 1_{\dG}, \distortion^{-1} 1_{\dG} \rangle\nonumber\:\:,\\
\Delta & \defeq & \langle 1_{\dG}, \delta\rangle\nonumber\:\:.
\end{eqnarray}
We obtain the simplified system:
\begin{eqnarray}
\left\{
\begin{array}{rcl}
\alpha r & = & \frac{2P_\delta}{\alpha} \cdot \lambda_r  + \Delta \cdot \lambda_n \\
\alpha n & = & \Delta \cdot \lambda_r + S\cdot 
\lambda_n
\end{array}
\right. \:\:,\label{eqsystemEF2}
\end{eqnarray}
admitting the solution
\begin{eqnarray}
\lambda_r & = & \frac{\alpha^2 S r - \alpha^2 \Delta n}{2SP_\delta - \alpha\Delta^2}\:\:,\label{eqlambdar}\\
\lambda_n & = & \frac{2\alpha P_\delta n - \alpha^2 \Delta r}{2
  SP_\delta - \alpha \Delta^2}\:\:.\label{eqlambdan}
\end{eqnarray}
We note that $2SP_\delta - \alpha\Delta^2 = \alpha(\langle \delta, \distortion
\delta\rangle \langle 1_{\dG}, \distortion^{-1} 1_{\dG} \rangle - (\langle
1_{\dG}, \delta\rangle)^2)\geq 0$ from Cauchy-Schwartz inequality.
Suppose that $2SP_\delta -
   \alpha\Delta^2 > 0$. Multiplying eq. (\ref{eqbfo}) by $\alpha \distortion b$ yields
\begin{eqnarray}
2P_b & = & \lambda_r r + \lambda_n n\nonumber\\
 & = & \frac{\alpha^2 S r^2 - \alpha^2 \Delta n r}{2SP_\delta -
   \alpha\Delta^2} + \frac{2\alpha P_\delta n^2 - \alpha^2 \Delta n
   r}{2 SP_\delta - \alpha\Delta^2}\nonumber\\
 & = & \frac{\alpha^2 S r^2 - 2 \alpha^2 \Delta n r + 2\alpha P_\delta
   n^2}{2 SP_\delta -
   \alpha \Delta^2}\:\:,\nonumber
\end{eqnarray}
that is, $r$ is solution of 
\begin{eqnarray}
\alpha^2 S r^2 -2\alpha^2 \Delta n r + 2(\alpha P_\delta n^2 - 2 S P_b P_\delta + \alpha P_b
\Delta^2) & = & 0\:\:,\nonumber
\end{eqnarray}
from which
\begin{eqnarray}
r & = & \frac{\Delta n}{S} \pm \frac{1}{\alpha
  S}\sqrt{\alpha^2\Delta^2n^2 - 2 S(\alpha P_\delta n^2 - 2 S P_b P_\delta + \alpha P_b
\Delta^2)}\nonumber\\
 & = & \frac{\Delta n}{ S} \pm \sqrt{ \left(\frac{\Delta
       n}{S}\right)^2 -  \frac{2}{\alpha^2 S}(\alpha P_\delta n^2 - 2 S P_b P_\delta + \alpha P_b
\Delta^2)}\nonumber\\
 & = & \frac{\Delta n}{S} \pm \sqrt{ \left(\frac{\Delta
       n}{S}\right)^2 -  \left(\frac{\Delta
       n}{S}\right)^2\left(\frac{2S}{\alpha\Delta^2} \cdot
     P_\delta - \frac{4 S^2}{\alpha^2 \Delta^2 n^2}\cdot P_b P_\delta +
     \frac{2 S}{\alpha n^2} \cdot P_b
\right)}\nonumber\\
 & = & \frac{\Delta n}{S}\cdot \left(1 \pm \sqrt{1 -
     (\xi(\delta) + \xi(b) - \xi(\delta)\cdot \xi(b))}\right) \nonumber\\
 & = & \frac{\Delta n}{S}\cdot \left(1 \pm \sqrt{(\xi(\delta)  - 1)(\xi(b) - 1)}\right)\label{lastxie}\:\:,
\end{eqnarray}
with, whenever $\langle 1_{\dG},
u\rangle\neq 0$,
\begin{eqnarray}
\xi(u) & \defeq & \frac{\langle 1_{\dG}, \distortion^{-1} 1_{\dG} \rangle
  \langle u, \distortion u\rangle}{(\langle 1_{\dG}, u\rangle)^2}\nonumber\\
 & & \geq 1 \:\:,\nonumber
\end{eqnarray}
from Cauchy-Schwartz inequality. This ends the proof of Theorem
\ref{theffront} when $\Delta \neq 0$ and $\langle \delta, \distortion
\delta\rangle \langle 1_{\dG}, \distortion^{-1} 1_{\dG} \rangle - (\langle
1_{\dG}, \delta\rangle)^2 > 0$.

Now, if $\Delta = 0$, 
$r$ is solution of
\begin{eqnarray}
\alpha^2 S r^2 - 2 P_\delta (2 S P_b - \alpha n^2) & = & 0\:\:.\label{eqfqEFDeltaZ}
\end{eqnarray}
Remark that $2 S P_b - \alpha n^2 = \alpha(\langle 1_{\dG}, \distortion^{-1}
1_{\dG} \rangle \langle b, \distortion b\rangle - (\langle 1_{\dG},
b\rangle)^2) \geq 0$ (Cauchy-Schwartz inequality and $\alpha \geq 0$), so eq. (\ref{eqfqEFDeltaZ})
always has a solution,
\begin{eqnarray}
r & = & \pm \left(\frac{n}{S}\right)\cdot \sqrt{\frac{2SP_\delta}{\alpha}\cdot \left(\xi(b) - 1\right)}\:\:.\nonumber
\end{eqnarray}
To finish up, when $2SP_\delta -
   \alpha\Delta^2 = 0$, system (\ref{eqsystemEF2}) simplifies to
\begin{eqnarray}
\left\{
\begin{array}{rcl}
\alpha S r & = & \Delta\cdot(\Delta \cdot \lambda_r  + S \cdot \lambda_n) \\
\alpha n & = & \Delta \cdot \lambda_r + S\cdot 
\lambda_n
\end{array}
\right. \:\:,\nonumber
\end{eqnarray}
and so $r = \Delta n / S$. We check that $2SP_\delta -
   \alpha\Delta^2 = 0$ is equivalent to stating $\xi(\delta) = 1$, in
   which case we also check that eq. (\ref{lastxie}) becomes $r =
   \Delta n / S$. This ends the proof of Theorem \ref{theffront}.
\end{proof}
 The equation
depends on $\Pi(b)$ since $\xi(b) = (2 \langle 1_{\dG},
\distortion^{-1} 1_{\dG}\rangle/(\alpha n^2))\cdot \Pi(b) \propto \Pi(b)$. The question is now how large can $r$ be \textit{independently of the
  mutation process}, under the constraint that the mutator picks $b \propto \alpha b'$ with
$\alpha = O(\epsilon)$ and $\langle 1_{\dG}, b'\rangle$ 
constant (call it setting ``\textbf{D}''). This prevents this mutator to ``artificially'' beat ours just because of the
magnitude of mutations. 
Let us define vectors $\ell_l \defeq [\langle 1_{\dG}, \distortion^{-1} 1_{\dG} \rangle\:\:\langle 1_{\dG},
  t-f\rangle]^\top$, $\ell_r \defeq [\langle 1_{\dG},
  t-f\rangle \:\: \langle t-f, \distortion (t-f)\rangle]^\top$ and
  $\ell_* \defeq [n \:\: r]^\top$. We also put in
(\textbf{D}) the constraint $\xi(t-f)>1$ (implying $\det[\ell_l|\ell_r] \neq 0$), and 
the fact that the decomposition $\ell_* \defeq v_l\ell_l
  + v_r \ell_r$ satisfies $|v_l| + |v_r| =
  O(\trace{[\ell_l|\ell_r]})$. This implies in particular that $n,r$ cannot be significantly larger than
  $\|t-f_0\|^2_{\dG}$, and so nature cannot have the organism ``jump''
  from far ($f_0$) to close to target ($t$) in just one or few mutations.
\begin{lemma}\label{lemopteff}
Under settings (\textbf{SF} + \textbf{D}), returns on the efficient
frontier satisfy $r  =  \tilde{O}(\epsilon)$.
\end{lemma}
\begin{proof}
(We prove an explicit bound without the tilde notation) 
We shall prove the more explicit bound that $r  =  O(\uprho(t-f_0) \cdot \epsilon)$, with (here, $x\vee y \defeq \max\{x,y\}$)
\begin{eqnarray}
\uprho(u) & \defeq & |\langle 1_{\dG},
    u\rangle| \dG \cdot
\left(\frac{\hessianmax \expressmax}{\hessianmin \expressmin}\right)^{\frac{5}{2}}  \cdot
\left(1\vee \left(\frac{1}{\hessianmax \expressmax}\right) +
  \frac{\hessianmin \expressmin\|u\|_{\dG}}{\sqrt{\dG}}\cdot 1\vee \left(\frac{\|u\|_{\dG}}{\sqrt{\dG}}\right)\right)^2\:\:.\nonumber
\end{eqnarray}
Let us define for short $0<\psi_{\min}\leq \psi_{\max}$ the minimal
and maximal eigenvalues of $\distortion$. Let $\delta \defeq t
- f$ for short and column vectors
\begin{eqnarray}
\ell_l = \left[
\begin{array}{c}
\langle 1_{\dG}, \distortion^{-1} 1_{\dG} \rangle\\
\langle 1_{\dG},
  \delta\rangle
\end{array}
\right]\:\:, \ell_r = \left[
\begin{array}{c}
\langle 1_{\dG},
  \delta\rangle\\
\langle \delta, \distortion \delta\rangle\\
\end{array}
\right]\:\:, \ell_* = \left[
\begin{array}{c}
n\\
r
\end{array}
\right]\:\:.\nonumber
\end{eqnarray}
Since $\langle 1_{\dG}, \distortion^{-1} 1_{\dG}\rangle \leq \dG /\psi_{\min}$
and $\langle b, \distortion b\rangle\leq \psi_{\max} \|b\|^2_{\dG}$, we have
\begin{eqnarray}
\xi(b) & = & \frac{\langle 1_{\dG}, \distortion^{-1} 1_{\dG} \rangle
  \langle b, \distortion b\rangle}{n^2}\nonumber\\
 & \leq & \frac{\dG \psi_{\max}}{n^2 \psi_{\min}} \cdot \|b\|^2_{\dG}\:\:.\nonumber
\end{eqnarray}
We have also from eq. (\ref{eqbfo}),
\begin{eqnarray}
\|b\|^2_{\dG} & = & \frac{\lambda^2_r}{\alpha^2} \cdot
\|\delta\|^2_{\dG} + \frac{2\lambda_r\lambda_n}{\alpha^2} \cdot \langle \delta, \distortion^{-1} 1_{\dG} \rangle + \frac{\lambda^2_n }{\alpha^2}  \cdot
\langle 1_{\dG}, \distortion^{-2} 1_{\dG} \rangle\label{normbg}\:\:.
\end{eqnarray}
It comes from eqs (\ref{eqlambdar}, \ref{eqlambdan}),
\begin{eqnarray}
\frac{\lambda^2_r}{\alpha^2} & = & \frac{\alpha^2 (\langle 1_{\dG},
  \distortion^{-1} 1_{\dG} \rangle \cdot r - \langle 1_{\dG},
  \delta\rangle \cdot n)^2}{(\langle 1_{\dG}, \distortion^{-1} 1_{\dG} \rangle
  \langle \delta, \distortion \delta\rangle - (\langle 1_{\dG},
  \delta\rangle)^2)^2}\:\:,\nonumber\\
  & = & \frac{\alpha^2 \det^2[\ell_l|\ell_*]}{\det^2[\ell_l|\ell_r]}\:\:,\nonumber\\
\frac{\lambda_r \lambda_n}{\alpha^2} & = & \frac{\alpha^2 \det[\ell_l|\ell_*]\det[\ell_*|\ell_r]}{\det^2[\ell_l|\ell_r]}\:\:,\nonumber\\
\frac{\lambda^2_n}{\alpha^2} & = & \frac{\alpha^2 \det^2[\ell_*|\ell_r]}{\det^2[\ell_l|\ell_r]}\:\:.\nonumber
\end{eqnarray}
Using the fact that $\langle \delta, \distortion^{-1} 1_{\dG} \rangle \leq
\sqrt{\dG} \|\delta\|_{\dG} /\psi_{\min}$ and $\langle 1_{\dG},
\distortion^{-2} 1_{\dG} \rangle\leq \dG/\psi_{\min}^2$, we obtain from eq. (\ref{normbg})
\begin{eqnarray}
\xi(b) & \leq & \frac{\dG \psi_{\max}}{\psi_{\min}} \cdot 
  \left(\frac{\alpha}{n}\right)^2\cdot\left(
    \frac{\psi_{\min} \|\delta\|_{\dG}\cdot \det[\ell_l|\ell_*] +
    \sqrt{\dG}\cdot \det[\ell_*|\ell_r]}{\psi_{\min}\cdot
      \det[\ell_l|\ell_r]}\right)^2\:\:.
\end{eqnarray}
Let $v_l, v_r$ be such that $\ell_* = v_l\cdot \ell_l + v_r\cdot
\ell_r$. Such reals are guaranteed to exist since $\det[\ell_l|\ell_r]  = \langle 1_{\dG}, \distortion^{-1} 1_{\dG}
\rangle\cdot \langle \delta, \distortion \delta\rangle - (\langle 1_{\dG},
  \delta\rangle)^2\neq 0$ by assumption. Then $\det[\ell_*|\ell_r] =
  v_l \cdot \det[\ell_l|\ell_r]$ and $\det[\ell_l|\ell_*] =
  v_r \cdot \det[\ell_l|\ell_r]$. We get
\begin{eqnarray}
\xi(b) & \leq & \frac{\dG \psi_{\max}}{\psi_{\min}} \cdot 
  \left(\frac{\alpha}{n}\right)^2\cdot\left(
    \frac{\sqrt{\dG}}{\psi_{\min}}+\|\delta\|_{\dG}\right)^2\cdot \left(
    \frac{|\det[\ell_l|\ell_*]| +
    |\det[\ell_*|\ell_r]|}{
      \det[\ell_l|\ell_r]}\right)^2\nonumber\\
 &  & = \frac{\dG \psi_{\max}}{\psi_{\min}}\cdot 
  \left(\frac{\alpha}{n}\right)^2\cdot
    \left(
    \frac{\sqrt{\dG}}{\psi_{\min}}+\|\delta\|_{\dG}\right)^2 \cdot \left(|v_l| +
    |v_r|\right)^2\nonumber\\
&  \leq & \frac{\dG \psi_{\max}}{\psi_{\min}}\cdot 
  \left(\frac{\alpha}{n}\right)^2\cdot
    \left(
    \frac{\sqrt{\dG}}{\psi_{\min}}+\|\delta\|_{\dG}\right)^2 \cdot
  \left(\frac{\dG}{\psi_{\min}} + \psi_{\max}\|\delta\|^2_{\dG}\right)^2\label{lastan}\:\:.
\end{eqnarray}
The last inequality comes from assumption (\textbf{D}), since
\begin{eqnarray}
\trace{[\ell_l|\ell_r]} & = & \langle 1_{\dG}, \distortion^{-1} 1_{\dG}\rangle + \langle \delta, \distortion \delta\rangle\nonumber\\
 & \leq & \frac{\dG}{\psi_{\min}} + \psi_{\max}\|\delta\|^2_{\dG}\nonumber\:\:.
\end{eqnarray}
If $n = \theta(\alpha)$, then the dependence in the magnitude of
mutations disappear and, taking the square root in ineq. (\ref{lastan}),
\begin{eqnarray}
\sqrt{\xi(b)} & = & O\left( \sqrt{\frac{\dG
      \psi_{\max}}{\psi_{\min}}} \left(
    \frac{\sqrt{\dG}}{\psi_{\min}}+\|\delta\|_{\dG}\right) \cdot
\left(\frac{\dG}{\psi_{\min}} + \psi_{\max}\|\delta\|^2_{\dG}\right)\right)\:\:,
\end{eqnarray}
out of which we get, assuming $\langle 1_{\dG}, \delta\rangle \neq 0$,
the following upperbound for returns on the efficient frontier:
\begin{eqnarray}
r & = & O\left\{ \frac{|\langle 1_{\dG},
    \delta\rangle|\psi_{\max}}{\dG} \cdot\left(1+  \sqrt{\frac{\dG
      \psi_{\max}}{\psi_{\min}}} \left(
    \frac{\sqrt{\dG}}{\psi_{\min}}+\|\delta\|_{\dG}\right) \cdot
\left(\frac{\dG}{\psi_{\min}} +
  \psi_{\max}\|\delta\|^2_{\dG}\right)\right) \cdot \epsilon\right\}
\nonumber\\
 & = & O\left\{ |\langle 1_{\dG},
    \delta\rangle|\cdot 
\left(\frac{\psi_{\max}}{\psi_{\min}}\right)^{\frac{3}{2}} \cdot \left(
   1 + \frac{\psi_{\min}\|\delta\|_{\dG}}{\sqrt{\dG}}\right) \cdot
\left(\frac{\dG}{\psi_{\min}} +
  \psi_{\max}\|\delta\|^2_{\dG}\right) \cdot \epsilon\right\}\nonumber\\
 & = & O\left\{ |\langle 1_{\dG},
    \delta\rangle| \dG \cdot
\left(\frac{\psi_{\max}}{\psi_{\min}}\right)^{\frac{5}{2}} \cdot \left(
   1+ \frac{\psi_{\min}\|\delta\|_{\dG}}{\sqrt{\dG}}\right) \cdot
\left(\frac{1}{\psi_{\max}} +
  \frac{\psi_{\min}\|\delta\|^2_{\dG}}{\dG}\right) \cdot \epsilon\right\}\nonumber\\
 & = & O\left\{ |\langle 1_{\dG},
    \delta\rangle| \dG \cdot
\left(\frac{\psi_{\max}}{\psi_{\min}}\right)^{\frac{5}{2}}  \cdot
\left(1 \vee \left(\frac{1}{\psi_{\max}}\right) +
  \frac{\psi_{\min}\|\delta\|_{\dG}}{\sqrt{\dG}}\cdot 1\vee \left(\frac{\|\delta\|_{\dG}}{\sqrt{\dG}}\right)\right)^2 \cdot \epsilon\right\}\:\:.\nonumber
\end{eqnarray}
We then conclude, noting that  we can fix
$\psi_{\min} = \hessianmin \expressmin$ and $\psi_{\max} = \hessianmax
\expressmax$. 
\end{proof}
Hence, in the (\textbf{D}) regime, the mutation mechanism on the
efficient frontier enjoys a dependence on $\epsilon$ of the same order
as that of superior beneficial mutations, that always exist in
$\mutbasis$ under setting (\textbf{SF}) alone --- if we modify our mutator so that it picks the best
mutation at each iteration, like $\mathsf{Opt\mbox{-}Sel}$ in \citep{akAEa},
then we are guaranteed to have mutations with a near-optimal
dependence in $\epsilon$ throughout all $\hgmon$.

\section{Toy experiments (full)}\label{app-toy-exp-setting}

\begin{table}[t]
\begin{center}
\begin{tabular}{c|cc}\hline\hline
& Performance plots & Evolution in the ambient space\\\hline
\rotatebox{90}{\hspace{2cm}Unsupervised} & \includegraphics[width=0.4\columnwidth]{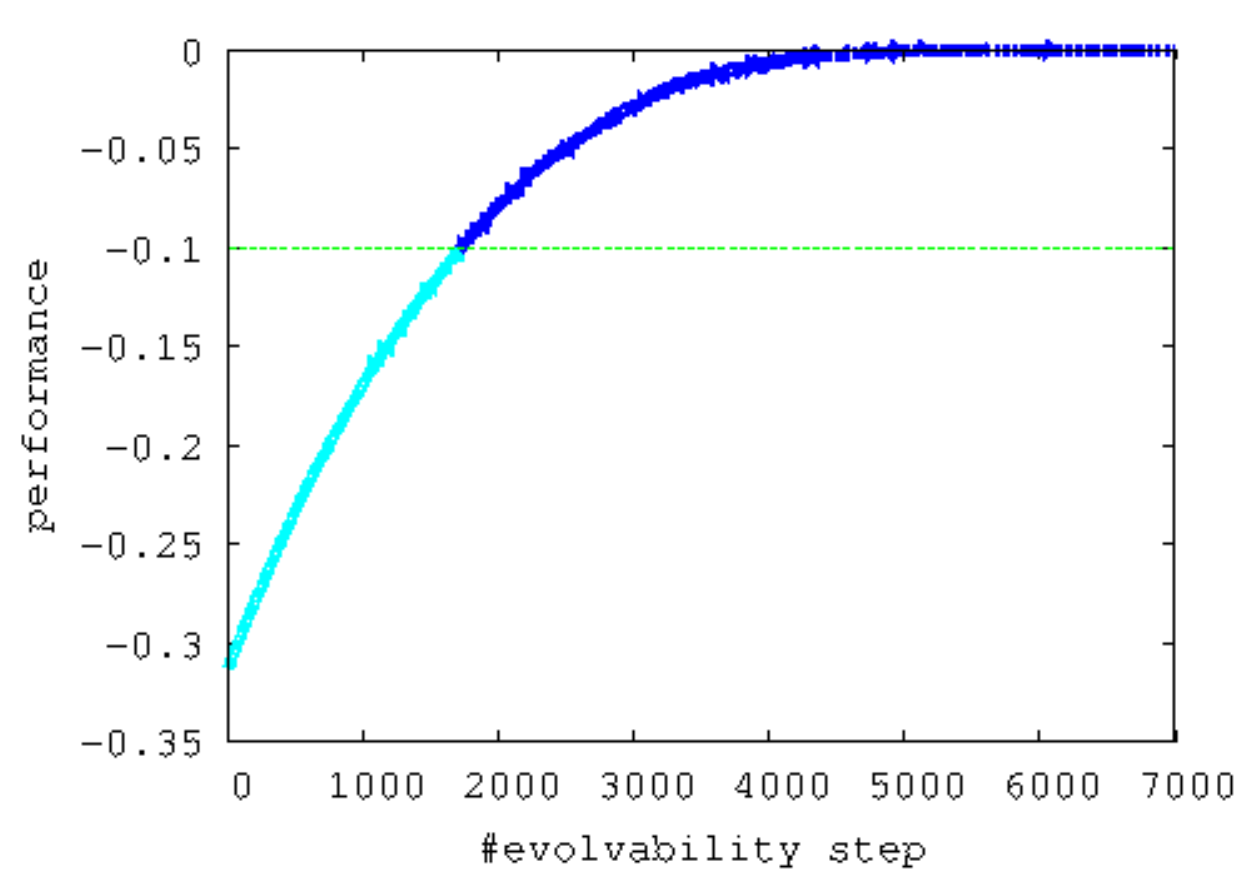} & \includegraphics[trim=0bp 100bp 0bp
150bp,clip,width=0.3\columnwidth]{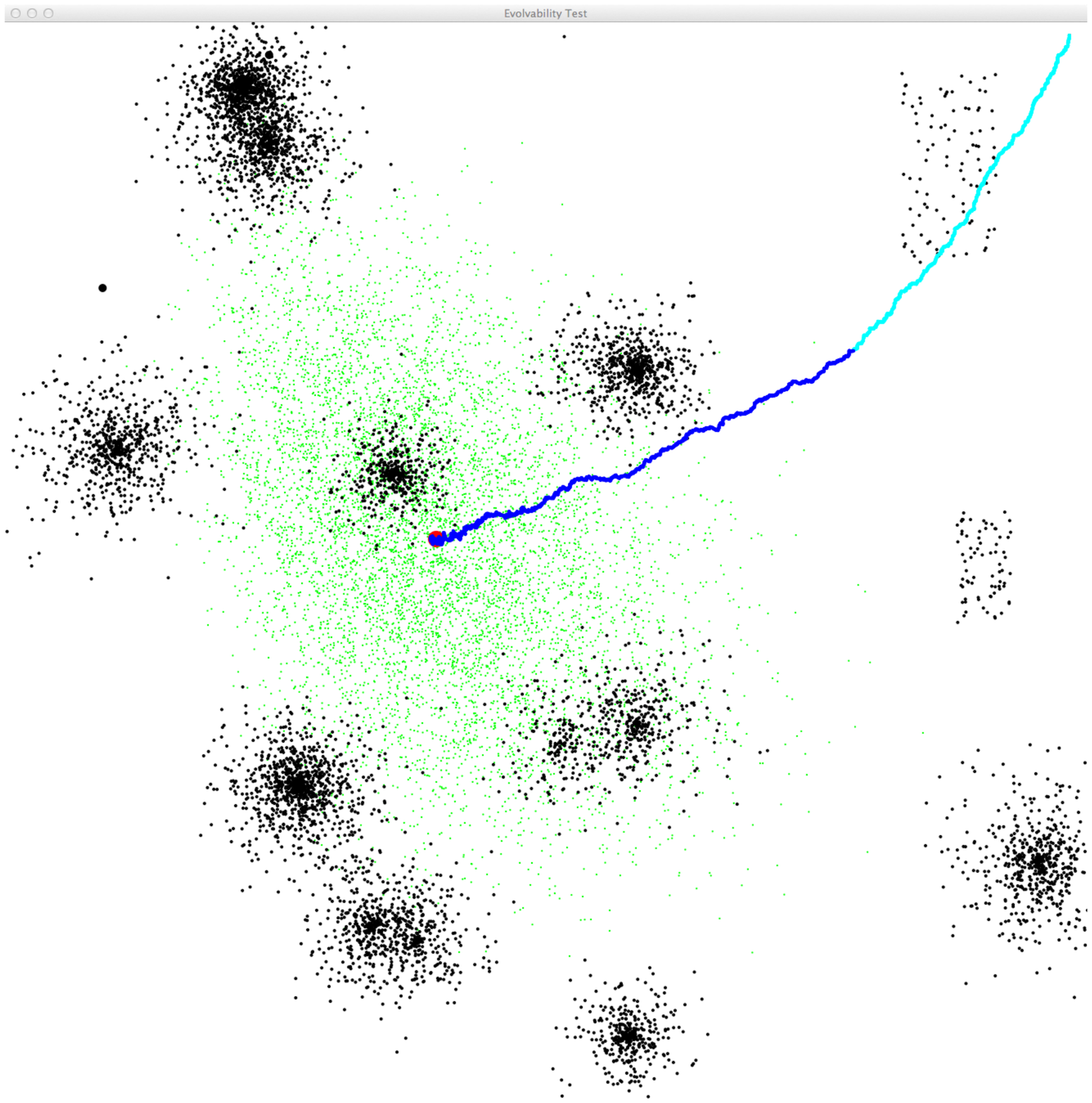}\\
\hline
\rotatebox{90}{\hspace{2cm}Supervised}& \includegraphics[width=0.4\columnwidth]{Figs/testSuper_super_optvstrueperf_2.pdf} & \includegraphics[trim=0bp 130bp 0bp
150bp,clip,width=0.3\columnwidth]{Figs/Supervised_EPS0_1_CLAMP50000_2.pdf}\\ \hline\hline
\end{tabular}
\end{center}
\caption{{\small Experiments on toy domains for unsupervised (top) and
  supervised learning (bottom). Left plots: performances, indicating
  the $-\epsilon$ evolvability threshold in green. In the case
  of supervised learning, the data (right: black = positive class,
  grey = negative class) are not linearly separable, so we compute the
  performance minus that of the optimal linear separator
  (eq. (\ref{condgenrob})). Right plots: data and organisms
  evolved. Organisms are displayed in cyan and dark blue. Dark blue
  means that the organism satisfies the condition of the evolvability model
  (eqs. (\ref{condgen}) for top, (\ref{condgenrob}) for bottom). A red
  dot (top) and red vector (down) display the target. Top right: green
  dots are estimates of the target used for evolution. Bottom right: the last
  vector evolved is displayed with a dark blue line; it is almost confounded
  with the target.}}
\label{t-exp-full}
\end{table}

\newcommand{\snd}{$\mathsf{Simple\mbox{-}Evol}$}

\begin{algorithm}[t]
\caption{\snd($\mathcal{S}$)}
\begin{algorithmic}
\label{algo:SND}
\STATE \textbf{Input}: sample $\mathcal{S}$;
\STATE Initialize $f \leftarrow f_0$;
\STATE For $t=0, 1, ..., T-1$
\STATE Step $t$.1: \textbf{if} $t\%1000 = 0$ \textbf{then} $\mutbasis \leftarrow \mathsf{Gen}(\mathcal{S})$;
\STATE Step $t$.2: Compute $\bene(f_t)$ and $\neut(f_t)$ using
$\neigh_\epsilon(f_t)$ as in (\ref{defneighbor})
\STATE Step $t$.3: \textbf{if} $\bene(f_t) \neq \emptyset$ then
$f_{t+1} \sim_{\mathrm{unif.}} \bene(f_t)$;
\STATE \qquad\qquad \textbf{else if} $\neut(f_t) \neq \emptyset$ then
$f_{t+1} \sim_{\mathrm{unif.}} \neut(f_t)$;
\STATE \qquad\qquad \textbf{else} $f_{t+1} \sim_{\mathrm{unif.}} \neigh_\epsilon(f_t)$;
\STATE \textbf{Return} $f_{T}$;
\end{algorithmic}
\end{algorithm}

The
high-level implementation of the algorithm, \snd, is sketched in
Algorithm \ref{algo:SND} ($\mathsf{Gen}(\mathcal{S})$ returns a random
set from $\mathcal{S}$ that generates ${\mathbb{R}}^2$). Details of
the experiments are as follows.

\noindent$\hookrightarrow$ \textit{Supervised learning}: we generate a mixture of 2D spherical Gaussians with
  random variance and a random number of vectors in each; vectors of
  each Gaussian are all labeled positive (black) or negative
  (grey, class picked uniformly at random). The data is not linearly separable, so the optimal linear
  separator does not have zero error, and we are in the agnostic
  setting of Theorem \ref{throbust}. The performance chosen is
  (minus) the square loss. To speed the mutator, we threshold the
  number $m$ of conditions sampled at each iteration to a maximum of
  $50 000$.\\
\noindent$\hookrightarrow$ \textit{Unsupervised learning}: to guarantee an optimum that can be measured and
  compared with, the problem is the estimation of a sample mean. The
  performance of an organism $f$ with respect to target $\hat{\mu}$ is
  $\perf_{\mu,\varphi}(f, {\mathcal{S}}) \defeq -\|f -
  \hat{\mu}_{\mathcal{S}}\|_2^2$. Notice that the target is the
  distribution's expectation. To complicate this easy task, we
  restrict the computation of $\hat{\mu}_{\mathcal{S}}$ over
  \textbf{5} vectors chosen at random in the data. Hence, while the
  expectation of $\hat{\mu}_{\mathcal{S}}$ is still the sample's
  expectation, the variance of $\hat{\mu}_{\mathcal{S}}$ is large. In
  Table \ref{t-exp-full} (top right), the small green dots display the
  $\hat{\mu}_{\mathcal{S}}$ picked. They spread on a large portion of
  the domain around the target.\\

In both supervised and unsupervised experiments, all data are
normalized to fit in a disk of unit norm. we implement the
mutator and all other parameters as they are given above, picking $\epsilon = \frac{1}{10}$
and $\mutbasis$ consisting of two randomly chosen vectors in the data
that generate ${\mathbb{R}}^2$,
flipped to get four mutations. Each 1000 evolvability steps, we renew
the basis, still completely at random. Sometimes, in particular for
the supervised experiment, both $\bene$ and $\neut$ are empty. In this
case of ``failure'', we just force evolution's hand by taking one of the mutants, chosen
uniformly at random in the neighborhood. Therefore, there is not other
optimization process carried out in our implementation than the weak optimization
achieved by the mutator.

\end{document}